\documentclass[11pt]{article}
\usepackage[colorlinks,linkcolor=blue,citecolor=orange]{hyperref}
\usepackage[letterpaper]{geometry}
\usepackage{amsmath,amsthm,amsfonts,amssymb,booktabs}
\usepackage{enumerate,color,xcolor}
\usepackage{graphicx}
\usepackage{subfigure}
\usepackage{url}
\usepackage{todo}
\newcommand{\Id}{\mathsf I_d}
\numberwithin{equation}{section}
\theoremstyle{plain}
\newtheorem{theorem}{Theorem}

\newtheorem{corollary}[theorem]{Corollary}

\newtheorem{lemma}[theorem]{Lemma}

\newtheorem{proposition}[theorem]{Proposition}

\newtheorem{remark}[theorem]{Remark}
\newtheorem{assumption}[theorem]{Assumption}

\usepackage{comment}

\newcommand{\E}{\mathbb{E}}
\newcommand{\ones}{\textbf{1}}

\newcommand{\beq}{\begin{eqnarray}}
\newcommand{\eeq}{\end{eqnarray}}
\newcommand{\beqs}{\begin{eqnarray*}}
\newcommand{\eeqs}{\end{eqnarray*}}
\newcommand{\R}{\mathbb{R}}

\makeatletter
\def\munderbar#1{\underline{\sbox\tw@{$#1$}\dp\tw@\z@\box\tw@}}
\makeatother

\begin{document}
\begin{center}
  \Large \bf Underdamped Anchored Langevin Dynamics
\end{center}
\author{}
\begin{center}
 {Qi Feng}\,\footnote{Department of
        Mathematics, Florida State University, 1017 Academic Way, Tallahassee,
        FL-32306, United States of America; qfeng2@fsu.edu},
  {Hoang M. Nguyen\,\footnote{Department of Mathematics, Florida State University, 1017 Academic Way, Tallahassee, FL-32306, United States of America; ngminhhoang7@gmail.com}},
  {Lingjiong Zhu\,\footnote{Department of Mathematics, Florida State University, 1017 Academic Way, Tallahassee, FL-32306, United States of America; zhu@math.fsu.edu}}
\end{center}
\begin{center}
 \today
\end{center}


\begin{abstract}
Classical Langevin Monte Carlo algorithms, such as the unadjusted Langevin algorithm (ULA), are
based on discretization schemes of the overdamped Langevin dynamics and are widely used for sampling in machine learning because they scale to high dimensions and large datasets. However, they require differentiable log-densities, and are less efficient for sampling heavy-tailed distributions.
Although various approaches have been proposed
in the literature for sampling non-differentiable and heavy-tailed targets,
all of them are developed in the overdamped setting only.
We propose a novel momentum-based underdamped anchored Langevin dynamics (UALD),
and underdamped anchored Langevin Monte Carlo (UALMC) algorithm based
on a discretization scheme of UALD.
Our paper is the first to provide a kinetic
(momentum-based) Langevin algorithm for non-differentiable and heavy-tailed targets that is
free of asymptotic bias.
The method replaces the original potential with a smooth reference potential and modifies the underdamped Langevin dynamics via multiplicative scaling in the drift term.
We establish non-asymptotic guarantees in the 2-Wasserstein distance to the target distribution and show
that the resulting iteration complexity has better dependence on dimension and accuracy
level compared to overdamped-based anchored Langevin algorithms.
We provide numerical experiments to illustrate the theory and practical performance of our proposed approach.
\end{abstract}

\section{Introduction}

The problem of \textit{sampling} a target distribution
$\pi(x)\propto\exp(-U(x))$
for a function $U:\mathbb{R}^d \to \mathbb{R}$
arises in many applications, including Bayesian statistical inference \cite{gelman1995bayesian}, Bayesian formulations of inverse problems \cite{stuart2010inverse}, and Bayesian classification and regression tasks in machine learning \cite{andrieu2003introduction,teh2016consistency,DistMCMC19,GIWZ2024,DIGing,regime2025}.
\textit{Langevin algorithms} are popular Markov chain Monte Carlo (MCMC) methods that allow one to sample from a target distribution of interest, and the classical one is based on the discretization of
\textit{overdamped Langevin dynamics} \cite{Dalalyan,DM2016,DMP2016,CB2018,DM2017,DK2017,Raginsky}:
\begin{equation}\label{eq:overdamped-2}
dX_{t}=-\nabla U(X_{t})dt+\sqrt{2}dB_{t},
\end{equation}
where $U:\mathbb{R}^{d}\rightarrow\mathbb{R}$
and $B_{t}$ is a standard $d$-dimensional Brownian motion. Under some mild assumptions on $U(\cdot)$, the diffusion \eqref{eq:overdamped-2} admits a unique stationary distribution with the density $\pi(x)\propto\exp(-U(x))$,
also known as the \emph{Gibbs distribution} \cite{chiang1987diffusion,stroock-langevin-spectrum}. In computing practice, the diffusion \eqref{eq:overdamped-2} is simulated by considering its discretization,
and one of the most commonly used discretization schemes is the Euler--Maruyama discretization of
\eqref{eq:overdamped-2}, often known as the \textit{unadjusted Langevin algorithm} (ULA) in the literature; see e.g. \cite{DM2017}:
\begin{equation}\label{discrete:overdamped}
x_{k+1}=x_{k}-\eta\nabla U(x_{k})+\sqrt{2\eta}\xi_{k+1},
\end{equation}
where $\xi_{k}$ are independent and identically distributed (i.i.d.) $\mathcal{N}(0,\Id)$ Gaussian random vectors and $\eta>0$ is the stepsize.

When the target density $e^{-U(x)}$ is \emph{non-differentiable} or \emph{heavy-tailed}, classical Langevin methods face serious practical and theoretical limitations.
The lack of differentiability in $U$ causes the gradient-based updates to be ill-defined.
Such an issue is not just theoretical curiosity, but arises often in modern Bayesian inference problems, including Bayesian logistic regression with sparsity inducing priors \cite{vono2018sparse}, Bayesian deep learning problems with non-smooth activation functions \cite{gurbuzbalaban2024penalized,castillo2024posterior,fortuin2022priors}.
When the target distribution is heavy-tailed,
the exponential ergodicity might fail to hold and concentration behavior assumed in many convergence analyses and can lead to instability or divergence \cite{RT96,he2022heavy,he-transformed,he-mean-square-analysis}.

There have been many studies on sampling non-differentiable targets.
Zeroth-order Langevin algorithms approximate first-order information based on evaluations of the potential function $U$ using finite-differences \cite{roy2022stochastic,dwivedi2019log}. This can be beneficial in some settings when the gradients do not exist or when they are hard to compute. However, zeroth-order methods are typically slower than first-order methods, as their gradient estimates tend to be noisier.
There are also alternative approaches that rely on approximating the potential with a smooth, differentiable surrogate to enable gradient-based sampling such as \cite{Zhou14}. In addition, proximal MCMC methods approximate the non-smooth function \( U \) using its \textit{Moreau--Yoshida envelope} (MYE)~\cite{DMP2016,Goldman21,Mou19,Salim20,Lamperski20,Bernton18,Wibisono19,Pereyra2016,Eftekhari2023}, which provides a smooth surrogate. The MYE of a function $U$ yields a smooth approximation \( U^\lambda \) that can be used with gradient-based Langevin algorithms for sampling. However, computing the gradient of the MYE requires evaluating a proximal step with respect to \(U\), which is typically computationally expensive \cite{Goldman21}, with efficient computations available only in specific cases.
Moreover, using the gradient of \( U^{\lambda} \) introduces bias and in some cases, a very tight envelope (i.e., a small \( \lambda \)) may be needed to obtain an accurate approximation. This, however, necessitates small stepsizes, leading to slow mixing \cite[Example 4.1]{Goldman21}.
Other envelopes that approximate a non-smooth potential with a smooth one, such as the forward-backward envelope, can also be used \cite{Eftekhari2023}; but they still require computing proximal steps. \textit{Piecewise-deterministic Markov processes} (PDMPs) such as the bouncy particle and zig-zag samplers, which do not suffer from asymptotic bias, offer an alternative. They can be applied to target distributions that are differentiable almost everywhere  \cite{deligiannidis2019exponential,durmus2020geometric,bierkens2019ergodicity}, and in practice, may be preferable to MYE-based methods—particularly when the MYE is difficult to compute or not well-defined
\cite{Goldman21}. However, PDMPs may encounter computational difficulties in high dimensions; event-time simulations can be demanding, and their performance can be sensitive to the availability of tight event-rate bounds \cite{Goldman21}.
Finally, mirror-descent–type algorithms that employ the Bregman–Moreau envelope instead of the MYE to handle non-smoothness on Riemannian manifolds \cite{Lau2022}.
When $U$ is non-differentiable but convex, there are also subgradient-based approaches, which relaxes the differentiability requirement. Among subgradient-based approaches, \cite{Durmus19} proposed the \textit{stochastic subgradient Langevin dynamics} (SSGLD) and provided convergence guarantees by leveraging the fact that sampling can be viewed as optimization in the space of probability measures. Other subgradient-based approaches include \cite{habring2024subgradient}.
Another approach is using Gaussian smoothing \cite{nesterov2017random} where one would obtain a smooth approximation of $U$ by convolving it with a Gaussian kernel. \cite{Chatterji20} analyze Gaussian-smoothing-based Langevin dynamics for convex and non-smooth $U$. This approach replaces $\nabla U$ in the Langevin dynamics \eqref{eq:overdamped-2} with $\nabla U_0$ and then simulates the dynamics, which, as the discretization parameter vanishes, converges to an approximate target density $\pi_0(x) \propto e^{-U_0(x)}$ and thus suffers from bias.

Sampling from target distributions with heavy tails where $U(x)$ grows sublinearly presents unique challenges for standard sampling algorithms and the literature is quite limited. Several works introduced modification of classical ULA
by using Metropolis steps \cite{Kamatani2017,belomestny2021fourier,bell2024adaptive}, which can be expensive for some applications involving high dimensionality and large datasets \cite{welling2011bayesian}.
\cite{he-mean-square-analysis} develops zeroth- and first-order Langevin algorithms for heavy-tailed distributions whose potentials satisfy weighted Poincar\'e inequalities, including \(t\)-distributions. Their first-order method discretizes associated It\^o diffusions and extends to zeroth-order variants via Gaussian smoothing. However, when a non-smooth potential \(U\) is approximated by a smoothed \(U_0\), the limiting It\^o diffusions converge to \(\pi_0(x) \sim e^{-U_0(x)}\), leading to asymptotic bias. In related work, \cite{he-transformed} characterize the class of heavy-tailed densities for which polynomial-order complexity guarantees can be obtained when ULA is applied to suitably transformed versions of the target.

To the best of our knowledge, all of the approaches listed above have been
developed and analyzed in the \emph{overdamped} setting only. The
Moreau--Yosida/proximal algorithms
\cite{DMP2016,Pereyra2016,Goldman21,Mou19,Salim20,Lamperski20,Bernton18,Wibisono19,Eftekhari2023},
the subgradient methods \cite{Durmus19,habring2024subgradient}, the
Gaussian-smoothing approach \cite{Chatterji20} and the Bregman--Moreau
envelope method \cite{Lau2022} are all obtained by discretizing the
first-order diffusion \eqref{eq:overdamped-2}; in particular, a kinetic
counterpart of the Bregman--Moreau scheme is explicitly listed as a direction
for future work in \cite{Lau2022}. Conversely, the non-asymptotic literature
on underdamped (kinetic) Langevin algorithms
\cite{Cheng,cheng-nonconvex,dalalyan2018kinetic,GGZ,GGZ2,Ma2021} assumes
throughout that $\nabla U$ exists and is Lipschitz; the recent relaxations of
that hypothesis concern drifts that fail to be \emph{globally} Lipschitz or
that grow superlinearly, but are still differentiable, and therefore do not
cover non-differentiable targets. The only exception we are aware of is the
class of PDMP samplers \cite{deligiannidis2019exponential,durmus2020geometric,bierkens2019ergodicity},
which are non-reversible but are not discretizations of a kinetic diffusion.
Accordingly, our paper is the first to provide a kinetic
(momentum-based) Langevin algorithm for non-differentiable targets that is
free of asymptotic bias, together with non-asymptotic guarantees; we
highlight this as one of the contributions of the paper.

Recently, \cite{GNZZ2025} proposed and studied \textit{anchored Langevin dynamics} (ALD)
that modifies the Langevin dynamics itself,
such that it is \emph{anchored} with a new potential $U_{0}$ and the resulting dynamics converge to the original target \(\pi(x) \propto e^{-U(x)}\) and are free of such bias:
\begin{equation}\label{overdamped:SDE:modified}
	dX_{t}=-\nabla U_{0}(X_{t})e^{U(X_{t})-U_{0}(X_{t})}dt
	+\sqrt{2}e^{(U(X_{t})-U_{0}(X_{t}))/2}dB_{t}.
\end{equation}
Non-asymptotic convergence guarantees are obtained for \eqref{overdamped:SDE:modified}
and its Euler-Maruyama discretization in \cite{GNZZ2025}.
Very recently, \cite{RALD2026} proposed and studied \textit{reflected anchored Langevin dynamics} (RALD),
which is \eqref{overdamped:SDE:modified} with a reflected boundary.
This allows sampling on a constrained domain where the target density is not necessarily differentiable.
\cite{non-reversible-anchored-2026} introduced and studied \textit{non-reversible anchored Langevin dynamics} (NALD)
and \textit{non-reversible reflected anchored Langevin dynamics} (NRALD) by adding an anti-symmetric
matrix to the drift term of ALD and NALD respectively.
Through the lens of finite-time non-asymptotic convergence analysis, large deviations analysis and asymptotic variance reduction,
\cite{non-reversible-anchored-2026} showed that by breaking reversibility, NALD and NRALD can converge
to the target distributions faster than their reversible counterparts.

In the literature, the algorithms based on underdamped (a.k.a. second-order, inertial or kinetic) Langevin diffusion dynamics are popular alternatives to the ones based on overdamped Langevin dynamics; see e.g.
\cite{Cheng,cheng-nonconvex,dalalyan2018kinetic,GGZ,GGZ2,Ma2021}:
\begin{align}
&dV_{t}=-\gamma V_{t}dt- \nabla U(X_{t})dt+\sqrt{2\gamma }dB_{t}, \label{eq:VL}
\\
&dX_{t}=V_{t}dt, \label{eq:XL}
\end{align}
where $\gamma>0$ is the friction coefficient, $X_{t},V_{t} \in \mathbb{R}^d$ models the position and the momentum of a particle moving in a field of force (described by the gradient of $U$) plus a random (thermal) force described by the $d$-dimensional standard Brownian motion $B_{t}$. It is known that under some mild assumptions on $U$, the Markov process $(X_{t}, V_{t})_{t\geq 0}$ is ergodic and admits a unique stationary distribution $\pi$ with density
$\pi(x,v) \propto \exp\left(- \left(\frac{1}{2} \Vert v\Vert^2 +U(x)\right) \right)$.
Hence, the $x$-marginal distribution of the stationary distribution with the density $\pi(x,v)$ is exactly the invariant distribution of the overdamped Langevin diffusion.
It is known that the momentum-based underdamped Langevin dynamics can accelerate the convergence to the Gibbs distribution
and outperform the overdamped Langevin dynamics \cite{Eberle,Ma2021,GGZ}.
This motivates us to ask the question
whether there is an underdamped analog of the anchored Langevin dynamics \eqref{overdamped:SDE:modified}.
In this paper, we give a positive answer to this question by proposing the \textit{underdamped anchored Langevin dynamics} (UALD):
\begin{align}
&dV_{t}=-\gamma V_{t}dt- \nabla U_{0}(X_{t})e^{U(X_{t})-U_{0}(X_{t})}dt+\sqrt{2\gamma}dB_{t}, \label{eq:VL:modified:2}
\\
&dX_{t}=V_{t}e^{U(X_{t})-U_{0}(X_{t})}dt, \label{eq:XL:modified:2}
\end{align}
where $U_{0}:\mathbb{R}^{d}\rightarrow\mathbb{R}$ plays the role of reference potential. We will show
that UALD preserves the Gibbs distribution $\pi(x,v) \propto \exp\left(- \left(\frac{1}{2} \Vert v\Vert^2 +U(x)\right) \right)$ as an invariant distribution.
Next, we consider the discretization of \eqref{eq:VL:modified:2}-\eqref{eq:XL:modified:2}.
It is known in the literature that
for underdamped Langevin algorithms,
one can use finer discretization scheme than the Euler-Maruyama \cite{Cheng,dalalyan2018kinetic}.
Inspired by the discretization scheme for the classical underdamped Langevin dynamics in \eqref{eq:VL}-\eqref{eq:XL} in \cite{Cheng},
we introduce the \textit{underdamped anchored Langevin Monte Carlo} (UALMC) algorithm:
\begin{equation}\label{eq:scheme}
\begin{aligned}
  v_{k+1}&=\psi_0(\eta)\,v_k-\psi_1(\eta)\,g_k+\sqrt{2\gamma}\,\xi_{k+1},\\
  x_{k+1}&=x_k+a_k\left(\psi_1(\eta)\,v_k-\psi_2(\eta)\,g_k\right)
     +\sqrt{2\gamma}\,a_k\,\xi'_{k+1},
\end{aligned}
\end{equation}
where $(\xi_{k+1},\xi'_{k+1})$ is a centered $2d$-dimensional Gaussian random vector, independent across $k$
and across coordinates, with the per-coordinate covariance
\begin{equation}\label{eq:Cnoise}
C=\int_0^\eta\begin{pmatrix}\psi_0(t)\\ \psi_1(t)\end{pmatrix}
\begin{pmatrix}\psi_0(t)\\ \psi_1(t)\end{pmatrix}^{\!\top}dt
=\left(\begin{array}{cc}
\frac{1-e^{-2\gamma\eta}}{2\gamma} & \frac{1-2e^{-\gamma\eta}+e^{-2\gamma\eta}}{2\gamma^{2}}
\\
\frac{1-2e^{-\gamma\eta}+e^{-2\gamma\eta}}{2\gamma^{2}} & \frac{2\gamma\eta-3+4e^{-\gamma\eta}-e^{-2\gamma\eta}}{2\gamma^{3}}
\end{array}\right),
\end{equation}
where
\begin{equation}\label{eq:psidef}
  \psi_0(t)=e^{-\gamma t},\qquad
  \psi_1(t)=\int_0^t\psi_0(s)\,ds=\frac{1-e^{-\gamma t}}{\gamma},\qquad
  \psi_2(t)=\int_0^t\psi_1(s)\,ds,
\end{equation}
and
\begin{equation}\label{eq:frozen}
  a_k:=a(x_k):=e^{U(x_k)-U_0(x_k)},\qquad
  g_k:=a(x_k)\nabla U_0(x_k).
\end{equation}
The UALMC algorithm \eqref{eq:scheme} is obtained by discretizing the UALD
\eqref{eq:VL:modified:2}-\eqref{eq:XL:modified:2} by an exponential
integrator that freezes the state-dependent coefficients over each step.
To see that \eqref{eq:scheme} is a discretization scheme of \eqref{eq:VL:modified:2}-\eqref{eq:XL:modified:2}, we first introduce
the following auxiliary SDE:
\begin{align}
&d\tilde{V}_{t}=-\gamma\tilde{V}_{t}dt- \nabla U_{0}\left(\tilde{X}_{\eta\lfloor t/\eta\rfloor}\right)e^{U(\tilde{X}_{\eta\lfloor t/\eta\rfloor})-U_{0}(\tilde{X}_{\eta\lfloor t/\eta\rfloor})}dt+\sqrt{2\gamma}dB_{t}, \label{eq:VL:auxiliary}
\\
&d\tilde{X}_{t}=\tilde{V}_{t}e^{U(\tilde{X}_{\eta\lfloor t/\eta\rfloor})-U_{0}(\tilde{X}_{\eta\lfloor t/\eta\rfloor})}dt.\label{eq:XL:auxiliary}
\end{align}
In other words, for any $k\eta\leq t<(k+1)\eta$ with $k\in\mathbb{N}$,
\begin{align}
&d\tilde{V}_{t}=-\gamma\tilde{V}_{t}dt- \nabla U_{0}\left(\tilde{X}_{k\eta}\right)e^{U(\tilde{X}_{k\eta})-U_{0}(\tilde{X}_{k\eta})}dt+\sqrt{2\gamma}dB_{t},
\\
&d\tilde{X}_{t}=\tilde{V}_{t}e^{U(\tilde{X}_{k\eta})-U_{0}(\tilde{X}_{k\eta})}dt,
\end{align}
which implies that $(\tilde{V}_{t},\tilde{X}_{t})$ follows an Ornstein-Uhlenbeck process
on $t\in[k\eta,(k+1)\eta)$, that can be solved explicitly and $\left(\tilde{V}_{k\eta},\tilde{X}_{k\eta}\right)$
has the same distribution as $(v_{k},x_{k})$, where $(v_{k},x_{k})$
satisfies the iterates \eqref{eq:scheme}.

The contributions of our paper can be summarized as follows.
\begin{itemize}
\item
Compared to the existing literature on Langevin-type algorithms
for sampling non-differentiable target densities,
our paper is the first to provide a kinetic
and momentum-based Langevin algorithm for non-differentiable targets that is
free of asymptotic bias, together with non-asymptotic guarantees.
\item
We first verify that the continuous-time UALD \eqref{eq:VL:modified:2}-\eqref{eq:XL:modified:2} preserves the Gibbs distribution as an invariant distribution (Theorem~\ref{thm:invariant}). We obtain a non-asymptotic 2-Wasserstein convergence bound for the continuous-time UALD \eqref{eq:VL:modified:2}-\eqref{eq:XL:modified:2} (Theorem~\ref{thm:contraction}, Corollary~\ref{cor:Wasserstein}). The proof technique is based on using synchronous coupling, establishing a contraction for the linear approximation (Proposition~\ref{prop:linear}) and then for the full dynamics by controlling the perturbation term and our analysis relies on establishing novel Lyapunov bound (Lemma~\ref{lem:drift})
and exponential moment (Lemma~\ref{lem:expmoment}).
We further show that UALD can sample heavy-tailed distributions such as
Student's~$t$ distributions with exponential convergence in time
(Theorem~\ref{thm:heavyL2}).
\item
We obtain finite-time non-asymptotic convergence bound for the discrete-time UALMC algorithm \eqref{eq:scheme} (Theorem~\ref{thm:disc}).
Based on this, we further obtain iteration complexity (Corollary~\ref{cor:iteration}),
which improves upon that of ALMC algorithm based on the discretization of \eqref{overdamped:SDE:modified}
with better dependence on dimension and accuracy level (Remark~\ref{rmk:comparison}).
Compared to the non-asymptotic convergence analysis for the classical KLMC in the literature \cite{dalalyan2018kinetic}, the novelty of our analysis is threefold.
First, the position equation \eqref{eq:XL:modified:2} carries the
state-dependent speed $a(x)=e^{U(x)-U_0(x)}$, so \emph{both} the force and
the speed must be frozen; the frozen speed generates a local error term
that is \emph{quadratic} in the velocity (see \eqref{eq:dX}), a term that has
no counterpart in the classical KLMC analysis.
Second, the local error involves the \emph{unbounded}
stationary velocity and $\nabla U_0$, which we never bound uniformly but
whose moments we control through the Lyapunov structure of
Lemma~\ref{lem:drift}; indeed, the anchored force $g=a\nabla U_{0}$ is
\emph{not} globally Lipschitz (see \eqref{eq:glip}), and since the local error
contains \emph{products} of two unbounded factors, the $L_{2}$ deviation
splitting of \cite{dalalyan2018kinetic} does not suffice; we replace it by
uniform polynomial moment bounds for the iterates
(Lemma~\ref{lem:itermom}).
And third,
Theorem~\ref{thm:contraction} contracts the
\emph{weighted} functional $\E[\mathcal H_t\Psi_t']$ rather than
$\E[\mathcal H_t]$, so the discrete recursion must propagate a weighted
distance (Lemma~\ref{lem:cont-step}).
\item
Numerical experiments demonstrate the efficiency of our proposed algorithm.
In particular, we conducted synthetic experiments
such as sampling Laplace distributions
and heavy-tailed multivariate Student's distributions
and Bayesian logistic regression and Bayesian neural network
experiments using real data (Section~\ref{sec:numerical}).
\end{itemize}

\subsection*{Notations}
We summarize the notations here for readers' convenience.
\begin{itemize}
\item Define $\mathcal{P}_{2}(\mathbb{R}^{d})$
as the space consisting of all the Borel probability measures $\mu$
on $\mathbb{R}^{d}$ with the finite second moment
(based on the Euclidean norm).
\item For any $\mu_{1},\mu_{2}\in\mathcal{P}_{2}(\mathbb{R}^{d})$,
the $2$-Wasserstein
distance $\mathcal{W}_{2}$ between $\mu_{1}$ and $\mu_{2}$ is defined as:
$\mathcal{W}_{2}(\mu_{1},\mu_{2}):=\left(\inf\mathbb{E}\left[\Vert X_{1}-X_{2}\Vert^{2}\right]\right)^{1/2},$
where the infimum is taken over all joint distributions of the random variables $X_{1},X_{2}$ with marginal distributions
$\mu_{1},\mu_{2}$ respectively; see e.g. \cite{villani2008optimal}.
\item A function $g:\mathbb{R}^{d}\rightarrow\mathbb{R}$ is $m$-strongly convex and $L$-smooth if
\begin{equation*}
\frac{L}{2}\left\Vert x-y \right\Vert^2 \geq g(x) - g(y) - \nabla g(y)^{\top}(x-y) \geq \frac{m}{2}\left\Vert x-y\right\Vert^2, \quad \text{for every $x,y \in \mathbb R^d$}.
\end{equation*}
\item $\Id$ denotes the $d\times d$ identity matrix.
\item For a symmetric positive definite matrix $S$ we write
$|w|_{S}:=(w^{\top}Sw)^{1/2}$, and $\|\cdot\|_{\mathrm{op}}$ for the operator
norm.
\end{itemize}

\section{Main Results}
In this section, we present the main results on the theoretical analysis of our paper. We present the theoretical analysis for the continuous-time dynamics \eqref{eq:VL:modified:2}-\eqref{eq:XL:modified:2} in Section~\ref{sec:continuous}, and the theoretical analysis for the discrete-time algorithm \eqref{eq:scheme} in Section~\ref{sec:discrete}.

\subsection{Continuous-Time Analysis}\label{sec:continuous}
\subsubsection{Invariance distribution}
First, we verify that the Gibbs distribution $\pi(x,v) \propto \exp\left(- \left(\frac{1}{2} \Vert v\Vert^2 +U(x)\right) \right)$
is an invariant distribution for the SDEs \eqref{eq:VL:modified:2}-\eqref{eq:XL:modified:2}.
\begin{theorem}\label{thm:invariant}
The Gibbs distribution $\pi(x,v) \propto \exp\left(- \left(\frac{1}{2} \Vert v\Vert^2 +U(x)\right) \right)$
is an invariant distribution for the SDEs \eqref{eq:VL:modified:2}-\eqref{eq:XL:modified:2}.
\end{theorem}
\begin{proof}
The infinitesimal generator of the SDEs \eqref{eq:VL:modified:2}-\eqref{eq:XL:modified:2} is given by
\begin{equation}
\mathcal{L}_{0}:=\gamma\Delta_{v}-\gamma v\cdot\nabla_{v}-e^{U(x)-U_{0}(x)}\nabla_{x}U_{0}(x)\cdot\nabla_{v}
+e^{U(x)-U_{0}(x)}v\cdot\nabla_{x},
\end{equation}
and we have
\begin{align}
&\mathcal{L}_{0}^{\ast}e^{-\frac{1}{2}\Vert v\Vert^{2}-U(x)}
\nonumber
\\
&=\Delta_{v}\left(\gamma e^{-\frac{1}{2}\Vert v\Vert^{2}-U(x)}\right)+\text{div}_{v}\left(\gamma ve^{-\frac{1}{2}\Vert v\Vert^{2}-U(x)}\right)
\nonumber
\\
&\qquad\qquad
+\text{div}_{v}\left(e^{U(x)-U_{0}(x)}\nabla_{x}U_{0}(x)e^{-\frac{1}{2}\Vert v\Vert^{2}-U(x)}\right)
-\text{div}_{x}\left(e^{U(x)-U_{0}(x)}ve^{-\frac{1}{2}\Vert v\Vert^{2}-U(x)}\right)
\nonumber
\\
&=\Delta_{v}\left(\gamma e^{-\frac{1}{2}\Vert v\Vert^{2}-U(x)}\right)+\text{div}_{v}\left(\gamma ve^{-\frac{1}{2}\Vert v\Vert^{2}-U(x)}\right)
\nonumber
\\
&\qquad\qquad
+\text{div}_{v}\left(\nabla_{x}U_{0}(x)e^{-\frac{1}{2}\Vert v\Vert^{2}-U_{0}(x)}\right)
-\text{div}_{x}\left(ve^{-\frac{1}{2}\Vert v\Vert^{2}-U_{0}(x)}\right)=0,
\end{align}
where $\mathcal{L}_{0}^{\ast}$ is the adjoint operator of $\mathcal{L}_{0}$.
Hence, the SDEs \eqref{eq:VL:modified:2}-\eqref{eq:XL:modified:2} preserve
the Gibbs distribution $\pi(x,v) \propto \exp\left(- \left(\frac{1}{2} \Vert v\Vert^2 +U(x)\right) \right)$ as its invariant distribution.
The proof is complete.
\end{proof}

\subsubsection{Non-asymptotic convergence bound}
In this section, we study the non-asymptotic convergence of the continuous-time UALD
\eqref{eq:VL:modified:2}--\eqref{eq:XL:modified:2} to its invariant
distribution and obtain an explicit upper bound, using the synchronous coupling
method together with a Lyapunov argument.

\paragraph{Synchronous coupling.}
We recall from \eqref{eq:VL:modified:2}-\eqref{eq:XL:modified:2} that for the underdamped anchored Langevin dynamics (UALD):
\begin{align}
&dV_t=-\gamma V_t\,dt-\nabla U_0(X_t)e^{U(X_t)-U_0(X_t)}\,dt+\sqrt{2\gamma}\,dB_t,\label{eq:UALD:1}
    \\
&dX_t=V_t\,e^{U(X_t)-U_0(X_t)}\,dt,
  \label{eq:UALD:2}
\end{align}
and consider the synchronous coupling $(X',V')$ solving the \emph{same} equations with
the \emph{same} Brownian motion $B$ but a different initial condition:
\begin{align}
&dV'_t=-\gamma V'_t\,dt-\nabla U_0(X'_t)e^{U(X'_t)-U_0(X'_t)}\,dt+\sqrt{2\gamma}\,dB_t,\label{eq:UALD:coupled:1}
    \\
&dX'_t=V'_t\,e^{U(X'_t)-U_0(X'_t)}\,dt.
  \label{eq:UALD:coupled:2}
\end{align}
By subtracting \eqref{eq:UALD:coupled:1}-\eqref{eq:UALD:coupled:2} from \eqref{eq:UALD:1}-\eqref{eq:UALD:2}, we obtain
\[
  d\begin{pmatrix}V_t-V_t'\\ X_t-X_t'\end{pmatrix}
  =\begin{pmatrix}
    -\left[\gamma(V_t-V_t')+\nabla U_0(X_t)e^{U(X_t)-U_0(X_t)}-\nabla U_0(X_t')e^{U(X_t')-U_0(X_t')}\right]\\[2pt]
    V_t e^{U(X_t)-U_0(X_t)}-V_t'e^{U(X_t')-U_0(X_t')}
  \end{pmatrix}dt .
\]
By using the interpolation identity
\begin{align}\label{interpolation:eqn}
g(z_1)-g(z_2)=\left\langle z_1-z_2,\int_0^1\nabla g(sz_1+(1-s)z_2)\,ds\right\rangle,
\end{align}
for $g=e^{U-U_0}$, we get
\begin{equation}
  e^{U(X_t)-U_0(X_t)}-e^{U(X_t')-U_0(X_t')}=\left\langle F_t,\,X_t-X_t'\right\rangle,
\end{equation}
where
\begin{equation}\label{eq:Ft}
  F_t:=\int_0^1\nabla e^{U-U_0}\left(sX_t+(1-s)X_t'\right)\,ds,
\end{equation}
and by using \eqref{interpolation:eqn} for $g=\nabla U_0$ (componentwise),
we get
\begin{equation}
  \nabla U_0(X_t)-\nabla U_0(X_t')=H_t\,\left(X_t-X_t'\right),
\end{equation}
where
\begin{equation}\label{eq:Ht}
H_{t}:=\int_0^1\nabla^{2} U_{0}\left(sX_t+(1-s)X_t'\right)\,ds,
\end{equation}
that satisfies
\begin{equation}
  m\,\Id\preceq H_t\preceq L\,\Id,
\end{equation}
due to Assumption~\ref{ass:main}(B).
Next, let us define the difference:
\[
W_t:=\left(V_t-V_t',\,X_t-X_t'\right)^{\top}.
\]
Then, a direct computation gives the closed linear system:
\begin{equation}\label{eq:diffdyn}
  \frac{d}{dt}W_t=\left(M_0(H_t)+E_t\right)\,W_t,
  \qquad
  M_0(H_t):=\begin{pmatrix}-\gamma\Id&-H_t\\ \Id&0\end{pmatrix},
\end{equation}
where
\begin{equation}\label{eq:Et}
  E_t:=\begin{pmatrix}
    0 & (1-a_t)H_t-\nabla U_0(X_t')F_t^{\top}\\[2pt]
    (a_t-1)\Id & V_t'F_t^{\top}
  \end{pmatrix},
  \qquad a_t:=e^{U(X_t)-U_0(X_t)}.
\end{equation}
Next, we introduce the constant symmetric weight
\[
S:=\begin{pmatrix}a\,\Id&\epsilon\,\Id\\ \epsilon\,\Id&b\,\Id\end{pmatrix},
\]
and its quadratic form
\[
  \mathcal H_t:=W_t^{\top}SW_t
  =a\left\|V_t-V_t'\right\|^2+2\epsilon\left\langle V_t-V_t',X_t-X_t'\right\rangle+b\left\|X_t-X_t'\right\|^2 .
\]
Here $S\succ0$ if and only if $a,b>0$ and $ab>\epsilon^2$, in which case
\[
\lambda_{\min}(S)\|W\|^2\le\mathcal H\le\lambda_{\max}(S)\|W\|^2,
\]
so that contraction
of $\mathcal H$ yields $\mathcal W_2$ contraction up to
$\sqrt{\lambda_{\max}/\lambda_{\min}}$.
\paragraph{Linear contraction.}
Note that the difference $W_t=\left(V_t-V_t',\,X_t-X_t'\right)^{\top}$ obeys the linear time-varying system \eqref{eq:diffdyn}. The matrix
$E_t$ collects \emph{exactly} the discrepancy produced by anchoring to a
mismatched potential $U_0\neq U$: it is built from $a_t-1$ and $F_t$, and both
vanish identically when $U=U_0$ (then $a_t\equiv1$, $F_t\equiv0$, so $E_t\equiv0$).
By removing $E_t$, we obtain the linearized system:
$$
\dot W=M_0(H_t)W,
$$
i.e. the
frozen-Hessian linearization of the \emph{exact} underdamped gradient flow of
$U_0$. Proposition~\ref{prop:linear} shows that this $E_t$-free part already
contracts in the constant metric $S$, uniformly over the admissible Hessian range
$H_t\in[m\Id,L\Id]$ from \eqref{eq:Ht}; ``linear'' refers to this unperturbed
part. The full (i.e. perturbed) bound is then obtained by controlling $E_t$, which is small precisely when the anchoring
mismatch $(\omega_0,\omega_1)$ is small.
\begin{proposition}[Linear contraction]\label{prop:linear}
Let $M_0(H)=\begin{pmatrix}-\gamma\Id&-H\\ \Id&0\end{pmatrix}$ with $H$ symmetric,
$m\Id\preceq H\preceq L\Id$, $0<m\le L$. Take $a=1$, $b=\epsilon\gamma$,
$\epsilon\in(0,\gamma)$, so $S=\begin{pmatrix}\Id&\epsilon\Id\\
\epsilon\Id&\epsilon\gamma\Id\end{pmatrix}\succ0$. If
\begin{equation}\label{eq:star}
  \epsilon(\gamma-\epsilon)>\frac{L}{4},
\end{equation}
then for every symmetric $H\in[m\Id,L\Id]$,
\[
  \left(M_0(H)\right)^{\top}S+S\,M_0(H)\ \preceq\ -2\kappa\,S ,
\]
with the \emph{explicit, largest admissible} constant
\begin{equation}\label{eq:kappa}
  \kappa
  :=\frac12\left(\gamma-\sqrt{\max\{q(m),q(L)\}}\right),
  \qquad
  q(h):=\gamma^2-4h+\frac{h^2}{\epsilon(\gamma-\epsilon)} .
\end{equation}
Condition \eqref{eq:star} is solvable if and only if $\gamma>\sqrt L$ (optimized at
$\epsilon=\gamma/2$), and is exactly the condition $\kappa>0$.
\end{proposition}
\begin{proof}
With $a=1,\ b=\epsilon\gamma$, a block computation gives
\[
  G^{\mathrm{lin}}(H):=\left(M_0(H)\right)^{\top}S+S\,M_0(H)
  =\begin{pmatrix}2(\epsilon-\gamma)\Id&-H\\ -H&-2\epsilon H\end{pmatrix}.
\]
The lower-right block $-2\epsilon H\prec0$. Thus, by the Schur complement
criterion for block matrices, see e.g.\ \cite[Theorem 7.7.7]{HornJohnson},
$G^{\mathrm{lin}}(H)\prec0$ if and only if
$2(\epsilon-\gamma)\Id+\frac{1}{2\epsilon}H\prec0$.
Since $H\preceq L\Id$,
\[
2(\epsilon-\gamma)\Id+\frac{1}{2\epsilon}H\preceq\left(2(\epsilon-\gamma)+\frac{L}{2\epsilon}\right)\Id\prec0,
\]
where we used \eqref{eq:star}.
For the explicit rate, note $S$ is block-scalar and $M_0(H)$ is
block-structured, so conjugating by $\mathrm{diag}(Q,Q)$ that diagonalizes
$H=Q\,\mathrm{diag}(h_i)\,Q^{\top}$ decouples the problem into $d$ planar systems
with $h_i\in[m,L]$. Thus $G^{\mathrm{lin}}(H)\preceq-2\kappa S$ for all
$H\in[m\Id,L\Id]$ if and only if, for every $h\in[m,L]$,
$-G(h)\succeq2\kappa\,S_0$ with
\[
G(h)=\begin{pmatrix}
2(\epsilon-\gamma)&-h
\\
-h&-2\epsilon h
\end{pmatrix},
\qquad
S_0=\begin{pmatrix}
1&\epsilon\\ \epsilon
&\epsilon\gamma
\end{pmatrix}.
\]
The largest such $\kappa$ is $\frac12\min_{h\in[m,L]}\mu_-(h)$, where $\mu_-(h)$ is
the smaller generalized eigenvalue of the pencil $(-G(h),S_0)$. Expanding
$\det(-G(h)-\mu S_0)=0$ yields
\[
  \mu^2-2\gamma\mu+\left(4h-\frac{h^2}{\epsilon(\gamma-\epsilon)}\right)=0,
\]
which implies that
\[
  \mu_-(h)=\gamma-\sqrt{q(h)},\quad q(h)=\gamma^2-4h+\frac{h^2}{\epsilon(\gamma-\epsilon)} .
\]
Minimizing $\mu_-$ over $[m,L]$ means maximizing the convex $q$, attained at an
endpoint, giving \eqref{eq:kappa}. Moreover $\mu_-(h)>0$ if and only if $4\epsilon(\gamma-\epsilon)>h$,
so $\kappa>0$ for all $h\le L$ if and only if \eqref{eq:star} holds.
Finally
$\epsilon\mapsto\epsilon(\gamma-\epsilon)$ peaks at $\epsilon=\gamma/2$ with value
$\gamma^2/4$, so \eqref{eq:star} is solvable if and only if $\gamma>\sqrt L$.
The proof is complete.
\end{proof}
\paragraph{Lyapunov bound.}
Next, we adopt a Lyapunov approach
to extend the linear contraction to the full contraction. First, we make
the following main assumption.

\begin{assumption}\label{ass:main}
We assume throughout the rest of the paper:
\begin{itemize}
\item[(A)] \emph{Controlled anchoring}: $\omega_0:=\|U-U_0\|_\infty$ and
  $\omega_1:=\|\nabla(U-U_0)\|_\infty$ are finite.
\item[(B)] \emph{Confinement}: $U_0$ is $m$-strongly convex and $L$-smooth, with
  minimizer $x_\ast$ and $U_0^\ast:=U_0(x_\ast)$.
\end{itemize}
\end{assumption}
Write $a(x)=e^{U(x)-U_0(x)}$, so the reference generator is
\[
  \mathcal A=a(x)\,v\cdot\nabla_x-\left(\gamma v+a(x)\nabla U_0\right)\cdot\nabla_v+\gamma\Delta_v .
\]
The plain energy is not self-controlling
($\mathcal A(U_0+\frac12\|v\|^2)=-\gamma\|v\|^2+\gamma d$ has no $-c(\cdot)$ on the
position part), so we add the hypocoercive cross term
\begin{equation}\label{eq:Gdef}
  \mathcal G(x,v):=\left(U_0(x)-U_0^\ast\right)+\frac12\|v\|^2+\eta_{0}\,\langle v,\nabla U_0(x)\rangle,
\end{equation}
where $\eta_{0}>0$ is chosen to be small.
Next, we verify a Foster--Lyapunov drift condition.
\begin{lemma}[Foster--Lyapunov drift]\label{lem:drift}
Under Assumption~\ref{ass:main}(B), there is $\bar{\eta}_0>0$ such that for all
$\eta_{0}\in(0,\bar{\eta}_0)$, we have
\begin{equation}\label{eq:equiv}
\frac12\bar{\mathcal G}\le\mathcal G \le \frac32\bar{\mathcal G},
\end{equation}
(in particular $\mathcal G\ge0$) where
\begin{equation}
\bar{\mathcal G}:=(U_0-U_0^\ast)+\frac12\|v\|^2,
\end{equation}
and
\begin{equation}\label{eq:drift}
  \mathcal A\mathcal G\ \le\ -c\,\mathcal G+b,\qquad c>0,\quad b=\gamma d,
\end{equation}
with $c$ depending only on $(\gamma,m,L,\eta_{0},a_{\min},a_{\max})$, where
$e^{-\omega_0}=:a_{\min}\le a\le a_{\max}:=e^{\omega_0}$ (finite under
Assumption~\ref{ass:main}(A)).
\end{lemma}

\begin{proof}
A direct computation gives
\[
  \mathcal A\mathcal G=-\gamma\|v\|^2+\gamma d
  +\eta_{0}\left(a\,v^{\top}\nabla^2U_0\,v-\gamma\,v\cdot\nabla U_0-a\|\nabla U_0\|^2\right).
\]
By using the bound $a\,v^{\top}\nabla^2U_0\,v\le a_{\max}L\|v\|^2$ and Young's
inequality $|\gamma v\cdot\nabla U_0|\le\frac{\gamma\zeta}{2}\|v\|^2+\frac{\gamma}{2\zeta}\|\nabla U_0\|^2$, we get
\[
  \mathcal A\mathcal G\le-\left(\gamma-\eta_{0} a_{\max}L-\frac{\eta_{0}\gamma\zeta}{2}\right)\|v\|^2
  -\eta_{0}\left(a_{\min}-\frac{\gamma}{2\zeta}\right)\|\nabla U_0\|^2+\gamma d .
\]
Choose $\zeta>\gamma/(2a_{\min})$ and $\eta_{0}$ small so the $\|v\|^2$ coefficient
$-c_v$ is negative and both bracketed factors are $>0$. The Polyak--{\L}ojasiewicz
bound $\|\nabla U_0\|^2\ge2m(U_0-U_0^\ast)$, which is implied by
$m$-strong convexity of $U_{0}$, see e.g. \cite[Theorem 2.1.10]{Nesterov2004}, then gives
\[
\mathcal A\mathcal G\le-c_0\bar{\mathcal G}+\gamma d,
\]
with
$c_0=\min\{2c_v,\,2m\eta_{0}(a_{\min}-\frac{\gamma}{2\zeta})\}$.
Finally, by $L$-smoothness of $U_0$, we have $\|\nabla U_0\|^2\le2L(U_0-U_0^\ast)$ which yields
\[
|\eta_{0}\langle v,\nabla U_0\rangle|\le\frac{\eta_{0}}{2}\|v\|^2+\eta_{0} L(U_0-U_0^\ast),
\]
and hence \eqref{eq:equiv} holds for $\eta_{0}\le\min\{\frac12,\frac1{2L}\}$; combined with
$\bar{\mathcal G}\ge\frac12\mathcal G$ this yields \eqref{eq:drift} with $c=c_0/2$.
\end{proof}
\paragraph{Perturbation estimates and the drift of $\mathcal H_t$.}
Under Assumption~\ref{ass:main}(A), since $a_t=e^{U(X_t)-U_0(X_t)}$ and
$|U-U_0|\le\omega_0$, we have
\[
  |a_t-1|\le\alpha_0:=e^{\omega_0}-1,
\]
and by \eqref{eq:Ft} and $\nabla e^{U-U_0}=e^{U-U_0}\nabla(U-U_0)$, we have
\[
  \|F_t\|\le a_{\max}\,\omega_1=:c_F=e^{\omega_0}\omega_1.
\]
By bounding the blocks of $E_t$ in \eqref{eq:Et} (operator norm, triangle inequality)
and using $\nabla^2U_0\preceq L\Id$, we obtain
\[
  \|E_t\|_{\mathrm{op}}
\le\alpha_0(1+L)+c_F\left(\left\|V_t'\right\|+\left\|\nabla U_0(X_t')\right\|\right).
\]
From $\bar{\mathcal G}=(U_0-U_0^\ast)+\frac12\|v\|^2$, $\|\nabla U_0\|^2\le2L(U_0-U_0^\ast)$,
and the equivalence \eqref{eq:equiv} ($\bar{\mathcal G}\le2\mathcal G$),
\[
  \left\|V_t'\right\|\le2\sqrt{\mathcal G_t'},\qquad \left\|\nabla U_0(X_t')\right\|\le2\sqrt L\,\sqrt{\mathcal G_t'},
\]
where
\[
\mathcal G_t':=\mathcal G(X_t',V_t'),
\]
and hence we obtain
\begin{equation}\label{eq:Ebound}
  \|E_t\|_{\mathrm{op}}\le C_1+C_2\sqrt{\mathcal G_t'},
  \qquad
  C_1:=\alpha_0(L+1),\quad C_2:=2c_F\left(1+\sqrt L\right),
\end{equation}
with $C_1\to0$, $C_2\to0$ as $(\omega_0,\omega_1)\to0$ (since $\alpha_0,c_F\to0$).
By differentiating $\mathcal H_t=W_t^{\top}SW_t$ along \eqref{eq:diffdyn}, we get
\begin{align}\label{dot:H:two:terms}
  \dot{\mathcal H}_t
  =W_t^{\top}\left(\left(M_0(H_t)\right)^{\top}S+S\,M_0(H_t)\right)W_t
  +W_t^{\top}\left(E_t^{\top}S+S\,E_t\right)W_t .
\end{align}
The first term on the right-hand side of \eqref{dot:H:two:terms} can be bounded by using Proposition~\ref{prop:linear}
(valid since $H_t\in[m\Id,L\Id]$ by \eqref{eq:Ht}) such that
\begin{align}
W_t^{\top}\left(\left(M_0(H_t)\right)^{\top}S+S\,M_0(H_t)\right)W_t\le-2\kappa\,\mathcal H_t.
\end{align}
Moreover, we can bound the second
term on the right-hand side of \eqref{dot:H:two:terms} as:
\begin{align}
W_t^{\top}\left(E_t^{\top}S+S\,E_t\right)W_t\le2\|S\|\,\|E_t\|_{\mathrm{op}}\|W_t\|^2\le\beta_0\|E_t\|_{\mathrm{op}}\mathcal H_t,
\end{align}
with $\beta_0:=2\|S\|/\lambda_{\min}(S)$. Hence, with
\begin{align}\label{eq:Kt}
K_t:=\beta_0\|E_t\|_{\mathrm{op}}\le\beta_0\left(C_1+C_2\sqrt{\mathcal G_t'}\right),
\end{align}
we have
\begin{equation}\label{eq:Hineq}
  \dot{\mathcal H}_t\ \le\ (-2\kappa+K_t)\,\mathcal H_t .
\end{equation}
\paragraph{Exponential weight.} For $\Psi=e^{\alpha\mathcal G}$ the diffusion produces
a carr\'e-du-champ correction (only $\gamma\Delta_v$ contributes),
\begin{equation}\label{eq:expid}
  \mathcal A\,e^{\alpha\mathcal G}
  =e^{\alpha\mathcal G}\left(\alpha\,\mathcal A\mathcal G+\gamma\alpha^2\|\nabla_v\mathcal G\|^2\right).
\end{equation}
With $\nabla_v\mathcal G=v+\eta_{0}\nabla U_0$, the bounds $\|v\|^2\le2\bar{\mathcal G}$,
$\|\nabla U_0\|^2\le2L\bar{\mathcal G}$ and $\bar{\mathcal G}\le2\mathcal G$ give
\begin{equation}\label{eq:carre}
  \gamma\|\nabla_v\mathcal G\|^2\le C_\Gamma\,\mathcal G,\qquad C_\Gamma:=8\gamma\left(1+\eta_{0}^2L\right).
\end{equation}
By combining \eqref{eq:expid}, \eqref{eq:carre} and the drift \eqref{eq:drift}, we conclude that
\begin{equation}\label{eq:expdrift}
  \mathcal A\,e^{\alpha\mathcal G}\le\alpha\,e^{\alpha\mathcal G}\left(-(c-\alpha C_\Gamma)\mathcal G+b\right),
  \qquad 0<\alpha<\frac{c}{C_\Gamma},\quad c':=c-\alpha C_\Gamma>0.
\end{equation}

The exponential drift inequality \eqref{eq:expdrift} has a
consequence that will be used repeatedly below: since the perturbation
$K_t$ in \eqref{eq:Kt} grows like $\sqrt{\mathcal G_t'}$, integrating
\eqref{eq:Hineq} produces exponentials of $\int_0^t\sqrt{\mathcal G_s'}\,ds$,
and the following lemma shows that such exponentials are integrable, with
an expectation that grows at most exponentially in $t$. Besides
controlling the perturbation, this is what justifies the martingale
argument in the proof of Theorem~\ref{thm:contraction} below.

\begin{lemma}[Exponential moment of the integrated potential]\label{lem:expmoment}
Assume Assumption~\ref{ass:main}\textup{(A)--(B)} and let $\alpha\in(0,c/C_\Gamma)$, so $c'=c-\alpha C_\Gamma>0$.
Then for every $\lambda\ge0$ and $t\ge0$,
\begin{equation}\label{eq:expmoment}
\E\left[\exp\left(\lambda\int_0^t\sqrt{\mathcal G_s'}\,ds\right)\,\Big|\,\mathcal F_0\right]
  \le e^{\alpha\mathcal G_0'}\,e^{\rho^\ast t},
  \qquad \rho^\ast=\rho^\ast(\lambda):=\alpha b+\frac{\lambda^2}{4\alpha c'} .
\end{equation}
In particular, if $\E\left[e^{\alpha\mathcal G_0'}\right]<\infty$ then
\[
\E\left[\exp\left(\lambda\int_0^t\sqrt{\mathcal G_s'}\,ds\right)\right]\le e^{\rho^\ast t}\,\E\left[e^{\alpha\mathcal G_0'}\right]<\infty.
\]
\end{lemma}

\begin{proof}
Fix $\lambda\ge0$, set
\begin{align}
\Psi_t'=e^{\alpha\mathcal G_t'},
\qquad
R_t:=\exp\left(\lambda\int_0^t\sqrt{\mathcal G_s'}\,ds-\rho^\ast t\right),
\qquad
Z_t:=\Psi_t'R_t\ge0.
\end{align}
Since $R_t$ has finite variation with
\[
dR_t=R_t\left(\lambda\sqrt{\mathcal G_t'}-\rho^\ast\right)\,dt,
\]
and
\[
d\Psi_t'=(\mathcal A\Psi)_t'\,dt+dM_t,
\]
where $M_t$ is a local martingale; cf.\ the
evolution in the proof of Theorem~\ref{thm:contraction}, we have
\begin{align}\label{drift:term:eqn}
  dZ_t=Z_t\left(\frac{(\mathcal A\Psi)_t'}{\Psi_t'}+\lambda\sqrt{\mathcal G_t'}-\rho^\ast\right)\,dt+R_t\,dM_t .
\end{align}
Inserting \eqref{eq:expdrift} as
\[
\frac{(\mathcal A\Psi)_t'}{\Psi_t'}\le-\alpha c'\mathcal G_t'+\alpha b,
\]
and completing the
square,
\[
-\alpha c'\mathcal G_t'+\lambda\sqrt{\mathcal G_t'}\le\frac{\lambda^2}{4\alpha c'},
\]
the drift term in \eqref{drift:term:eqn} can be bounded as:
\[
Z_t\left(\frac{(\mathcal A\Psi)_t'}{\Psi_t'}+\lambda\sqrt{\mathcal G_t'}-\rho^\ast\right)
\le Z_t\left(\alpha b+\frac{\lambda^2}{4\alpha c'}-\rho^\ast\right)=0.
\]
Thus, $Z_t$ is a
nonnegative local supermartingale; with a localizing sequence $(\tau_n)$ and by Fatou's lemma,
$\E[Z_t\mid\mathcal F_0]\le Z_0=e^{\alpha\mathcal G_0'}$. Since $\Psi_t'\ge1$,
\[
  \E\left[e^{\lambda\int_0^t\sqrt{\mathcal G_s'}ds}\,\Big|\,\mathcal F_0\right]
  =e^{\rho^\ast t}\E[R_t\mid\mathcal F_0]\le e^{\rho^\ast t}\E[Z_t\mid\mathcal F_0]\le e^{\alpha\mathcal G_0'}e^{\rho^\ast t}. \qedhere
\]
This completes the proof.
\end{proof}

Now we state the main theorem for the contraction of the continuous-time dynamics.

\begin{theorem}[Contraction]\label{thm:contraction}
Assume Assumption~\ref{ass:main}(A)--(B). With the constants
$\beta_0=2\|S\|/\lambda_{\min}(S)$, $C_1=\alpha_0(L+1)$,
$C_2=2c_F(1+\sqrt L)$ \textup{(}from \eqref{eq:Ebound}\textup{)},
$C_\Gamma=8\gamma(1+\eta_{0}^2L)$ (from \eqref{eq:carre}), $b=\gamma d$,
fix $\alpha\in(0,c/C_\Gamma)$ and set $c'=c-\alpha C_\Gamma>0$,
$\Psi_t':=e^{\alpha\mathcal G_t'}$ with $\mathcal G_t':=\mathcal G(X_t',V_t')$
the Lyapunov function \eqref{eq:Gdef} along the coupling process, and
$\Phi_t:=\mathcal H_t\,\Psi_t'$. Then,
\[
  \frac{d}{dt}\E[\Phi_t]\le-(2\kappa-\delta)\,\E[\Phi_t],
  \qquad
  \delta=\delta(\alpha):=\beta_0C_1+\alpha b+\frac{(\beta_0C_2)^2}{4\alpha c'},
\]
and consequently,
\[
  \lambda_{\min}(S)\,\E\left[\|W_t\|^2\right]\le\E[\Phi_t]\le e^{-(2\kappa-\delta)t}\,\E[\Phi_0].
\]
\end{theorem}

\begin{proof}
First of all, by \eqref{eq:equiv}, $\mathcal G\ge\frac12\bar{\mathcal G}\ge0$, so $\Psi_t'\ge1$,
$\mathcal H_t\ge0$, $\Phi_t\ge0$.

\textbf{Synchronous coupling.} Synchronous coupling drives both copies by the same Brownian
motion with identical additive velocity noise, so the increments cancel in the
difference \eqref{eq:diffdyn} and $W_t$ solves an ODE; hence $\mathcal H_t$ is
finite variation, $d\mathcal H_t=\dot{\mathcal H}_t\,dt$.

\textbf{Evolution of $\Phi_t$.}
By Dynkin's formula,
\[
d\Psi_t'=(\mathcal A\Psi)_t'\,dt+dM_t,
\]
with
$(\mathcal A\Psi)_t'=(\mathcal A e^{\alpha\mathcal G})(X_t',V_t')$ and $M_t$ a local
martingale. As $\mathcal H_t$ is finite variation, the product rule has no
cross-variation term, and thus
\[
d\Phi_t=\left(\Psi_t'\dot{\mathcal H}_t+\mathcal H_t(\mathcal A\Psi)_t'\right)dt+\mathcal H_t\,dM_t.
\]
Taking expectations (the stochastic integral is a true martingale by
Lemma~\ref{lem:expmoment}),
\[
  \frac{d}{dt}\E[\Phi_t]=\E[\Psi_t'\dot{\mathcal H}_t]+\E[\mathcal H_t(\mathcal A\Psi)_t'].
\]

\textbf{Two inputs.}
Multiplying the drift inequality \eqref{eq:Hineq}
\textup{(}derived above using Proposition~\ref{prop:linear} for the $M_0(H_t)$
part\textup{)} by $\Psi_t'\ge0$ gives
\[
\Psi_t'\dot{\mathcal H}_t\le(-2\kappa+K_t)\Phi_t,
\]
while \eqref{eq:expdrift} at the
primed point gives
\[
\frac{(\mathcal A\Psi)_t'}{\Psi_t'}\le-\alpha c'\mathcal G_t'+\alpha b.
\]
Since $\Phi_t\ge0$,
\[
  \frac{d}{dt}\E[\Phi_t]\le\E\!\left[\Phi_t\left(-2\kappa+K_t+\frac{(\mathcal A\Psi)_t'}{\Psi_t'}\right)\right].
\]

\textbf{Closing the nonlinearity.}
With $K_t\le\beta_0C_1+\beta_0C_2\sqrt{\mathcal G_t'}$,
and the inequality by
completing the square,
\[
  \beta_0C_2\,s-\alpha c'\,s^2
  =-\alpha c'\left(s-\frac{\beta_0C_2}{2\alpha c'}\right)^2+\frac{(\beta_0C_2)^2}{4\alpha c'}
  \le\frac{(\beta_0C_2)^2}{4\alpha c'},
\]
applying to $s:=\sqrt{\mathcal G_t'}\ge0$, we get
\[
-2\kappa+K_t+\frac{(\mathcal A\Psi)_t'}{\Psi_t'}\le-2\kappa+\beta_0C_1+\alpha b+\frac{(\beta_0C_2)^2}{4\alpha c'}=-(2\kappa-\delta),
\]
which yields that
\[
\frac{d}{dt}\E[\Phi_t]\le-(2\kappa-\delta)\E[\Phi_t].
\]

\textbf{Conclusion.} Gr\"onwall's inequality gives
\[
\E[\Phi_t]\le e^{-(2\kappa-\delta)t}\E[\Phi_0],
\]
and $\Psi_t'\ge1$, $S\succ0$ give
\[
\Phi_t\ge\mathcal H_t\ge\lambda_{\min}(S)\|W_t\|^2,
\]
which proves the claim.
The proof is complete.
\end{proof}

In Theorem~\ref{thm:contraction}, since $C_1,C_2\to0$ as $(\omega_0,\omega_1)\to0$, the optimized
deficit satisfies $\delta_\star:=\inf_{\alpha\in(0,c/C_\Gamma)}\delta(\alpha)\to0$;
hence for sufficiently small anchoring mismatch $\delta_\star<2\kappa$ and the
coupled difference contracts in $\mathcal W_2$ at rate
$\frac12(2\kappa-\delta_\star)>0$. A rigorous quantitative form of this limit is
given in Proposition~\ref{prop:smallmismatch} below.

\begin{proposition}[Small-mismatch rate]\label{prop:smallmismatch}
With the notation of Theorem~\ref{thm:contraction}, the choice
$\alpha_\circ:=\dfrac{\beta_0C_2}{2\sqrt{b\,c'(\alpha_\circ)}}<c/(2C_\Gamma)$, we have
\begin{equation}\label{eq:smallmismatch}
\delta(\alpha_\circ)=\beta_0C_1+\beta_0C_2\sqrt{b/c'(\alpha_\circ)}
  \le\beta_0C_1+\beta_0C_2\sqrt{2b/c}.
\end{equation}
In particular $\delta_\star\le\beta_0C_1+\beta_0C_2\sqrt{2b/c}\to0$ as
$(\omega_0,\omega_1)\to0$, so for all sufficiently small mismatch $\delta_\star<2\kappa$.
\end{proposition}

\begin{proof}
For fixed $c'$, by AM--GM inequality, $\alpha\mapsto\alpha b+\frac{(\beta_0C_2)^2}{4\alpha c'}$ is minimized
at $\alpha_\circ=\beta_0C_2/(2\sqrt{bc'})$ with value $\beta_0C_2\sqrt{b/c'}$. As $C_2\to0$, $\alpha_\circ\to0$. Thus, for sufficiently small $C_{2}$,
we have $\alpha_\circ<c/(2C_\Gamma)$ and $c'(\alpha_\circ)=c-\alpha_\circ C_\Gamma>c/2$;
substituting the lower bound $c'(\alpha_\circ)>c/2$ gives the estimate \eqref{eq:smallmismatch}, and
$C_1,C_2\to0$ gives $\delta_\star\to0$.
This completes the proof.
\end{proof}

Theorem~\ref{thm:contraction} yields the following $\mathcal W_2$ convergence
guarantee for UALD \eqref{eq:UALD:1}-\eqref{eq:UALD:2} to its invariant law $\pi$.

\begin{corollary}\label{cor:Wasserstein}
Under the assumptions of Theorem~\ref{thm:contraction}, if
$\E_{(X_0',V_0')\sim\pi}[\Phi_0]<\infty$ then for all $t\ge0$,
\[
  \mathcal W_2\left(\mathrm{Law}(X_t,V_t),\pi\right)
  \le\left(\frac{\E_{(X_0',V_0')\sim\pi}[\Phi_0]}{\lambda_{\min}(S)}\right)^{1/2}e^{-(\kappa-\delta/2)t}.
\]
\end{corollary}

\begin{proof}
Take $(X_0',V_0')\sim\pi$, so that $\mathrm{Law}(X_t',V_t')=\pi$ for all $t\geq 0$. By the
definition of 2-Wasserstein distance and applying Theorem~\ref{thm:contraction},
\[
  \left(\mathcal W_2\left(\mathrm{Law}(X_t,V_t),\pi\right)\right)^2
  \le\E_{(X_0',V_0')\sim\pi}\|W_t\|^2
  \le\frac{\E_{(X_0',V_0')\sim\pi}[\Phi_0]}{\lambda_{\min}(S)}\,e^{-(2\kappa-\delta)t}.
\]
This completes the proof.
\end{proof}

\paragraph{Sampling heavy-tailed distributions.}

Next, we show that UALD can sample heavy-tailed distributions
just as ALD in \cite{GNZZ2025}.
Following \cite{GNZZ2025}, let
$q:\mathbb{R}^{d}\rightarrow[1,\infty)$ be $C^{1}$ with
$\int_{\mathbb{R}^{d}}\left(q(x)\right)^{-1-\beta}dx<\infty$ for some $\beta>0$, and take
\[
  U_{0}(x):=\beta\log q(x),\qquad U(x):=(\beta+1)\log q(x),
  \qquad\text{so that}\qquad
  \pi_{X}(x)\propto \left(q(x)\right)^{-(1+\beta)},
\]
which has polynomial tails. Since $e^{U-U_{0}}=q$ and
$\nabla U_{0}=\beta\nabla q/q$, the UALD
\eqref{eq:VL:modified:2}--\eqref{eq:XL:modified:2} becomes
\begin{equation}\label{eq:UALD:heavy}
  dV_{t}=-\gamma V_{t}\,dt-\beta\nabla q(X_{t})\,dt+\sqrt{2\gamma}\,dB_{t},
  \qquad
  dX_{t}=q(X_{t})\,V_{t}\,dt .
\end{equation}
Note that the true force $-\nabla U=-(\beta+1)\nabla q/q$ \emph{vanishes} in
the tails---the classical reason why the classical kinetic Langevin dynamics
\eqref{eq:VL}--\eqref{eq:XL} fails on heavy-tailed targets---whereas the
anchored force $-\beta\nabla q$ does not degenerate; for the multivariate
Student's~$t$ target of \cite[Example 2]{GNZZ2025}, i.e.
$q(x)=1+\frac1\nu(x-\mu_{t})^{\top}\Sigma^{-1}(x-\mu_{t})$ and
$\beta=\frac{d+\nu}{2}-1$, it is \emph{linear},
$-\beta\nabla q(x)=-\frac{2\beta}{\nu}\Sigma^{-1}(x-\mu_{t})$.

We work under the following
assumption, which parallels the class of heavy-tailed targets treated in
\cite[Corollary 13]{GNZZ2025}.

\begin{assumption}[Strongly convex anchoring of a heavy-tailed target]
\label{ass:HQ}
Let $q\in C^{2}(\R^{d})$ satisfy $\min_{x}q(x)=1$
\textup{(}a normalization, as multiplying $q$ by a constant only rescales
time in \eqref{eq:UALD:heavy}\textup{)} and
\[
  m_{q}\,\Id\ \preceq\ \nabla^{2}q(x)\ \preceq\ L_{q}\,\Id,
  \qquad x\in\R^{d},
\]
for some $0<m_{q}\le L_{q}<\infty$. Let $\beta>1$ satisfy
\begin{equation}\label{eq:HQcond}
  \beta\ >\ 1+\frac{(d+4)L_{q}}{2m_{q}},
\end{equation}
and set $U_{0}:=\beta\log q$, $U:=(\beta+1)\log q$, so that
$\pi_{X}\propto q^{-(1+\beta)}$ has polynomial tails of order $2(1+\beta)$
\textup{(}integrable, since \eqref{eq:HQcond} implies
$2(1+\beta)>d$\textup{)}.
\end{assumption}

Two elementary consequences of Assumption~\ref{ass:HQ} will be used
repeatedly: by the Polyak--{\L}ojasiewicz inequality and $L_{q}$-smoothness
of $q$,
\begin{equation}\label{eq:qbounds:gen}
  2m_{q}\left(q-1\right)\ \le\ \|\nabla q\|^{2}\ \le\ 2L_{q}\left(q-1\right)
  \ \le\ c_{0}^{2}\,q,
  \qquad c_{0}^{2}:=2L_{q},
  \qquad
  d\,m_{q}\le\Delta q\le d\,L_{q} .
\end{equation}
For the Student's~$t$ target, $q(x)=1+\|x\|^{2}/\nu$ and
$\beta=\frac{d+\nu}2-1$, so that $m_{q}=L_{q}=\frac2\nu$ and
\eqref{eq:HQcond} reads $\frac{d+\nu}2-1>1+\frac{d+4}2$, i.e.\ $\nu>8$,
a dimension-free condition. We further define
\begin{equation}\label{eq:HQconst}
  \kappa_{0}:=2(\beta-1)m_{q}-dL_{q}>0,
  \qquad
  \theta:=1-\frac{2c_{0}^{2}}{\kappa_{0}}
  =1-\frac{4L_{q}}{\kappa_{0}}\in(0,1),
\end{equation}
where positivity of $\kappa_{0}$ and $\theta$ is exactly
\eqref{eq:HQcond}. Throughout, $\mu:=\pi_{X}$ with density
$\rho\propto q^{-(1+\beta)}$, $M:=\mathcal N(0,\Id)$, $\pi=\mu\otimes M$;
norms without subscript are those of $L^{2}(\pi)$, and norms of functions
of $x$ alone those of $L^{2}(\mu)$. Let
$\Pi f:=\int_{\mathbb{R}^{d}}f(\cdot,v)M(dv)$, and
\[
  T:=q(x)\left[v\cdot\nabla_{x}-\nabla U_{0}(x)\cdot\nabla_{v}\right],
  \qquad
  S:=\gamma\left(\Delta_{v}-v\cdot\nabla_{v}\right),
\]
so that the generator of \eqref{eq:UALD:heavy} is $\mathcal A=T+S$.

Our proof of exponential convergence follows the
$L^{2}$-hypocoercivity method of Dolbeault, Mouhot and Schmeiser
\cite{DMS2015}, which rests on three ingredients: the algebraic structure
of the splitting $\mathcal A=T+S$ of the generator into a transport part
and a collision part (microscopic coercivity, the identity $\Pi T\Pi=0$,
and the properties of the auxiliary operator $A$), a macroscopic
coercivity, i.e.\ spectral gap, estimate for the position marginal, and a
weighted regularity estimate. We establish these in
Lemmas~\ref{lem:H4structure}, \ref{lem:H4macro} and~\ref{lem:H4reg}
respectively, beginning with the algebraic structure.

\begin{lemma}[Structure]\label{lem:H4structure}
Under Assumption~\ref{ass:HQ}, the following hold on the core of smooth
compactly supported functions.
\begin{itemize}
\item[(i)] $T$ is anti-symmetric and $S$ is symmetric in $L^{2}(\pi)$, with
$-\langle Sf,f\rangle=\gamma\|\nabla_{v}f\|^{2}\ge\gamma\|(1-\Pi)f\|^{2}$.
\item[(ii)] $\Pi T\Pi=0$.
\item[(iii)] The operator $\Lambda:=(T\Pi)^{\ast}(T\Pi)$ acts on functions
$u=u(x)$ as
\[
  \Lambda u=-q^{2}\Delta u+(\beta-1)\,q\,\nabla q\cdot\nabla u
  =-\frac1\rho\nabla\cdot\left(q^{2}\rho\,\nabla u\right),
\]
is symmetric and nonnegative on $L^{2}(\mu)$ with
$\langle\Lambda u,u\rangle_{L^{2}(\mu)}=\int_{\mathbb{R}^{d}} q^{2}\|\nabla u\|^{2}d\mu$,
and $-\Pi TTu=\Lambda u$.
\item[(iv)] With $A:=(1+\Lambda)^{-1}(T\Pi)^{\ast}$ one has
$Af=A(1-\Pi)f\in\Pi L^{2}(\pi)$, $TAf\in(1-\Pi)L^{2}(\pi)$, and
\[
  \|Af\|\le\frac12\|(1-\Pi)f\|,
  \qquad
  \|TAf\|\le\|(1-\Pi)f\| .
\]
\item[(v)] $AS=-\gamma A$.
\end{itemize}
\end{lemma}

\begin{proof}
(i) Anti-symmetry of $T$ can be easily checked, valid for any $C^{1}$ function $q>0$; the
statement for $S$ is Gaussian integration by parts plus the Gaussian
Poincar\'e inequality for each fixed $x$.

(ii) $T\Pi f=q\,v\cdot\nabla(\Pi f)$ is odd in $v$.

(iii) For $u=u(x)$,
\[
  TTu=q\left(v\cdot\nabla q\right)\left(v\cdot\nabla u\right)
  +q^{2}\,v^{\top}\nabla^{2}u\,v-q^{2}\nabla U_{0}\cdot\nabla u ,
\]
and averaging over $v$ with $\int_{\mathbb{R}^{d}} v_{i}v_{j}M(dv)=\delta_{ij}$ and
$\nabla U_{0}=\beta\nabla q/q$ gives the formula; the divergence form is a
direct computation with $\rho\propto q^{-(1+\beta)}$.

(iv) Since $\Pi T\Pi=0$ we have $(T\Pi)^{\ast}=(T\Pi)^{\ast}(1-\Pi)$;
$u:=Af$ solves $u+\Lambda u=(T\Pi)^{\ast}(1-\Pi)f$.
$u=Af=(1+\Lambda)^{-1}(T\Pi)^{\ast}(1-\Pi)f$
is a function of $x$ alone, so that $u=\Pi u$, which gives
$Af=A(1-\Pi)f\in\Pi L^{2}(\pi)$, and $TAf=T\Pi u$ is odd in $v$ by (ii),
so that $TAf\in(1-\Pi)L^{2}(\pi)$. Taking the $L^{2}(\pi)$ inner product
of the identity $u+\Lambda u=(T\Pi)^{\ast}(1-\Pi)f$ with $u$, we
obtain
\begin{align}\label{above:eqn}
  \|u\|^{2}+\|T\Pi u\|^{2}
  =\left\langle(T\Pi)^{\ast}(1-\Pi)f,u\right\rangle
  =\left\langle(1-\Pi)f,T\Pi u\right\rangle
  \le\|(1-\Pi)f\|\,\|T\Pi u\|,
\end{align}
where we used
$\langle\Lambda u,u\rangle=\langle(T\Pi)^{\ast}(T\Pi)u,u\rangle
=\|T\Pi u\|^{2}$ on the left-hand side of \eqref{above:eqn} and the definition of the adjoint
together with the Cauchy--Schwarz inequality on the right-hand side of \eqref{above:eqn}.
Dropping $\|u\|^{2}\ge0$ and dividing by $\|T\Pi u\|$ (the bound being
trivial if $T\Pi u=0$) gives $\|TAf\|=\|T\Pi u\|\le\|(1-\Pi)f\|$, while
keeping $\|u\|^{2}$ and using $xy-y^{2}\le\frac14x^{2}$ for $x,y\ge0$,
with $x=\|(1-\Pi)f\|$ and $y=\|T\Pi u\|$, gives
\[
\|Af\|^{2}=\|u\|^{2}\le\|(1-\Pi)f\|\,\|T\Pi u\|-\|T\Pi u\|^{2}
\le\frac14\|(1-\Pi)f\|^{2}.
\]

(v) With $m(x):=\int_{\mathbb{R}^{d}} vg(x,v)M(dv)$ one has
$\Pi Tg=q[\nabla\cdot m-\nabla U_{0}\cdot m]$, and since
$Sv_{i}=-\gamma v_{i}$, the first moment of $Sg$ equals $-\gamma m$; hence
$\Pi TS=-\gamma\Pi T$ and $AS=-\gamma A$.
\end{proof}

The second ingredient is the macroscopic coercivity, i.e.\ a
spectral gap for the operator $\Lambda$ of
Lemma~\ref{lem:H4structure}(iii), which amounts to a Poincar\'e
inequality for $\mu$ with the weight $q^{2}$ in the Dirichlet form. Since
$\mu$ has only polynomial tails, we establish this inequality through the
Lyapunov-function criterion for Poincar\'e inequalities of
\cite{BBCG2008}, for which the weight $q^{2}$ provides the required
confinement.

\begin{lemma}[Macroscopic coercivity]\label{lem:H4macro}
Under Assumption~\ref{ass:HQ} there is $\lambda_{M}>0$ such that for every
$u\in L^{2}(\mu)$ with $\int_{\mathbb{R}^{d}} u\,d\mu=0$,
\[
  \|T\Pi u\|^{2}=\int_{\R^{d}}q^{2}\|\nabla u\|^{2}\,d\mu
  \ \ge\ \lambda_{M}\int_{\R^{d}}u^{2}\,d\mu .
\]
\end{lemma}

\begin{proof}
The identity follows from $T\Pi u=qv\cdot\nabla u$. For the inequality, let
\[
G:=-\Lambda=q^{2}\Delta-(\beta-1)q\nabla q\cdot\nabla
\]
be the
$\mu$-symmetric Markov generator with Dirichlet form
$\int_{\mathbb{R}^{d}} q^{2}\|\nabla u\|^{2}d\mu$, and take the Lyapunov function $W:=q$.
By \eqref{eq:qbounds:gen},
\begin{align*}
  GW&=q^{2}\Delta q-(\beta-1)q\|\nabla q\|^{2}
  \\
  &\le dL_{q}q^{2}-2(\beta-1)m_{q}\left(q-1\right)q
  =-\kappa_{0}\,q\,W+2(\beta-1)m_{q}\,W .
\end{align*}
Hence, with $R_{0}:=\left(2(\beta-1)m_{q}+\kappa_{0}\right)/\kappa_{0}$ and
the compact set $K:=\{q\le R_{0}\}$ \textup{(}compact since strong
convexity forces $q(x)\rightarrow\infty$\textup{)},
\[
  GW\ \le\ -\kappa_{0}W+b\,\ones_{K},
  \qquad b:=\sup_{K}\left(GW+\kappa_{0}W\right)<\infty.
\]
Since $W\ge1$ and $\mu$ restricted to $K$ has a smooth positive density and
the diffusion matrix $q^{2}\Id$ is uniformly elliptic on $K$, a local
Poincar\'e inequality holds on $K$, and the Lyapunov criterion of
\cite{BBCG2008} yields the spectral gap $\lambda_{M}>0$.
\end{proof}

The third ingredient is the boundedness of the auxiliary
operator $AT(1-\Pi)$ on $L^{2}(\pi)$, which in the framework of
\cite{DMS2015} reduces to a weighted $H^{2}$-type elliptic regularity
estimate for the resolvent $(1+\Lambda)^{-1}$. Because of the weight
$q^{2}$ in $\Lambda$ and the polynomial tails of $\mu$, this estimate is
not a standard one, and we prove it directly, keeping track of the
dependence of the constant on $d$, $\beta$, $m_{q}$ and $L_{q}$.

\begin{lemma}[Weighted regularity]\label{lem:H4reg}
Under Assumption~\ref{ass:HQ}, there is an explicit
$C_{M}<\infty$ that depends on $d,\beta,m_{q},L_{q}$ such that
\[
  \left\|AT(1-\Pi)f\right\|\ \le\ C_{M}\,\|(1-\Pi)f\|,
  \qquad f\in L^{2}(\pi).
\]
\end{lemma}

\begin{proof}
We divide the proof into several steps.

\textbf{Step 0 (reduction).} As
$\left(AT(1-\Pi)\right)^{\ast}=-(1-\Pi)T\,A^{\ast}$ with
$A^{\ast}w=T\Pi(1+\Lambda)^{-1}\Pi w$, it suffices to prove
\begin{equation}\label{eq:H4goal}
  \left\|(1-\Pi)TTu\right\|\ \le\ C_{M}\left(\|u\|+\|\Lambda u\|\right)
\end{equation}
for $u:=(1+\Lambda)^{-1}\Pi w$, since $\|u\|\vee\|\Lambda u\|\le\|\Pi w\|$
by the spectral theorem; indeed, here, $\Lambda$ denotes the
nonnegative self-adjoint operator on $L^{2}(\mu)$ associated with the
closed quadratic form $u\mapsto\int_{\mathbb{R}^{d}}q^{2}\|\nabla u\|^{2}d\mu$,
i.e.\ the Friedrichs extension of the symmetric nonnegative operator of
Lemma~\ref{lem:H4structure}(iii), see \cite[Theorem X.23]{ReedSimon1975};
by the functional-calculus form of the spectral theorem for self-adjoint
operators \cite[Theorem VIII.5]{ReedSimon1980}, we have
$\|(1+\Lambda)^{-1}\|_{\mathrm{op}}\le\sup_{s\ge0}(1+s)^{-1}=1$ and
$\|\Lambda(1+\Lambda)^{-1}\|_{\mathrm{op}}\le\sup_{s\ge0}s(1+s)^{-1}\le1$,
which gives $\|u\|\le\|\Pi w\|$ and $\|\Lambda u\|\le\|\Pi w\|$. We prove \eqref{eq:H4goal} for smooth compactly
supported $u$ and conclude by approximation (Dirichlet resolvents
on balls and passage to the limit). Write
\begin{align*}
  &N:=\|u\|_{L^{2}(\mu)},\quad L:=\|\Lambda u\|_{L^{2}(\mu)},\qquad
  I:=\int_{\mathbb{R}^{d}} q^{2}u^{2}d\mu,\qquad
  \\
  &G_{s}:=\int_{\mathbb{R}^{d}} q^{s}\|\nabla u\|^{2}d\mu,\qquad
  H:=\int_{\mathbb{R}^{d}} q^{4}\|\nabla^{2}u\|_{F}^{2}\,d\mu .
\end{align*}

\textbf{Step 1 (chaos decomposition).} By Lemma~\ref{lem:H4structure}(iii),
\[
(1-\Pi)TTu=v^{\top}R(x)v-\operatorname{tr}R(x),
\]
with
$R:=q\,\mathrm{sym}(\nabla q\otimes\nabla u)+q^{2}\nabla^{2}u$, and since
$\E\left[\left(v^{\top}Rv-\operatorname{tr}R\right)^{2}\right]=2\|R\|_{F}^{2}$ for symmetric $R$
and $v\sim\mathcal N(0,\Id)$, the bound \eqref{eq:qbounds:gen} gives
\begin{equation}\label{eq:chaosbound}
  \left\|(1-\Pi)TTu\right\|^{2}
  \ \le\ 4c_{0}^{2}\,G_{3}+4H .
\end{equation}

\textbf{Step 2 (weighted Hardy inequality).} Integrating
$0=\int_{\mathbb{R}^{d}}\nabla\cdot\left(u^{2}q^{1-\beta}\nabla q\right)dx$,
\[
  (\beta-1)\int_{\mathbb{R}^{d}} u^{2}q^{-\beta}\|\nabla q\|^{2}dx
  =2\int_{\mathbb{R}^{d}} u\,q^{1-\beta}\nabla u\cdot\nabla q\,dx
  +\int_{\mathbb{R}^{d}} u^{2}q^{1-\beta}\Delta q\,dx .
\]
Bounding the left-hand side from below by
$2(\beta-1)m_{q}\left(I-J\right)$ with $J:=\int_{\mathbb{R}^{d}} qu^{2}d\mu\ge0$
\textup{(}by \eqref{eq:qbounds:gen}\textup{)}, the last term by
$dL_{q}I$, the middle term by
$c_{0}I^{1/2}G_{3}^{1/2}$, and $J\le I^{1/2}N$,
\begin{equation}\label{eq:hardy}
  I^{1/2}\ \le\ \frac{2c_{0}}{\kappa_{0}}\,G_{3}^{1/2}+C_{q}\,N,
  \qquad
  C_{q}:=\frac{2(\beta-1)m_{q}}{\kappa_{0}} .
\end{equation}

\textbf{Step 3 (weighted gradient bound).} 
For smooth compactly supported $u$ and
$\varphi$, the divergence form of $\Lambda$ in
Lemma~\ref{lem:H4structure}(iii) and an integration by parts (with no
boundary terms, since $\varphi$ has compact support) give
\[
  \int_{\mathbb{R}^{d}}(\Lambda u)\,\varphi\,d\mu
  =-\int_{\mathbb{R}^{d}}\nabla\cdot\left(q^{2}\rho\,\nabla u\right)\varphi\,dx
  =\int_{\mathbb{R}^{d}}q^{2}\,\nabla u\cdot\nabla\varphi\,d\mu .
\]
Choosing $\varphi:=qu$ and using $\nabla(qu)=q\nabla u+u\nabla q$,
\[
  \int_{\mathbb{R}^{d}}(\Lambda u)\,qu\,d\mu
  =\int_{\mathbb{R}^{d}}q^{3}\|\nabla u\|^{2}d\mu
  +\int_{\mathbb{R}^{d}}q^{2}u\,\nabla u\cdot\nabla q\,d\mu
  =G_{3}+\int_{\mathbb{R}^{d}}q^{2}u\,\nabla u\cdot\nabla q\,d\mu .
\]
Hence, by the Cauchy--Schwarz inequality, $\|qu\|_{L^{2}(\mu)}=I^{1/2}$
and $\|\nabla q\|\le c_{0}\sqrt q$ from \eqref{eq:qbounds:gen},
\[
  G_{3}
  \le\|\Lambda u\|_{L^{2}(\mu)}\,\|qu\|_{L^{2}(\mu)}
  +c_{0}\int_{\mathbb{R}^{d}}\left(q|u|\right)\left(q^{3/2}\|\nabla u\|\right)d\mu
  \le LI^{1/2}+c_{0}I^{1/2}G_{3}^{1/2}.
\]
Solving this quadratic inequality
for $G_{3}^{1/2}$ and using $\sqrt{a+b}\le\sqrt a+\sqrt b$ gives
\[
  G_{3}^{1/2}
  \le\frac12\left(c_{0}I^{1/2}+\sqrt{c_{0}^{2}I+4LI^{1/2}}\right)
  \le c_{0}I^{1/2}+L^{1/2}I^{1/4}.
\]

\textbf{Step 4 (bootstrap).} Substituting into \eqref{eq:hardy} and using
$\theta=1-2c_{0}^{2}/\kappa_{0}$ and Young's inequality,
\[
  \theta I^{1/2}\le\frac{2c_{0}}{\kappa_{0}}L^{1/2}I^{1/4}+C_{q}N
  \le\frac\theta2I^{1/2}+\frac{2c_{0}^{2}}{\theta\kappa_{0}^{2}}L+C_{q}N,
\]
so that
\begin{equation}\label{eq:Ibound}
  I^{1/2}\le K_{1}L+K_{2}N,
  \qquad
  G_{3}^{1/2}\le K_{3}L+K_{4}N,
\end{equation}
with
\[
K_{1}:=\frac{4c_{0}^{2}}{\theta^{2}\kappa_{0}^{2}},
\quad
K_{2}:=\frac{2C_{q}}{\theta},
\quad
K_{3}:=\left(c_{0}+\frac12\right)K_{1}+\frac12,
\quad
K_{4}:=\left(c_{0}+\frac12\right)K_{2}.
\]

\textbf{Step 5 (weighted Hessian bound).}
From $q^{2}\Delta u=(\beta-1)q\nabla q\cdot\nabla u-\Lambda u$ and
\eqref{eq:qbounds:gen},
\[
\int_{\mathbb{R}^{d}} q^{4}(\Delta u)^{2}d\mu\le2L^{2}+2(\beta-1)^{2}c_{0}^{2}G_{3}.
\]
With $w:=q^{4}\rho$ and two integrations by parts,
\[
  \int_{\mathbb{R}^{d}} q^{4}(\Delta u)^{2}d\mu
  =H+\sum_{i,j}\int_{\mathbb{R}^{d}}\left[\partial_{i}u\,\partial_{ij}u\,\partial_{j}w
  -\partial_{i}u\,\partial_{jj}u\,\partial_{i}w\right]dx ,
\]
and since $|\nabla w|\le c_{0}|3-\beta|q^{7/2}\rho$ and
$|\Delta u|\le\sqrt d\|\nabla^{2}u\|_{F}$, the error terms are bounded by
$(1+\sqrt d)c_{0}|\beta-3|G_{3}^{1/2}H^{1/2}$; Young's inequality gives
\begin{equation}\label{eq:Hbound}
  H\le4L^{2}+K_{5}G_{3},
  \qquad
  K_{5}:=\left[4(\beta-1)^{2}+(1+\sqrt d)^{2}(\beta-3)^{2}\right]c_{0}^{2}.
\end{equation}

\textbf{Step 6 (conclusion).} Combining
\eqref{eq:chaosbound}, \eqref{eq:Ibound}, \eqref{eq:Hbound},
\[
  \left\|(1-\Pi)TTu\right\|^{2}
  \le16L^{2}+4(c_{0}^{2}+K_{5})\left(K_{3}L+K_{4}N\right)^{2},
\]
which is \eqref{eq:H4goal} with
$C_{M}^{2}:=16+4(c_{0}^{2}+K_{5})(K_{3}+K_{4})^{2}$.
This completes the proof.
\end{proof}

In the next theorem, we provide an explicit exponential $L^2$ convergence for heavy-tailed targets.

\begin{theorem}[Exponential $L^{2}$ convergence for heavy-tailed targets]
\label{thm:heavyL2}
Let Assumption~\ref{ass:HQ} hold and let $(P_{t})_{t\ge0}$ be the
semigroup of the UALD \eqref{eq:UALD:heavy} on $L^{2}(\pi)$,
$\pi\propto q^{-(1+\beta)}e^{-\|v\|^{2}/2}$. With $\lambda_{M}$ and
$C_{M}$ the constants of Lemmas~\ref{lem:H4macro} and~\ref{lem:H4reg}, set
\[
  T_{M}:=\frac{\lambda_{M}}{1+\lambda_{M}},
  \qquad
  \varepsilon_{0}
  :=\min\left\{1,\ \frac\gamma2,\
  \frac{\gamma T_{M}}{2\left(C_{M}+\gamma/2\right)^{2}}\right\},
  \qquad
  \lambda:=\frac{\varepsilon_{0}T_{M}}{3} .
\]
Then for every $f\in L^{2}(\pi)$ and all $t\ge0$,
\[
  \left\|P_{t}f-\pi(f)\right\|_{L^{2}(\pi)}
  \ \le\ \sqrt3\,e^{-\lambda t}\,\left\|f-\pi(f)\right\|_{L^{2}(\pi)} .
\]
\end{theorem}

\begin{proof}
Without loss of generality $\pi(f)=0$. Write $f_{t}:=P_{t}f$, so
$\partial_{t}f_{t}=(T+S)f_{t}$ on a core, and define, following
\cite{DMS2015}, the modified entropy
$H_{\varepsilon}[f]:=\frac12\|f\|^{2}-\varepsilon\langle Af,f\rangle$
with $\varepsilon\in(0,\varepsilon_{0}]$. By
Lemma~\ref{lem:H4structure}(iv),
$|\langle Af,f\rangle|\le\frac14\|f\|^{2}$, so
$\frac14\|f\|^{2}\le H_{\varepsilon}[f]\le\frac34\|f\|^{2}$. Using the
anti-symmetry of $T$, symmetry of $S$, $S(Af)=0$, and writing
$X:=\|(1-\Pi)f\|$, $Y:=\|\Pi f\|$:
\[
  \frac{d}{dt}H_{\varepsilon}[f_{t}]
  =\langle Sf,f\rangle
  -\varepsilon\left[\langle AT\Pi f,f\rangle
  +\langle AT(1-\Pi)f,f\rangle
  +\langle ASf,f\rangle
  -\langle TAf,f\rangle\right],
\]
where, by Lemmas~\ref{lem:H4structure}--\ref{lem:H4reg}:
$\langle Sf,f\rangle\le-\gamma X^{2}$;
$\langle AT\Pi f,f\rangle
=\langle(1+\Lambda)^{-1}\Lambda\Pi f,\Pi f\rangle\ge T_{M}Y^{2}$
\textup{(}spectral calculus and Lemma~\ref{lem:H4macro}, as $\Pi f$ has
mean zero\textup{)};
$|\langle AT(1-\Pi)f,f\rangle|\le C_{M}XY$;
$|\langle ASf,f\rangle|=\gamma|\langle Af,\Pi f\rangle|\le\frac\gamma2XY$;
$|\langle TAf,f\rangle|\le X^{2}$. Hence,
\[
  \frac{d}{dt}H_{\varepsilon}[f_{t}]
  \le-(\gamma-\varepsilon)X^{2}-\varepsilon T_{M}Y^{2}
  +\varepsilon\left(C_{M}+\frac\gamma2\right)XY .
\]
Since $\varepsilon\le\gamma/2$ and, by Young's inequality and the choice
of $\varepsilon_{0}$,
$\varepsilon(C_{M}+\frac\gamma2)XY\le\frac\gamma4X^{2}
+\frac{\varepsilon T_{M}}2Y^{2}$, we conclude
\[
  \frac{d}{dt}H_{\varepsilon}[f_{t}]
  \le-\frac\gamma4X^{2}-\frac{\varepsilon T_{M}}2Y^{2}
  \le-\frac{\varepsilon T_{M}}2\|f_{t}\|^{2}
  \le-\frac{2\varepsilon T_{M}}3H_{\varepsilon}[f_{t}] ,
\]
and Gr\"onwall's inequality with the norm equivalence yields
\[
\|f_{t}\|^{2}\le3\|f_{0}\|^{2}e^{-\frac{2\varepsilon_{0}T_{M}}3t}.
\]
This completes the proof.
\end{proof}

\begin{remark}\label{rem:heavyscope}
Assumption~\ref{ass:HQ} is exactly the class of anchorings considered
in \cite[Corollary 13]{GNZZ2025} for the overdamped dynamics: $q$ strongly
convex with Lipschitz gradient; the bound $\|\nabla q\|\le c_{0}\sqrt q$
assumed there is automatic from smoothness, by \eqref{eq:qbounds:gen}. The
targets covered are the convex-type heavy-tailed measures
$\pi_{X}\propto q^{-(1+\beta)}$; for the Student's~$t$ distribution
\textup{(}$m_{q}=L_{q}$\textup{)} condition \eqref{eq:HQcond} reads
$\nu>8$, independently of dimension.
\end{remark}

\subsection{Discrete-Time Analysis}\label{sec:discrete}

In this section, we obtain finite-time non-asymptotic convergence guarantees for the discrete-time UALMC algorithm \eqref{eq:scheme}:
\begin{equation}
\begin{aligned}
  v_{k+1}&=\psi_0(\eta)\,v_k-\psi_1(\eta)\,g_k+\sqrt{2\gamma}\,\xi_{k+1},\\
  x_{k+1}&=x_k+a_k\left(\psi_1(\eta)\,v_k-\psi_2(\eta)\,g_k\right)
     +\sqrt{2\gamma}\,a_k\,\xi'_{k+1},
\end{aligned}
\end{equation}
where $\psi_{i}(\eta)$, $i=0,1,2$,
are given in \eqref{eq:psidef}
and $(\xi_{k+1},\xi'_{k+1})$ is a centered $2d$-dimensional Gaussian random vector, independent across $k$
and across coordinates, with the per-coordinate covariance given in \eqref{eq:Cnoise}.
We couple the discrete-time dynamics synchronously to the stationary continuous diffusion, and
propagate the continuous contraction one step at a time, paying at each step
a local discretization error of order $\eta^2$.

\subsubsection{Notation and standing constants}
Throughout, $B$ denotes the Brownian motion driving
\eqref{eq:VL:modified:2}--\eqref{eq:XL:modified:2}, $d\ge1$ the dimension,
$\pi(x,v)\propto e^{-U(x)-\frac12\|v\|^2}$ the joint Gibbs law and
$\pi_X\propto e^{-U}$ its position marginal. As in
Corollary~\ref{cor:Wasserstein}, primes decorate the \emph{stationary} copy
of the dynamics. We use the weight $S$ and the twisted form $\mathcal H$,
the energy $\mathcal G$ of \eqref{eq:Gdef}, and the constants
$\kappa$ (Proposition~\ref{prop:linear}), $\beta_0=2\|S\|/\lambda_{\min}(S)$,
$C_1=\alpha_0(L+1)$, $C_2=2c_F(1+\sqrt L)$ with $\alpha_0=e^{\omega_0}-1
=a_{\max}-1$ and $c_F=e^{\omega_0}\omega_1=a_{\max}\omega_1$ (from
\eqref{eq:Ebound}), $C_\Gamma=8\gamma(1+\eta_{0}^2L)$, and $c>0$, $b=\gamma d$
from Lemma~\ref{lem:drift}.
Along the chain \eqref{eq:scheme} we abbreviate
\begin{equation}\label{eq:Gk}
  \mathcal G_k:=\mathcal G(x_k,v_k),\qquad
  \bar{\mathcal G}_k:=\bar{\mathcal G}(x_k,v_k)
  =\left(U_0(x_k)-U_0^\ast\right)+\frac12\|v_k\|^2 ,
\end{equation}
and we recall from the Notations that $\eta$ is the stepsize of
\eqref{eq:scheme} while $\eta_{0}$ is the (fixed) parameter of the
hypocoercive energy \eqref{eq:Gdef}. All the elementary properties of the
coefficient functions \eqref{eq:psidef} that we use below are collected here:
for $\eta>0$,
\begin{equation}\label{eq:psibounds}
\begin{aligned}
  &0<\psi_0(\eta)=e^{-\gamma\eta}\le1,\qquad
  \psi_1(\eta)\le\eta,\qquad
  \psi_2(\eta)\le\frac12\eta^2,
  \\
  &\psi_1(\eta)\left(1-\psi_0(\eta)\right)
  =\frac{\left(1-e^{-\gamma\eta}\right)^2}{\gamma}\le\gamma\eta^2,
\end{aligned}
\end{equation}
and, with $\sigma_i^2:=\int_0^\eta\psi_i^2(s)\,ds$ $(i=0,1)$ the diagonal
entries of the covariance \eqref{eq:Cnoise},
\begin{equation}\label{eq:sigmabounds}
  \sigma_0^2=\frac{1-e^{-2\gamma\eta}}{2\gamma}\le\eta,
  \qquad
  \sigma_1^2\le\frac{\eta^3}{3},
\end{equation}
and, whenever $\gamma\eta\le\frac12$,
\begin{equation}\label{eq:sigmabounds:2}
  \psi_1(\eta)\ge\frac12\eta,
  \qquad
  1-e^{-2\gamma\eta}\ge\gamma\eta .
\end{equation}
We fix
\begin{equation}\label{eq:alpharange}
  \alpha\in\left(0,\frac{c}{2C_\Gamma}\right),
  \qquad c':=c-\alpha C_\Gamma>0,\qquad c'':=c-2\alpha C_\Gamma>0,
\end{equation}
(the factor $2$ makes $e^{2\alpha\mathcal G}$ integrable at stationarity,
see Lemma~\ref{lem:statmom}(iii)), and we define, for $p\in\{1,2\}$,
\begin{equation}\label{eq:deltap}
  \delta_{(p)}:=p\,\beta_0C_1+\alpha b+\frac{p^2(\beta_0C_2)^2}{4\alpha c'},
  \qquad\text{so that }\delta_{(1)}=\delta(\alpha)\text{ of
  Theorem~\ref{thm:contraction}}.
\end{equation}
Note that $\alpha$ is \emph{not} multiplied by $p$ in \eqref{eq:deltap}:
the exponential weight $\Psi'=e^{\alpha\mathcal G'}$ is the same for both
values of $p$, only the power of $\mathcal H$ changes; see
Lemma~\ref{lem:cont-step}.
Our standing smallness condition, the discrete analogue of
$\delta<2\kappa$, is
\begin{equation}\label{eq:standing}
  \delta_{(2)}<4\kappa,
  \qquad
  \theta:=\kappa-\frac14\delta_{(2)}>0 .
\end{equation}

The following lemma shows that the standing condition
\eqref{eq:standing} is attainable for small anchoring mismatch, and
records two of its consequences that will keep the constants of the
discrete-time analysis free of the dimension: a bound on the product
$c_F\sqrt d$, used in the local error estimate of Lemma~\ref{lem:local},
and a bound on the exponential weight at stationarity, used in
Lemma~\ref{lem:local} and in Remark~\ref{rem:D0}.

\begin{lemma}[Consequences of the standing condition]\label{lem:standing}
Let Assumption~\ref{ass:main}\textup{(A)--(B)} hold, let $\alpha$ be as in
\eqref{eq:alpharange}, and let $\delta_{(p)}$, $\theta$ be as in
\eqref{eq:deltap}--\eqref{eq:standing}. Then:
\begin{itemize}
\item[(i)] \textup{(Attainability.)} One has $\delta_{(1)}(\alpha)<\delta_{(2)}(\alpha)$
for every admissible $\alpha$, and, at the optimizer
$\alpha_\circ=\beta_0C_2/(2\sqrt{bc'})$ of
Proposition~\ref{prop:smallmismatch},
\[
  \delta_{(2)}(\alpha_\circ)=2\delta_{(1)}(\alpha_\circ)+\alpha_\circ b
  \le3\,\delta_{(1)}(\alpha_\circ).
\]
Consequently $\delta_{(2)}(\alpha_\circ)\to0$ as $(\omega_0,\omega_1)\to0$, so
that the standing condition \eqref{eq:standing} holds, and $\theta>0$, for all
sufficiently small anchoring mismatch.
\item[(ii)] \textup{(Dimension-free mismatch.)} Under \eqref{eq:standing},
\begin{align}\label{eq:cFd}
  c_F\sqrt d\ \le\ \frac{\kappa}{\beta_0\left(1+\sqrt L\right)}\sqrt{\frac{c'}{\gamma}}.
\end{align}
In particular, $c_F=\mathcal O(d^{-1/2})$.
\item[(iii)] \textup{(Dimension-free weight.)} Under \eqref{eq:standing},
\begin{align}\label{eq:cw}
  c_w^2:=\E_\pi\left[e^{2\alpha\mathcal G}\right]\le e^{2\alpha b/c''}\le e^{8\kappa/c''}.
\end{align}
\end{itemize}
\end{lemma}
\begin{proof}
\textup{(i)} By \eqref{eq:deltap},
$\delta_{(2)}-\delta_{(1)}=\beta_0C_1+3(\beta_0C_2)^2/(4\alpha c')\ge0$,
with equality only if $C_1=C_2=0$; and
\begin{align*}
  \delta_{(2)}(\alpha)-2\delta_{(1)}(\alpha)
  &=\left(2\beta_0C_1-2\beta_0C_1\right)+\left(\alpha b-2\alpha b\right)
  +\frac{(\beta_0C_2)^2}{\alpha c'}\left(1-\frac12\right)
  \\
  &=-\alpha b+\frac{(\beta_0C_2)^2}{2\alpha c'}.
\end{align*}
At $\alpha=\alpha_\circ=\beta_0C_2/(2\sqrt{bc'})$ we have
$\alpha_\circ b=\frac12\beta_0C_2\sqrt{b/c'}$ and
$(\beta_0C_2)^2/(2\alpha_\circ c')=\beta_0C_2\sqrt{b/c'}$, and hence
$\delta_{(2)}(\alpha_\circ)-2\delta_{(1)}(\alpha_\circ)=\frac12\beta_0C_2\sqrt{b/c'}
=\alpha_\circ b$. Since $\alpha_\circ b$ is one of the three nonnegative
summands of $\delta_{(1)}(\alpha_\circ)$ in \eqref{eq:deltap}, we get
$\delta_{(2)}(\alpha_\circ)\le3\delta_{(1)}(\alpha_\circ)$, and
Proposition~\ref{prop:smallmismatch} gives
$\delta_{(1)}(\alpha_\circ)\to0$ as $(\omega_0,\omega_1)\to0$.

\textup{(ii)} By the AM--GM inequality applied to the last two summands of
$\delta_{(2)}$ in \eqref{eq:deltap} with $p=2$,
\[
  2\beta_0C_2\sqrt{b/c'}\le\alpha b+\frac{(\beta_0C_2)^2}{\alpha c'}\le\delta_{(2)}<4\kappa ,
\]
so $\beta_0C_2\sqrt{b/c'}<2\kappa$. Substituting $b=\gamma d$ and
$C_2=2c_F(1+\sqrt L)$ and rearranging gives the claim.

\textup{(iii)} The first inequality is Lemma~\ref{lem:statmom}(iii), whose
proof is independent of the present lemma. For the second, $\alpha b$ is one
of the summands of $\delta_{(2)}$ in \eqref{eq:deltap}, so
$\alpha b\le\delta_{(2)}<4\kappa$ by \eqref{eq:standing}.
\end{proof}

\subsubsection{Non-asymptotic error bounds}
Next, we obtain the non-asymptotic 2-Wasserstein error bound
between the distribution of the iterate of UALMC algorithm
and the target distribution.
\begin{theorem}[Discrete contraction and bias]\label{thm:disc}
Let Assumption~\ref{ass:main}\textup{(A)--(B)} hold, fix $\alpha$ as in
\eqref{eq:alpharange} and assume the smallness condition
\eqref{eq:standing}. Start the scheme \eqref{eq:scheme} from
$x_0\sim\nu_0$ and $v_0\sim\mathcal N(0,\Id)$ independent, where $\nu_0$
satisfies, for some $D_0\ge1$,
\begin{equation}\label{eq:init}
  \left(\E_{\nu_0}\left[(U_0-U_0^\ast)^8\right]\right)^{1/8}\ \le\ D_0\,d ,
\end{equation}
and let $\nu_k$ denote the law of $x_k$. Then there are $\bar{\eta}\in(0,1]$ and
$C_{\mathrm{loc}}>0$, depending only on
$(\gamma,m,L,\omega_0,\omega_1,\eta_{0},\alpha,c,\kappa,D_0)$ and \emph{not} on
$d$, $k$ or $\eta$, such that for every $\eta\le \bar{\eta}$ and every $k\ge0$,
\begin{equation}\label{eq:disc-bound}
  \mathcal W_2(\nu_k,\pi_X) \le
  c_\flat\left(1-\frac12\theta \eta\right)^{k}\,\mathcal D_0
  +\frac{2\,C_{\mathrm{loc}}}{\theta}\,\sqrt d\,\eta ,
\end{equation}
where $c_\flat:=\sqrt{\lambda_{\max}(S)/\lambda_{\min}(S)}$ and
\begin{equation}\label{eq:D0}
  \mathcal D_0:=\inf_{\zeta\in\Gamma_{\!\perp}(\nu_0,\pi)}
  \left(\E_\zeta\left[\left\|x-x'\right\|^4\,e^{\alpha\mathcal G(x',v')}\right]\right)^{1/4},
\end{equation}
where $\Gamma_{\!\perp}(\nu_0,\pi)$ denotes the set of couplings $\zeta$
of $(x,x',v')$ such that $x\sim\nu_0$, $(x',v')\sim\pi$, and $v'$ is
independent of $(x,x')$ (such couplings exist since under $\pi$ the
marginals $x'$ and $v'$ are already independent, $v'\sim\mathcal N(0,\Id)$;
the independence requirement is what makes the coupling compatible with the
initialization $v_0\perp x_0$ of the scheme).
\end{theorem}
\begin{remark}[On the initialization condition \eqref{eq:init}]\label{rem:init}
The moment condition \eqref{eq:init} cannot be dispensed with in the statement
of Theorem~\ref{thm:disc}: the constant $C_{\mathrm{loc}}$ depends on $D_0$,
because the local error of Lemma~\ref{lem:local} is controlled through the
\emph{uniform} moments of the iterates established in
Lemma~\ref{lem:itermom}, and those moments are propagated from the
initialization. Condition \eqref{eq:init} is mild: it asks that the
initial potential energy have $d$-proportional moments up to order eight,
which holds for a deterministic start $\nu_0=\delta_{x_0}$ with
$\|x_0-x_\ast\|^2\lesssim d/L$, and for $\nu_0=\mathcal N(x_\ast,L^{-1}\Id)$.
The order eight is dictated by the two Cauchy--Schwarz splittings in
Lemma~\ref{lem:local} \textup{(}$L_4$ against the weight, then a product of
two unbounded factors\textup{)}.
\end{remark}
We infer from Theorem~\ref{thm:disc} that
the iterates converge
to a ball of radius $\mathcal{O}\left(\frac{C_{\mathrm{loc}}}{\theta}\sqrt d\,\eta\right)$
around $\pi_X$ at the geometric rate $1-\frac12\theta\eta$; the bias is
$\mathcal{O}(\sqrt d\,\eta)$.
We devote the rest of this section
to the proof of
Theorem~\ref{thm:disc}.

\subsubsection{Proof of Theorem~\ref{thm:disc}}
\textbf{Step 0: synchronous coupling, interpolation, and the weighted
distance.}
Let $(X'_t,V'_t)_{t\ge0}$ be a stationary solution of the UALD,
$(X'_0,V'_0)\sim\pi$, driven by a Brownian motion $B$; by
Theorem~\ref{thm:invariant} its position marginal is $\pi_X$ at all times,
and we write, as in the continuous section,
$\mathcal G_t':=\mathcal G(X'_t,V'_t)$ and $\Psi_t':=e^{\alpha\mathcal G_t'}$.
On the same probability space, with the \emph{same} $B$, run the continuous
interpolation $(\tilde X_t,\tilde V_t)$ of the scheme \eqref{eq:scheme}: on
each $[j\eta,(j+1)\eta]$, with $\tilde a_j:=a(\tilde X_{j\eta})$ and
$\tilde g_j:=\tilde a_j\nabla U_0(\tilde X_{j\eta})$,
\begin{equation}\label{eq:interp}
\begin{aligned}
  \tilde V_t&=\psi_0(t-j\eta)\tilde V_{j\eta}-\psi_1(t-j\eta)\,\tilde g_j
            +\sqrt{2\gamma}\!\int_{j\eta}^{t}\!\psi_0(t-s)\,dB_s,\\
  \tilde X_t&=\tilde X_{j\eta}+\tilde a_j\left(\psi_1(t-j\eta)\tilde V_{j\eta}
            -\psi_2(t-j\eta)\,\tilde g_j
            +\sqrt{2\gamma}\!\int_{j\eta}^{t}\!\psi_1(t-s)\,dB_s\right),
\end{aligned}
\end{equation}
started from $\tilde V_0=V'_0$ and from $\left(\tilde X_0,\left(X'_0,V'_0\right)\right)$
realizing (up to an arbitrarily small error) the infimum in \eqref{eq:D0}.
Evaluating \eqref{eq:interp} at $t=(j+1)\eta$ recovers \eqref{eq:scheme}, so an
induction gives $\left(\tilde X_{j\eta},\tilde V_{j\eta}\right)\overset{d}{=}(x_j,v_j)$ for
every $j$; note that $\tilde V_0=V'_0\sim\mathcal N(0,\Id)$ is independent
of $\tilde X_0$ because the infimum in \eqref{eq:D0} runs over the
couplings $\Gamma_{\perp}(\nu_0,\pi)$, consistent with the
initialization of the scheme. Consequently, since
$\tilde X_{k\eta}\sim\nu_k$ and $X'_{k\eta}\sim\pi_X$,
\begin{equation}\label{eq:W2-to-L2}
  \mathcal W_2(\nu_k,\pi_X)\ \le\ \left\|\tilde X_{k\eta}-X'_{k\eta}\right\|_{L_2}.
\end{equation}
Stack the difference
$\tilde W_t:=\left(\tilde V_t-V'_t,\ \tilde X_t-X'_t\right)\in\R^{2d}$, write
$|w|_S:=\sqrt{w^\top Sw}$, and define the \emph{weighted} $L_4$ distance
\begin{equation}\label{eq:Aj}
  A_j:=\left(\E\left[\,\left|\tilde W_{j\eta}\right|_S^4\,\Psi_{j\eta}'\,\right]\right)^{1/4}.
\end{equation}
The weight $\Psi'$ is forced on us by Theorem~\ref{thm:contraction}, which
contracts $\E[\mathcal H\Psi']$ and not $\E[\mathcal H]$: the weight does
not cancel between steps and must be carried along; the exponent $4$ is
needed because the local error below contains \emph{products} of unbounded
quantities. Since $\Psi'\ge1$, Jensen's inequality and
$\lambda_{\min}(S)\|w\|^2\le|w|_S^2\le\lambda_{\max}(S)\|w\|^2$ give the two
endpoint relations
\begin{equation}\label{eq:ends}
  \left\|\tilde X_{k\eta}-X'_{k\eta}\right\|_{L_2}
  \le\left(\lambda_{\min}(S)\right)^{-1/2}A_k,
  \qquad
  A_0\le\sqrt{\lambda_{\max}(S)}\;\mathcal D_0,
\end{equation}
where the first inequality in \eqref{eq:ends} is due to
\begin{align*}
\E\left\|\tilde X_{k\eta}-X'_{k\eta}\right\|^2
&\le\left(\lambda_{\min}(S)\right)^{-1}
\E\left[\left|\tilde W_{k\eta}\right|_S^2\right]
\\
&\le\left(\lambda_{\min}(S)\right)^{-1}
\left(\E\left[\left|\tilde W_{k\eta}\right|_S^4\right]\right)^{1/2}\le\left(\lambda_{\min}(S)\right)^{-1}A_k^2,
\end{align*}
and
the second inequality in \eqref{eq:ends} is because $\tilde V_0=V'_0$ makes
\[
\left|\tilde W_0\right|_S^2\le\lambda_{\max}(S)\left\|\tilde X_0-X'_0\right\|^2.
\]
The whole proof
is a bound on the sequence $(A_j)$.

\textbf{Step 1: one continuous step (weighted contraction).}
On $[j\eta,(j+1)\eta]$ introduce the auxiliary \emph{exact} UALD trajectory
$(\hat X_t,\hat V_t)$ driven by the same $B$ and pinned to the interpolant
at the left endpoint, $\left(\hat X_{j\eta},\hat V_{j\eta}\right)=\left(\tilde X_{j\eta},\tilde
V_{j\eta}\right)$. Write $\Xi_t:=\left(\tilde V_t-\hat V_t,\ \tilde X_t-\hat X_t\right)$ for the
local (discretization) difference and
\[
\hat W_t:=\left(\hat V_t-V'_t,\ \hat X_t-X'_t\right),
\]
for the coupled difference of
the two exact trajectories, so that
\[
\tilde W_t=\Xi_t+\hat W_t
\qquad\text{and}\qquad
\hat W_{j\eta}=\tilde W_{j\eta}.
\]
Since
$Y\mapsto\left(\E\left[|Y|^4\Psi_{(j+1)\eta}'\right]\right)^{1/4}$ is a norm (it is an $L_4$ norm
with respect to the finite measure $\Psi_{(j+1)\eta}'\,d\mathbb P$), Minkowski's
inequality implies
\begin{equation}\label{eq:split}
  A_{j+1}
  \le \underbrace{\left(\E\left[\left|\Xi_{(j+1)\eta}\right|_S^4\,\Psi_{(j+1)\eta}'\right]
  \right)^{1/4}}_{=:L_j\ \text{(local error)}}
  \ +\
  \underbrace{\left(\E\left[\left|\hat W_{(j+1)\eta}\right|_S^4\,\Psi_{(j+1)\eta}'\right]
  \right)^{1/4}}_{=:\widehat C_j\ \text{(continuous step)}} .
\end{equation}

The continuous-step term $\widehat C_j$ in \eqref{eq:split}
is controlled by the weighted one-step contraction in the following lemma, which
strengthens Theorem~\ref{thm:contraction} in two directions: it is
conditional on $\mathcal F_s$, so that it can be applied step by step
along the chain, and it is formulated for the $p$-th power of
$\widehat{\mathcal H}$ with $p\in\{1,2\}$, the case $p=2$ being the one
that matches the weighted $L_4$ distance \eqref{eq:Aj}.

\begin{lemma}[Weighted one-step contraction]\label{lem:cont-step}
Let $(\hat X,\hat V)$ and $(X',V')$ be two solutions of the UALD driven by
the same Brownian motion, with $(X'_0,V'_0)\sim\pi$, and let
$\widehat{\mathcal H}_t:=|\hat W_t|_S^2$. Then for $p\in\{1,2\}$ and all
$0\le s\le t$,
\begin{equation}\label{eq:cont-step}
  \E\left[\widehat{\mathcal H}_t^{\,p}\,\Psi_t'\,\big|\,\mathcal F_s\right]
  \ \le\ e^{-(2p\kappa-\delta_{(p)})(t-s)}\,
  \widehat{\mathcal H}_s^{\,p}\,\Psi_s'
  \qquad\text{a.s.}
\end{equation}
Consequently, in \eqref{eq:split},
$\widehat C_j\le e^{-\theta\eta}A_j\le(1-\frac12\theta\eta)A_j$ whenever
$\theta \eta\le1$.
\end{lemma}

\begin{proof}
The pair is a synchronously coupled difference of exactly the kind analyzed
in the continuous-time analysis of
Section~\ref{sec:continuous}, namely in the synchronous coupling
\eqref{eq:UALD:coupled:1}--\eqref{eq:UALD:coupled:2}, in the derivation
of \eqref{eq:diffdyn}--\eqref{eq:Hineq} and in the proof of
Theorem~\ref{thm:contraction}, with the stationary copy playing the role of the
primed trajectory in \eqref{eq:Et}--\eqref{eq:Ebound}; this matters because
the unbounded entries of $E_t$ ($V'_t$ and $\nabla U_0(X'_t)$) are then
evaluated along the \emph{stationary} path, while $a_t=a(\hat X_t)\in
[a_{\min},a_{\max}]$ and $\|F_t\|\le c_F$ are bounded regardless. The
derivation of \eqref{eq:Hineq} is pathwise, so
\begin{align*}
\dot{\widehat{\mathcal H}}_t\le(-2\kappa+K_t)\widehat{\mathcal H}_t,
\end{align*}
with $K_t$ as in \eqref{eq:Kt},
\begin{align*}
K_t\le\beta_0\left(C_1+C_2\sqrt{\mathcal G_t'}\right),
\end{align*}
and since
$\widehat{\mathcal H}^{\,p}$ is absolutely continuous,
\begin{align*}
\frac{d}{dt}\widehat{\mathcal H}^{\,p}_t\le
p(-2\kappa+K_t)\widehat{\mathcal H}^{\,p}_t.
\end{align*}
By Dynkin's formula,
$d\Psi_t'=(\mathcal A\Psi)_t'\,dt+dM_t$ with $M$ a local martingale, and as
$\widehat{\mathcal H}^{\,p}$ has finite variation the product rule carries
no cross-variation term. Hence, inserting \eqref{eq:expdrift} in the form
\[
\frac{(\mathcal A\Psi)_t'}{\Psi_t'}\le-\alpha c'\mathcal G_t'+\alpha b,
\]
and writing
$s:=\sqrt{\mathcal G_t'}$, the process
\[
Z_t:=e^{(2p\kappa-\delta_{(p)})t}\,\widehat{\mathcal H}^{\,p}_t\Psi_t'
\]
satisfies
\begin{align*}
  dZ_t&\le Z_t\left(\left(2p\kappa-\delta_{(p)}\right)-2p\kappa
  +p\beta_0C_1+p\beta_0C_2\,s-\alpha c's^2+\alpha b\right)dt
  \\
  &\qquad\qquad\qquad
  +e^{(2p\kappa-\delta_{(p)})t}\widehat{\mathcal H}^{\,p}_t\Psi_t'\,dM_t/\Psi_t'.
\end{align*}
By completing the square, we get
\[
p\beta_0C_2\,s-\alpha c's^2\le \frac{p^2(\beta_0C_2)^2}{4\alpha c'},
\]
which implies that
the drift
factor satisfies
\begin{align*}
&\left(2p\kappa-\delta_{(p)}\right)-2p\kappa
  +p\beta_0C_1+p\beta_0C_2\,s-\alpha c's^2+\alpha b
\\
&\le-\delta_{(p)}+p\beta_0C_1+\alpha b
+\frac{p^2(\beta_0C_2)^2}{4\alpha c'}=0,
\end{align*}
by \eqref{eq:deltap}. Thus $Z$ is a
nonnegative local supermartingale; a localizing sequence and Fatou's lemma
(conditionally on $\mathcal F_s$, exactly as in the proof of
Lemma~\ref{lem:expmoment}) give
$\E[Z_t\mid\mathcal F_s]\le Z_s$, which is the claim. For $p=1$, $s=0$ this
recovers Theorem~\ref{thm:contraction}. Finally, applying the claim \eqref{eq:cont-step} with
$p=2$, $s=j\eta$, $t=(j+1)\eta$ and $\widehat{\mathcal H}_{j\eta}=
\left|\tilde W_{j\eta}\right|_S^2$,
\[
  \widehat C_j^{\,4}
  =\E\left[\E\left[\widehat{\mathcal H}^{\,2}_{(j+1)\eta}\Psi_{(j+1)\eta}'
  \,\big|\,\mathcal F_{j\eta}\right]\right]
  \le e^{-(4\kappa-\delta_{(2)})\eta}\,A_j^4=e^{-4\theta\eta}A_j^4 ,
\]
and $e^{-x}\le1-\frac12x$ for $x\in[0,1]$ gives the second bound.
\end{proof}
\textbf{Step 2: moment bounds.}
The local error will be expressed through the path quantities
\[
  M_j:=\sup_{0\le\tau\le \eta}\left\|\hat V_{j\eta+\tau}\right\|,
  \qquad
  N_j:=\left\|\nabla U_0(\tilde X_{j\eta})\right\| .
\]
Unlike the classical KLMC in \cite{dalalyan2018kinetic}, the local error below contains the \emph{products}
$M_jN_j$ and $M_j^2$, so an $L_2$ bound on $(M_j,N_j)$ does not suffice.
Instead of the deviation splitting
in \cite{dalalyan2018kinetic} (which controls
$M_j$ by $\sqrt d+A_j$ in $L_2$), we prove \emph{uniform polynomial moment
bounds} directly: for the stationary path from the drift condition
\eqref{eq:drift}, and for the algorithm iterates from a discrete
Foster--Lyapunov inequality. This closes the argument at any polynomial
order and, as a by-product, removes the $A_j$-dependence of the local
error entirely.

We begin with the stationary copy $(X',V')$, whose moments
follow from the Gaussian velocity marginal of $\pi$ and from the drift
condition \eqref{eq:drift}; part (iii) below is the stationary
exponential-moment bound anticipated in \eqref{eq:cw}.

\begin{lemma}[Moments along the stationary path]\label{lem:statmom}
For every $p\ge1$ there is $c_\sharp(p)$, depending only on
$(\gamma,m,L,\omega_0,\eta_{0},p)$, such that for all $t\ge0$ and $\eta\le1$:
\begin{itemize}
\item[(i)] $\left\|\sup_{0\le\tau\le \eta}\|V'_{t+\tau}\|\right\|_{L_p}
  \le c_\sharp(p)\sqrt d$;
\item[(ii)] $\left\|\,\|\nabla U_0(X'_t)\|\,\right\|_{L_p}\le
  c_\sharp(p)\sqrt d$;
\item[(iii)] $\E\left[e^{2\alpha\mathcal G_t'}\right]\le e^{2\alpha b/c''}$,
  so that $c_w$ of \eqref{eq:cw} is finite and dimension-free under
  \eqref{eq:standing}.
\end{itemize}
\end{lemma}

\begin{proof}
Under $\pi$ the velocity marginal is exactly $\mathcal N(0,\Id)$, and hence all
its $L_p$ norms are bounded by $c(p)\sqrt d$, with the explicit constant
\begin{equation}\label{eq:chimoment}
  \left\|\,\|V'_t\|\,\right\|_{L_p}
  =\left(\E\left[\chi_d^{\,p}\right]\right)^{1/p}
  =\sqrt2\left(\frac{\Gamma\!\left(\frac{d+p}{2}\right)}{\Gamma\!\left(\frac d2\right)}\right)^{1/p}
  \le\sqrt{d+p}\ \le\ \sqrt{1+p}\,\sqrt d ,
\end{equation}
where $\chi_d$ denotes a chi-distributed random variable with $d$ degrees of
freedom; the first inequality follows from
\begin{align*}
\E\left[\chi_d^{\,p}\right]\le\left(\E\left[\chi_d^{\,2\lceil p/2\rceil}\right]\right)^{p/(2\lceil p/2\rceil)},
\end{align*}
and
\begin{align*}
\E\left[\chi_d^{\,2r}\right]=d(d+2)\cdots(d+2r-2)\le(d+p)^{r},
\end{align*}
and the second inequality follows from
$d\ge1$. Thus one may take $c(p)=\sqrt{1+p}$.
For the potential part, stationarity
gives $\E[\mathcal A\mathcal G^p]=0$ (a standard localization argument,
using the drift \eqref{eq:drift} to justify that $\mathcal G^p$ is in the
domain of the extended generator and integrable under $\pi$); since
\begin{align*}
\mathcal A\mathcal G^{p}=p\mathcal G^{p-1}\mathcal A\mathcal G
+\gamma p(p-1)\mathcal G^{p-2}\|\nabla_v\mathcal G\|^2,
\end{align*}
and
$\gamma\|\nabla_v\mathcal G\|^2\le C_\Gamma\mathcal G$ by \eqref{eq:carre},
the drift \eqref{eq:drift} yields
\begin{align*}
c\,\E[\mathcal G^p]\le\left(b+(p-1)C_\Gamma\right)\E[\mathcal G^{p-1}],
\end{align*}
and hence by induction
\begin{align*}
\E[\mathcal G^p]\le\left((b+pC_\Gamma)/c\right)^p\le(c(p)\,d)^p,
\end{align*}
using
$b=\gamma d$ and $d\ge1$, with $c(p)=(\gamma+pC_\Gamma)/c$. Together with (by \eqref{eq:equiv})
\[
\|\nabla U_0\|^2\le2L(U_0-U_0^\ast)\le2L\bar{\mathcal G}\le4L\mathcal G,
\]
this proves (ii). For (iii), stationarity and
\eqref{eq:expdrift} at level $\beta\in(0,2\alpha]$ give
\begin{align*}
\E\left[\mathcal G e^{\beta\mathcal G}\right]\le\frac{b}{c-\beta C_\Gamma}
\E\left[e^{\beta\mathcal G}\right]\le\frac{b}{c''}\E\left[e^{\beta\mathcal G}\right],
\end{align*}
which implies that
\begin{align*}
\frac{d}{d\beta}\log\E\left[e^{\beta\mathcal G}\right]\le\frac{b}{c''}.
\end{align*}
Integrating over
$[0,2\alpha]$ gives (iii). For (i), from the velocity equation and
$0\le\psi_0\le1$,
\[
  \sup_{\tau\le \eta}\left\|V'_{t+\tau}\right\|
  \le\left\|V'_t\right\|+a_{\max}\!\int_0^\eta\!\left\|\nabla U_0(X'_{t+s})\right\|\,ds
  +\sqrt{2\gamma}\,\sup_{\tau\le \eta}
  \left\|\int_0^\tau e^{-\gamma(\tau-s)}dB_{t+s}\right\|,
\]
where
\[
\left\Vert a_{\max}\!\int_0^\eta\!\left\|\nabla U_0(X'_{t+s})\right\|\,ds\right\Vert_{L_{p}}
\leq a_{\max}\eta\,c_\sharp(p)\sqrt d,
\]
by Minkowski's
integral inequality and (ii), and the stochastic convolution satisfies
\[
\int_0^\tau e^{-\gamma(\tau-s)}dB_s=B_\tau-\gamma\int_0^\tau
e^{-\gamma(\tau-s)}B_s\,ds,
\]
so that
\[
\sup_{\tau\leq\eta}
\left\Vert\int_0^\tau e^{-\gamma(\tau-s)}dB_s\right\Vert
\leq
(1+\gamma\eta)\sup_{\tau\le \eta}\|B_\tau\|,
\]
and by Doob's inequality, we conclude that
\[
\left\Vert
\sup_{\tau\leq\eta}
\left\Vert\int_0^\tau e^{-\gamma(\tau-s)}dB_s\right\Vert\right\Vert_{L_{p}}
\leq
(1+\gamma\eta)
\left\Vert
\sup_{\tau\le \eta}\|B_\tau\|\right\Vert_{L_{p}}
\leq
c(p)\sqrt{\eta d},
\]
where, for $p>1$, one may take $c(p)=\frac{p}{p-1}\sqrt{1+p}\,(1+\gamma)$
by Doob's $L_p$ maximal inequality together with \eqref{eq:chimoment} applied
to $B_\eta/\sqrt\eta\sim\mathcal N(0,\Id)$.
Combining this proves (i).
\end{proof}

We next turn to the iterates of the scheme \eqref{eq:scheme},
which are not stationary. The following lemma shows that their polynomial
moments are nevertheless bounded uniformly in $k$, with the dimension
dependence inherited from the initialization \eqref{eq:init}; the proof
is a discrete Foster--Lyapunov argument for the hypocoercive energy
$\mathcal G$ of \eqref{eq:Gdef}, i.e.\ the discrete-time counterpart of
Lemma~\ref{lem:drift}.

\begin{lemma}[Uniform moments of the iterates]\label{lem:itermom}
Let \eqref{eq:init} hold and recall $\bar{\mathcal G}_k$ from \eqref{eq:Gk}.
There are $\bar{\eta}\in(0,1]$ and, for each $p\in[1,8]$, a
constant $c_I(p)$, both depending only on
$(\gamma,m,L,\omega_0,\eta_{0},p)$, such that for all $\eta\le \bar{\eta}$,
\[
  \sup_{k\ge0}\left(\E\left[\bar{\mathcal G}_k^{\,p}\right]\right)^{1/p}
  \ \le\ c_I(p)\,(1+D_0)\,d .
\]
\end{lemma}

\begin{proof}
We verify a one-step drift for the hypocoercive energy $\mathcal G$ of
\eqref{eq:Gdef}, then upgrade it to $p$-th powers, and finally transfer to
$\bar{\mathcal G}$ by \eqref{eq:equiv}. Throughout, $C$ denotes a constant
depending only on $(\gamma,m,L,\omega_0,\eta_{0},p)$ whose value may change
from line to line, and we assume $\gamma\eta\le\frac12$ so that
\eqref{eq:psibounds}--\eqref{eq:sigmabounds} apply.

\textbf{Step (i): one-step drift for $\mathcal G$.}
Condition on $\mathcal F_k$ and write $x=x_k$, $v=v_k$,
$a=a_k\in[a_{\min},a_{\max}]$, $g=g_k=a\nabla U_0(x)$, $\psi_i=\psi_i(\eta)$.
It is convenient to split the update \eqref{eq:scheme} into its
$\mathcal F_k$-measurable part and its noise part,
\begin{equation}\label{eq:vbarxbar}
\begin{aligned}
  &v_{k+1}=\bar v+\sqrt{2\gamma}\,\xi_{k+1},\qquad \bar v:=\psi_0v-\psi_1g,
  \\
  &x_{k+1}=\bar x+\sqrt{2\gamma}\,a\,\xi'_{k+1},\qquad
  \bar x:=x+a\left(\psi_1v-\psi_2g\right).
\end{aligned}
\end{equation}

\textbf{(a) Kinetic part.} Using $\E[\xi_{k+1}\mid\mathcal F_k]=0$ and
$\E[\|\xi_{k+1}\|^2\mid\mathcal F_k]=\sigma_0^2d$,
\begin{align}
  \E\left[\frac12\|v_{k+1}\|^2\,\big|\,\mathcal F_k\right]
  &=\frac12\psi_0^2\|v\|^2-\psi_0\psi_1\langle v,g\rangle
  +\frac12\psi_1^2\|g\|^2+\gamma\sigma_0^2d
  \nonumber
  \\
  &\le\frac12\|v\|^2-\frac{\gamma\eta}{2}\|v\|^2
  -\psi_0\psi_1\langle v,g\rangle+\frac{\eta^2}{2}\|g\|^2+\gamma\eta d,
  \label{eq:kinetic:part}
\end{align}
where we used $\frac12\psi_0^2=\frac12-\frac12(1-e^{-2\gamma\eta})
\le\frac12-\frac12\gamma\eta$ and \eqref{eq:sigmabounds}.

\textbf{(b) Potential part.} By $L$-smoothness of $U_0$,
\[
  \E\left[U_0(x_{k+1})\,\big|\,\mathcal F_k\right]
  \le U_0(x)+\E\left[\langle\nabla U_0(x),x_{k+1}-x\rangle\,\big|\,\mathcal F_k\right]
  +\frac L2\E\left[\|x_{k+1}-x\|^2\,\big|\,\mathcal F_k\right].
\]
Since $a\langle\nabla U_0(x),v\rangle=\langle g,v\rangle$ and
$\langle\nabla U_0(x),ag\rangle=a^2\|\nabla U_0(x)\|^2\ge0$,
\begin{equation}\label{eq:potential:cross}
  \E\left[\langle\nabla U_0(x),x_{k+1}-x\rangle\,\big|\,\mathcal F_k\right]
  =\psi_1\langle g,v\rangle-\psi_2a^2\|\nabla U_0(x)\|^2
  \le\psi_1\langle g,v\rangle ,
\end{equation}
while, by $\|u_1+u_2+u_3\|^2\le3(\|u_1\|^2+\|u_2\|^2+\|u_3\|^2)$ and
\eqref{eq:psibounds}--\eqref{eq:sigmabounds},
\begin{align}
  \E\left[\|x_{k+1}-x\|^2\,\big|\,\mathcal F_k\right]
  &\le3a_{\max}^2\left(\psi_1^2\|v\|^2+\psi_2^2\|g\|^2+2\gamma\sigma_1^2d\right)
  \nonumber
  \\
  &\le3a_{\max}^2\left(\eta^2\|v\|^2+\frac{\eta^4}{4}\|g\|^2
  +\frac{2\gamma\eta^3}{3}d\right).\label{eq:xincr}
\end{align}

\textbf{(c) The two cross terms.} The two Hamiltonian cross terms in \eqref{eq:kinetic:part} and \eqref{eq:potential:cross} are precisely
$-\psi_0\psi_1\langle v,g\rangle$, produced by the kinetic part (a), and
$+\psi_1\langle g,v\rangle$, produced by the potential part (b). In the
continuous-time computation of Lemma~\ref{lem:drift} these two terms cancel
exactly; after discretization they only \emph{almost} cancel, leaving
\[
  -\psi_0\psi_1\langle v,g\rangle+\psi_1\langle g,v\rangle
  =\psi_1\left(1-\psi_0\right)\langle g,v\rangle
  \le\gamma\eta^2\|v\|\,\|g\|
  \le\frac{\gamma\eta^2}{2}\left(\|v\|^2+\|g\|^2\right),
\]
by \eqref{eq:psibounds}; that is, the exact continuous cancellation survives
discretization up to an error of order $\eta^2$.

\textbf{(d) Hypocoercive cross term.} Write
\begin{equation}\label{eq:hypo:split}
  \langle v_{k+1},\nabla U_0(x_{k+1})\rangle
  =\langle v_{k+1},\nabla U_0(x)\rangle
  +\langle v_{k+1},\nabla U_0(x_{k+1})-\nabla U_0(x)\rangle .
\end{equation}
For the first summand, \eqref{eq:vbarxbar} gives
\begin{align*}
  \E\left[\langle v_{k+1},\nabla U_0(x)\rangle\,\big|\,\mathcal F_k\right]
  &=\psi_0\langle v,\nabla U_0(x)\rangle-\psi_1a\|\nabla U_0(x)\|^2
  \\
  &=\langle v,\nabla U_0(x)\rangle-(1-\psi_0)\langle v,\nabla U_0(x)\rangle
  -\psi_1a\|\nabla U_0(x)\|^2 .
\end{align*}
The last term yields the position dissipation: by \eqref{eq:sigmabounds},
\[
-\psi_1a\|\nabla U_0(x)\|^2\le-\frac12a_{\min}\eta\|\nabla U_0(x)\|^2.
\]
For the middle term we use $1-\psi_0\le\gamma\eta$ and Young's inequality
with the parameter $\zeta:=2\gamma/a_{\min}$, exactly as in the proof of
Lemma~\ref{lem:drift},
\[
  \left|(1-\psi_0)\langle v,\nabla U_0(x)\rangle\right|
  \le\gamma\eta\left(\frac{\zeta}{2}\|v\|^2+\frac1{2\zeta}\|\nabla U_0(x)\|^2\right)
  =\frac{\gamma^2\eta}{a_{\min}}\|v\|^2
  +\frac{a_{\min}\eta}{4}\|\nabla U_0(x)\|^2 ,
\]
so that half of the dissipation survives. For the second summand in \eqref{eq:hypo:split}, by
$L$-smoothness, the Cauchy--Schwarz inequality (applied to the conditional
expectation given $\mathcal F_k$), \eqref{eq:vbarxbar} and \eqref{eq:xincr},
\begin{align*}
  \left|\E\left[\langle v_{k+1},\nabla U_0(x_{k+1})-\nabla U_0(x)\rangle
  \,\big|\,\mathcal F_k\right]\right|
  &\le L\left(\E\left[\|v_{k+1}\|^2\big|\mathcal F_k\right]\right)^{1/2}
  \left(\E\left[\|x_{k+1}-x\|^2\big|\mathcal F_k\right]\right)^{1/2}
  \\
  &\le3La_{\max}\eta\left(\|v\|^2+\eta^2\|g\|^2+2\gamma\eta d\right),
\end{align*}
where we used
\[
\left(\E\left[\left\|v_{k+1}\right\|^2|\mathcal F_k\right]\right)^{1/2}\le\|v\|+\eta\|g\|+\sqrt{2\gamma\eta d},
\]
and
\[
\left(\E\left[\left\|x_{k+1}-x\right\|^2|\mathcal F_k\right]\right)^{1/2}\le a_{\max}\eta
\left(\|v\|+\eta\|g\|+\sqrt{2\gamma\eta d}\right).
\]

\textbf{Collecting.} Adding (a)--(d), multiplying (d) by $\eta_{0}$, and using
$\|g\|^2\le a_{\max}^2\|\nabla U_0(x)\|^2$, we obtain
\[
  \E\left[\mathcal G_{k+1}\,\big|\,\mathcal F_k\right]
  \le\mathcal G_k-\eta\,A_v\,\|v\|^2-\eta\,A_g\,\|\nabla U_0(x)\|^2+C_0\,\eta\,d,
\]
with
\begin{align*}
  &A_v:=\frac{\gamma}{2}-\frac{\gamma\eta}{2}-\frac{3La_{\max}^2\eta}{2}
  -\eta_{0}\left(\frac{\gamma^2}{a_{\min}}+3La_{\max}\right),
  \\
  &A_g:=\frac{\eta_{0}a_{\min}}{4}
  -\left(\frac{\gamma\eta a_{\max}^2}{2}+\frac{3La_{\max}^4\eta^3}{8}
  +3\eta_{0}La_{\max}^3\eta^2\right),
\end{align*}
and
\[
C_0:=\gamma+2\gamma La_{\max}^2+6\eta_{0}\gamma La_{\max}.
\]
Choosing first
$\eta_{0}\le\bar\eta_0$ so small that
$\eta_{0}(\gamma^2/a_{\min}+3La_{\max})\le\gamma/8$, and then
$\bar\eta\in(0,1]$ so small that
$\frac{\gamma\bar\eta}{2}+\frac{3La_{\max}^2\bar\eta}{2}\le\gamma/8$ and
that the bracket in $A_g$ is at most $\eta_{0}a_{\min}/8$, we get
$A_v\ge\gamma/4$ and $A_g\ge\eta_{0}a_{\min}/8$ for all $\eta\le\bar\eta$.
Finally, the Polyak--{\L}ojasiewicz bound
$\|\nabla U_0(x)\|^2\ge2m(U_0(x)-U_0^\ast)$ and \eqref{eq:equiv} give
\begin{equation}\label{eq:onestepdrift}
  \E\left[\mathcal G_{k+1}\,\big|\,\mathcal F_k\right]
  \le\left(1-c_d\,\eta\right)\mathcal G_k+C_0\,\eta\,d,
  \qquad
  c_d:=\frac13\min\left\{\frac{\gamma}{2},\frac{m\,\eta_{0}a_{\min}}{4}\right\}>0 .
\end{equation}

\textbf{Step (ii): size of the fluctuation.}
Let $\mu_k:=\E[\mathcal G_{k+1}\mid\mathcal F_k]$ and
$\Delta_k:=\mathcal G_{k+1}-\mu_k$, so $\E[\Delta_k\mid\mathcal F_k]=0$. We
claim that for every $q\ge2$ there is $C_q$ with
\begin{equation}\label{eq:fluct}
  \left\|\Delta_k\right\|_{L_q(\mathbb P(\cdot\mid\mathcal F_k))}
  \ \le\ C_q\sqrt\eta\left(\sqrt{\mathcal G_k}+\sqrt d\right),
  \qquad \eta\le\bar\eta .
\end{equation}
Indeed, writing $\Theta:=\sqrt{2\gamma}\,\xi_{k+1}$ and
$\Theta':=\sqrt{2\gamma}\,a\,\xi'_{k+1}$, which are conditionally Gaussian with
$\|\,\|\Theta\|\,\|_{L_q}\le c_q\sqrt{\gamma\eta d}$ and
$\|\,\|\Theta'\|\,\|_{L_q}\le c_qa_{\max}\sqrt{\gamma\eta^3d}$ by
\eqref{eq:sigmabounds} and \eqref{eq:chimoment}, we have
\[
\|\Delta_k\|_{L_q}\le2\left\|\mathcal G_{k+1}-\mathcal G(\bar x,\bar v)\right\|_{L_q},
\]
and, by \eqref{eq:Gdef}, $L$-smoothness and the Cauchy--Schwarz inequality,
\begin{align*}
  \left|\mathcal G_{k+1}-\mathcal G(\bar x,\bar v)\right|
  &\le\left(\|\nabla U_0(\bar x)\|+\eta_{0}L\|\bar v\|\right)\|\Theta'\|
  +\frac L2\|\Theta'\|^2
  \\
  &\qquad
  +\left(\|\bar v\|+\eta_{0}\|\nabla U_0(\bar x)\|\right)\|\Theta\|
  +\frac12\|\Theta\|^2+\eta_{0}L\|\Theta\|\,\|\Theta'\| .
\end{align*}
By using
$\|v\|\le2\sqrt{\mathcal G_k}$ and $\|\nabla U_0(x)\|\le2\sqrt{L\mathcal G_k}$ from
\eqref{eq:equiv}, we have
\[
\|\bar v\|\le\|v\|+\eta\|g\|\le C\sqrt{\mathcal G_k},
\]
and
\[
\|\nabla U_0(\bar x)\|\le\|\nabla U_0(x)\|+L\|\bar x-x\|
\le C\sqrt{\mathcal G_k},
\]
for $\eta\le1$, and by taking $L_q$ norms and using $\eta\le1$,
it yields
\eqref{eq:fluct}.

\textbf{Step (iii): $p$-th moments.}
Let $p\ge2$ be an integer. Expanding $\mathcal G_{k+1}^{\,p}=(\mu_k+\Delta_k)^p$
and using $\E[\Delta_k\mid\mathcal F_k]=0$,
\begin{equation}\label{eq:binomial:expansion}
  \E\left[\mathcal G_{k+1}^{\,p}\,\big|\,\mathcal F_k\right]
  =\mu_k^{\,p}+\sum_{i=2}^{p}\binom pi\mu_k^{\,p-i}
  \E\left[\Delta_k^{\,i}\,\big|\,\mathcal F_k\right].
\end{equation}
By \eqref{eq:fluct}, for $i\ge2$ and $\eta\le1$,
\[
|\E[\Delta_k^i\mid\mathcal F_k]|\le
(C_p\sqrt\eta)^i\left(\sqrt{\mathcal G_k}+\sqrt d\right)^i\le C\eta(\mathcal G_k+d)^{i/2},
\]
while $\mu_k\le\mathcal G_k+C_0\eta d$ by
\eqref{eq:onestepdrift}; since $d\ge1$ implies $\mathcal G_k+d\ge1$, every
term of the sum on the right-hand side of \eqref{eq:binomial:expansion} is bounded by $C\eta(\mathcal G_k+d)^{p-1}$. Moreover, by the
binomial expansion of
\[
\mu_k^{\,p}\le\left((1-c_d\eta)\mathcal G_k
+C_0\eta d\right)^p,
\]
we obtain
\[
  \mu_k^{\,p}\le(1-c_d\eta)^p\mathcal G_k^{\,p}+C\eta\,d\left(\mathcal G_k+d\right)^{p-1}
  \le\left(1-pc_d\eta\right)\mathcal G_k^{\,p}
  +C\eta\,d\left(\mathcal G_k+d\right)^{p-1}.
\]
Applying Young's inequality
$d(\mathcal G_k+d)^{p-1}\le\varepsilon(\mathcal G_k+d)^p+C_\varepsilon d^{\,p}$
together with $(\mathcal G_k+d)^p\le2^{p-1}(\mathcal G_k^{\,p}+d^{\,p})$, and
choosing $\varepsilon$ small enough, we arrive at
\[
  \E\left[\mathcal G_{k+1}^{\,p}\,\big|\,\mathcal F_k\right]
  \le\left(1-\frac12pc_d\eta\right)\mathcal G_k^{\,p}+C_p\,\eta\,d^{\,p},
  \qquad \eta\le\bar\eta(p).
\]
Taking expectations and iterating over $k$,
\[
  \sup_{k\ge0}\E\left[\mathcal G_k^{\,p}\right]
  \le\E\left[\mathcal G_0^{\,p}\right]+\frac{2C_p\,d^{\,p}}{p\,c_d},
\]
and non-integer $p\in[1,8]$ follows from the integer case by Jensen's
inequality. It remains to bound the initial term: by \eqref{eq:equiv},
\[
\mathcal G_0\le\frac32\bar{\mathcal G}_0
=\frac32\left(\left(U_0(x_0)-U_0^\ast\right)+\frac12\|v_0\|^2\right),
\]
so by
\eqref{eq:init}, the independence of $v_0\sim\mathcal N(0,\Id)$ from $x_0$ and
\eqref{eq:chimoment},
\[
  \left(\E\left[\mathcal G_0^{\,p}\right]\right)^{1/p}
  \le\frac32\left(D_0d+\frac12\left\|\,\|v_0\|\,\right\|^2_{L_{2p}}\right)
  \le\frac32\left(D_0+\frac12(1+2p)\right)d .
\]
Finally $\bar{\mathcal G}_k\le2\mathcal G_k$ by \eqref{eq:equiv}, which gives
the assertion with $c_I(p)$ depending only on $(\gamma,m,L,\omega_0,\eta_{0},p)$, as required in the statement of the lemma.
\end{proof}

Combining Lemmas~\ref{lem:statmom} and~\ref{lem:itermom}, we
now bound the moments of the two path quantities $M_j$ and $N_j$ through
which the local error of Step 3 will be expressed; moments of order up
to $16$ are needed because of the two splittings using Cauchy--Schwarz inequality in the
proof of Lemma~\ref{lem:local}, see also Remark~\ref{rem:init}.

\begin{lemma}[Path moments entering the local error]\label{lem:pathmom}
There are $\bar{\eta}\in(0,1]$ and $c_\natural$, depending only on
$(\gamma,m,L,\omega_0,\eta_{0},D_0)$, such that for $\eta\le \bar{\eta}$, every $j$, and
every $p\le16$,
\[
  \|M_j\|_{L_p}\ \vee\ \|N_j\|_{L_p}\ \le\ c_\natural\sqrt d .
\]
\end{lemma}

\begin{proof}
$N_j\le\sqrt{2L\,\bar{\mathcal G}_j}$ by $L$-smoothness, so
Lemma~\ref{lem:itermom} (with $p\le8$ for $\bar{\mathcal G}$, i.e.\ $p\le16$
for $N_j$) gives the second bound. For $M_j$: exactly as in
Lemma~\ref{lem:statmom}(i), along the auxiliary exact path,
\[
  M_j\le\|\tilde V_{j\eta}\|
  +a_{\max}\!\int_0^\eta\left\|\nabla U_0(\hat X_{j\eta+s})\right\|\,ds
  +\sqrt{2\gamma}(1+\gamma\eta)\sup_{\tau\le \eta}\|B_{j\eta+\tau}-B_{j\eta}\| ,
\]
and $\|\nabla U_0(\hat X_{j\eta+s})\|\le N_j+La_{\max}sM_j$ since
$\|\hat X_{j\eta+s}-\tilde X_{j\eta}\|\le a_{\max}sM_j$. Hence,
\begin{equation}\label{eq:Mjrearr}
\left(1-\frac12La_{\max}^2\eta^2\right)M_j\le\|\tilde V_{j\eta}\|+a_{\max}\eta N_j+\sqrt{2\gamma}
(1+\gamma \eta)\sup_{\tau\le \eta}\|B_{j\eta+\tau}-B_{j\eta}\| .
\end{equation}
Here $M_j<\infty$ almost surely, since $\hat V$ has continuous paths on
the compact interval $[j\eta,(j+1)\eta]$, and $\|M_j\|_{L_p}<\infty$ by the
standard $L_p$ estimates for SDEs with Lipschitz drift and bounded diffusion
coefficient; see, e.g., \cite[Problem 5.3.15]{KaratzasShreve1991}
or \cite[Chapter 2, Section 4]{Mao2007}; these estimates require only
that the drift be locally Lipschitz and of linear growth and that the
diffusion coefficient be bounded, which is the case for the UALD
\eqref{eq:UALD:1}--\eqref{eq:UALD:2}: its drift
$(v,x)\mapsto\left(-\gamma v-a(x)\nabla U_0(x),\,a(x)v\right)$ satisfies
$\|a(x)\nabla U_0(x)\|\le a_{\max}L\|x-x_\ast\|$ and
$\|a(x)v\|\le a_{\max}\|v\|$ by Assumption~\ref{ass:main}, and its
diffusion coefficient $\sqrt{2\gamma}$ is constant; recall from
\eqref{eq:glip} that the force $g=a\nabla U_0$ is locally, but not
globally, Lipschitz. Therefore \eqref{eq:Mjrearr} may be rearranged. The prefactor on
the left-hand side of \eqref{eq:Mjrearr} satisfies
$1-\frac12La_{\max}^2\eta^2\ge\frac12$ as soon as
$\eta\le\bar\eta:=\min\left\{1,(La_{\max}^2)^{-1/2}\right\}$, and then, taking $L_p$
norms in \eqref{eq:Mjrearr} and applying the triangle inequality,
\begin{align*}
  \|M_j\|_{L_p}
  &\le2\left(\left\|\,\|\tilde V_{j\eta}\|\,\right\|_{L_p}
  +a_{\max}\eta\,\|N_j\|_{L_p}
  +\sqrt{2\gamma}(1+\gamma\eta)\left\|\sup_{\tau\le\eta}
  \|B_{j\eta+\tau}-B_{j\eta}\|\right\|_{L_p}\right)
  \\
  &\le c_\natural\sqrt d ,
\end{align*}
where we used
\[
\|\,\|\tilde V_{j\eta}\|\,\|_{L_p}
\le
\sqrt2\left\|\bar{\mathcal G}_j^{1/2}\right\|_{L_p}\le\sqrt2\,(c_I(p/2))^{1/2}
((1+D_0)d)^{1/2},
\]
from Lemma~\ref{lem:itermom}, the bound on $\|N_j\|_{L_p}$
just proved, and Doob's $L_p$ maximal inequality together with
\eqref{eq:chimoment} for the Brownian supremum, which contributes
$\frac{p}{p-1}\sqrt{1+p}\sqrt{\eta d}$.
\end{proof}
\textbf{Step 3: the local discretization error.}
Recall the coefficient regularity implied by
Assumption~\ref{ass:main}(A)--(B): $a_{\min}\le a\le a_{\max}$ and
$\nabla a=a\nabla(U-U_0)$, so that $a$ is Lipschitz with constant
$c_F=a_{\max}\omega_1$ (the same constant entering $C_2$ in
\eqref{eq:Ebound}); and the force $g=a\nabla U_0$ obeys, for all $x,y$,
\begin{equation}\label{eq:glip}
  \|g(x)-g(y)\|\le a_{\max}L\,\|x-y\|+c_F\,\|x-y\|\,\|\nabla U_0(y)\| ,
\end{equation}
as follows from the decomposition
\[
g(x)-g(y)=a(x)\left(\nabla U_0(x)-\nabla U_0(y)\right)
+\left(a(x)-a(y)\right)\nabla U_0(y).
\]

We are now ready to estimate the local error term $L_j$ in
\eqref{eq:split}. The following lemma shows that it is of order
$\eta^{2}$ with a prefactor of order $\sqrt d$; after division by the
contraction gap $\frac12\theta\eta$ in the recursion \eqref{eq:recursion}
below, this is what produces the $\mathcal O(\sqrt d\,\eta)$ bias term in
\eqref{eq:disc-bound}.

\begin{lemma}[Local strong error]\label{lem:local}
There are $\bar{\eta}\in(0,1]$ and $C_\ast>0$, depending only on
$(\gamma,m,L,\omega_0,\omega_1,\eta_{0},\alpha,c,\kappa,D_0)$ and not on $d$,
such that for all $\eta\le \bar{\eta}$ and all $j$,
\[
  L_j\ \le\ C_\ast\,\sqrt d\,\eta^2 .
\]
Moreover, the contribution to $C_\ast$ carrying $c_F$ vanishes as
$U_0\to U$.
\end{lemma}

\begin{proof}
Fix the step, write $\tau\in[0,\eta]$, and abbreviate the common left endpoint
$\hat V_0:=\tilde V_{j\eta}$, $\hat X_0:=\tilde X_{j\eta}$. The interpolant
\eqref{eq:interp} and the exact trajectory share the same additive Brownian
terms, which cancel in $\Xi$, leaving two pathwise identities.

\textbf{Velocity.}
The auxiliary exact trajectory $(\hat X,\hat V)$ solves
the UALD \eqref{eq:UALD:1}--\eqref{eq:UALD:2} on $[j\eta,(j+1)\eta]$; writing
its velocity equation \eqref{eq:UALD:1} in the variation-of-constants (mild)
form gives
\[
  \hat V_{j\eta+\tau}=\psi_0(\tau)\hat V_0
  -\int_0^\tau\psi_0(\tau-s)\,g(\hat X_{j\eta+s})\,ds
  +\sqrt{2\gamma}\int_0^\tau\psi_0(\tau-s)\,dB_{j\eta+s},
\]
with $g=a\nabla U_0$ as in \eqref{eq:glip}. The interpolant \eqref{eq:interp}
has exactly the same form with the force frozen at the left endpoint, since by
\eqref{eq:psidef},
\[
\psi_1(\tau)g(\hat X_0)=\int_0^\tau\psi_0(\tau-s)g(\hat X_0)\,ds.
\]
Subtracting the two identities, the initial and the
Brownian terms cancel and, using $0\le\psi_0\le1$, \eqref{eq:glip}, and
\[
\left\|\hat X_{j\eta+s}-\hat X_0\right\|=\left\|\int_0^s a(\hat X_{j\eta+r})\hat V_{j\eta+r}\,dr\right\|\le
a_{\max}sM_j,
\]
we get
\begin{equation}\label{eq:dV}
\left\|\tilde V_{j\eta+\tau}-\hat V_{j\eta+\tau}\right\|
  \le\int_0^\tau\|g(\hat X_{j\eta+s})-g(\hat X_0)\|\,ds
  \le\frac12\,a_{\max}\left(a_{\max}L+c_FN_j\right)M_j\,\tau^2 .
\end{equation}

\textbf{Position.} By \eqref{eq:interp} and \eqref{eq:psidef}, the
interpolated position satisfies
\[
\tilde X_{j\eta+\tau}=\hat X_0+a(\hat X_0)\int_0^\tau\tilde V_{j\eta+s}\,ds,
\]
whereas \eqref{eq:UALD:2} gives
\[
\hat X_{j\eta+\tau}=\hat X_0+\int_0^\tau a(\hat X_{j\eta+s})\hat V_{j\eta+s}\,ds;
\]
that is, the scheme freezes the \emph{speed} $a$ as well as the force.
Splitting the difference accordingly,
\[
\tilde X_{j\eta+\tau}-\hat X_{j\eta+\tau}
=a(\hat X_0)\int_0^\tau\left(\tilde V_{j\eta+s}-\hat V_{j\eta+s}\right)\,ds
+\int_0^\tau\left(a(\hat X_0)-a(\hat X_{j\eta+s})\right)\hat V_{j\eta+s}\,ds,
\]
and using
$|a(\hat X_0)-a(\hat X_{j\eta+s})|\le c_Fa_{\max}sM_j$ and $\|\hat V_{j\eta+s}\|\le M_j$,
\begin{equation}\label{eq:dX}
\left\|\tilde X_{j\eta+\tau}-\hat X_{j\eta+\tau}\right\|
  \le\frac16\,a_{\max}^2\left(a_{\max}L+c_FN_j\right)M_j\,\tau^3
  +\frac12\,c_Fa_{\max}M_j^2\,\tau^2 .
\end{equation}

\textbf{Assembly.} Evaluate \eqref{eq:dV}--\eqref{eq:dX} at $\tau=\eta\le1$ and
use $|w|_S\le\left(\lambda_{\max}(S)\right)^{1/2}\|w\|$:
\[
  \left|\Xi_{(j+1)\eta}\right|_S
  \le\left(\lambda_{\max}(S)\right)^{1/2}\eta^2\left[
  c_3\left(a_{\max}L+c_FN_j\right)M_j+\frac12a_{\max}c_FM_j^2\right],
\]
with
\[
c_3:=a_{\max}\left(\frac12+\frac16a_{\max}\right).
\]
For any nonnegative $P$, Cauchy--Schwarz inequality and Lemma~\ref{lem:statmom}(iii)
give
\begin{align*}
\left(\E\left[P^4\Psi_{(j+1)\eta}'\right]\right)^{1/4}\le
\left(\E\left[P^8\right]\right)^{1/8}\,\left(\E\left[e^{2\alpha\mathcal G'}\right]\right)^{1/8}
\le c_w^{1/4}\,\|P\|_{L_8},
\end{align*}
so that with Lemma~\ref{lem:pathmom} ($p\le16$) and
\begin{align*}
&\|M_jN_j\|_{L_8}\le\|M_j\|_{L_{16}}\|N_j\|_{L_{16}}\le c_\natural^2d,
\\
&\|M_j^2\|_{L_8}=\|M_j\|_{L_{16}}^2\le c_\natural^2d,
\end{align*}
we get
\begin{align*}
  L_j&\le\left(\lambda_{\max}(S)\right)^{1/2}c_w^{1/4}\eta^2\left[
  c_3a_{\max}L\,c_\natural\sqrt d
  +\left(c_3+\frac12a_{\max}\right)c_F\,c_\natural^2\,d\right]
  \\
  &=\left(\lambda_{\max}(S)\right)^{1/2}c_w^{1/4}c_\natural\left[
  c_3a_{\max}L+\left(c_3+\frac12a_{\max}\right)c_\natural\,c_F\sqrt d
  \right]\sqrt d\,\eta^2 .
\end{align*}
By \eqref{eq:cFd} the factor $c_F\sqrt d$ is bounded by a dimension-free
constant under the standing condition \eqref{eq:standing}; absorbing
everything into $C_\ast$ proves the bound, and the $c_F$-part of the
bracket vanishes as $U_0\to U$.
\end{proof}

\textbf{Step 4: the recursion and completing the proof of Theorem~\ref{thm:disc}.}
Insert Lemmas~\ref{lem:cont-step} and~\ref{lem:local} into the splitting
\eqref{eq:split}: for $\eta\le \bar{\eta}\wedge\theta^{-1}$,
\begin{equation}\label{eq:recursion}
  A_{j+1}\ \le\ \left(1-\frac12\theta\eta\right)A_j+C_\ast\sqrt d\,\eta^2 .
\end{equation}
Because the local error $L_j\le C_\ast\sqrt d\,\eta^2$ of Lemma~\ref{lem:local}, i.e.\ the last term in \eqref{eq:recursion},
no longer involves $A_j$, no a priori bound or
bootstrap is needed: unfolding \eqref{eq:recursion} directly,
\[
  A_k\le\left(1-\frac12\theta\eta\right)^kA_0
  +C_\ast\sqrt d\,\eta^2\sum_{i=0}^{k-1}\left(1-\frac12\theta\eta\right)^i
  \le\left(1-\frac12\theta\eta\right)^kA_0
  +\frac{2C_\ast\sqrt d}{\theta}\,\eta .
\]
Combining with the endpoint relations \eqref{eq:W2-to-L2} and
\eqref{eq:ends},
\[
  \mathcal W_2(\nu_k,\pi_X)
  \le\left(\lambda_{\min}(S)\right)^{-1/2}A_k
  \le\sqrt{\frac{\lambda_{\max}(S)}{\lambda_{\min}(S)}}
  \left(1-\frac12\theta\eta\right)^k\mathcal D_0
  +\frac{2C_\ast}{\theta\sqrt{\lambda_{\min}(S)}}\,\sqrt d\,\eta ,
\]
which is \eqref{eq:disc-bound} with
$C_{\mathrm{loc}}:=C_\ast/\sqrt{\lambda_{\min}(S)}$ and
$\bar{\eta}$ redefined as the minimum of the step ceilings of
Lemmas~\ref{lem:itermom}--\ref{lem:local}, $(2\gamma)^{-1}$
and $\theta^{-1}\wedge1$.
\hfill$\qed$
\begin{remark}[The novelty compared to classical KLMC]\label{rem:novelty}
The scheme \eqref{eq:recursion} is the familiar contraction-plus-local-error
recursion of the KLMC analysis \cite{dalalyan2018kinetic}. Two features are
specific to anchoring, and both are visible in the local error. First, the
position equation \eqref{eq:XL:modified:2} is driven at the state-dependent
\emph{speed} $a(x)$, which the scheme freezes; this produces the term
$\frac12c_Fa_{\max}M_j^2$ in \eqref{eq:dX}, quadratic in the velocity, with no
counterpart in KLMC. Second, the anchored force $g=a\nabla U_0$ is not
globally Lipschitz, and the amplitude-dependent term $c_FN_j$ in
\eqref{eq:glip}--\eqref{eq:dV} reflects this. Both carry the prefactor
$c_F=a_{\max}\omega_1$ and hence vanish as $U_0\to U$. Because these two
terms are \emph{products} of unbounded factors, the $L_2$ deviation splitting
of \cite{dalalyan2018kinetic} is no longer available, and we replace it by
the uniform polynomial moment bounds of Lemma~\ref{lem:itermom}; this in turn
forces the weighted $L_4$ metric \eqref{eq:Aj} and the power $p=2$ in
Lemma~\ref{lem:cont-step}.
\end{remark}

\begin{remark}[The initial weighted distance]\label{rem:D0}
By Cauchy--Schwarz inequality against the weight and \eqref{eq:cw},
$\mathcal D_0\le c_w^{1/4}\inf_{\zeta\in\Gamma_{\!\perp}(\nu_0,\pi)}(\E_\zeta\|x-x'\|^8)^{1/8}$, an
$L_8$-Wasserstein-type distance between $\nu_0$ and $\pi_X$. For a
deterministic start $\nu_0=\delta_{x_0}$, since $U\ge U_0-\omega_0$ makes
$\pi_X$ sub-Gaussian around $x_\ast$ with proxy variance $\lesssim1/m$,
\[
  \mathcal D_0\ \le\ c_w^{1/4}\left(\E_{\pi_X}\|x_0-X\|^8\right)^{1/8}
  \ \lesssim\ \|x_0-x_\ast\|+\sqrt{d/m}\,.
\]
\end{remark}

As a corollary of Theorem~\ref{thm:disc}, we obtain the following iteration complexity
for UALMC algorithm.
\begin{corollary}[Iteration complexity]\label{cor:iteration}
Take $\epsilon=\gamma/2$ and $\gamma\ge2\sqrt L$ in
Proposition~\ref{prop:linear}, so that $\kappa\asymp m/\gamma$, and assume
small mismatch, so that $\theta=\kappa-\delta_{(2)}/4\asymp m/\gamma$. Then
\eqref{eq:disc-bound} gives
$\mathcal W_2(\nu_K,\pi_X)\le\varepsilon$ after $K$ iterates, where
\[
K=\mathcal{O}\!\left(\frac1{\theta\eta}\log\frac{c_\flat\mathcal D_0}
  {\varepsilon}\right)
  \quad\text{with}\quad
  \eta\asymp\frac{\theta\,\varepsilon}{C_{\mathrm{loc}}\sqrt d},
  \qquad\text{i.e.}\quad
  K=\widetilde {\mathcal{O}}\left(\frac{\sqrt d}
  {\varepsilon}\right),
\]
where the notation $\widetilde{\mathcal O}$ hides logarithmic
factors as well as the constants $\theta$ and $C_{\mathrm{loc}}$, which are
$\mathcal O(1)$ in $d$ and $\varepsilon$.
\end{corollary}

\begin{proof}
With $\epsilon=\gamma/2$ we have $\epsilon(\gamma-\epsilon)=\gamma^2/4$ and
$q(h)=\gamma^2-4h+4h^2/\gamma^2$ in \eqref{eq:kappa}; for
$\gamma\ge2\sqrt L\ge2\sqrt m$ the convex function $q$ is maximized over
$[m,L]$ at $h=m$, and
\[
\kappa=\frac12\left(\gamma-\sqrt{q(m)}\right)
=\frac12\gamma\left(1-\sqrt{1-4m/\gamma^2+4m^2/\gamma^4}\right)\asymp m/\gamma.
\]
Taking $\eta\asymp\theta\varepsilon/(C_{\mathrm{loc}}\sqrt d)$ makes the bias
term in \eqref{eq:disc-bound} at most $\varepsilon/2$, and
$(1-\frac12\theta\eta)^k\le e^{-\theta\eta k/2}$ makes the first term at most
$\varepsilon/2$ once $k\ge\frac{2}{\theta\eta}\log(2c_\flat\mathcal D_0/\varepsilon)$.
Substituting the value of $\eta$ gives the stated iteration complexity.
\end{proof}

\begin{remark}[Dimension dependence]
In Corollary~\ref{cor:iteration}, we obtained the same scaling $\tilde{\mathcal{O}}(\sqrt d/\varepsilon)$ as the classical KLMC \cite{dalalyan2018kinetic}, with
all anchoring overhead confined to $C_{\mathrm{loc}}$ and to the
perturbative gap $\delta_{(2)}$ inside $\theta$. One caveat should be made
explicit: the standing condition \eqref{eq:standing} is genuinely
dimension-dependent. Since $b=\gamma d$, the optimized deficit scales like
$\beta_0C_2\sqrt{\gamma d/c'}$, so \eqref{eq:standing} (like $\delta<2\kappa$
in the continuous theory) implicitly requires
$\omega_1=\mathcal{O}(d^{-1/2})$; this is inherited from the continuous
perturbation argument, not created by the discretization, and it is
exactly the mechanism (via \eqref{eq:cFd}) that keeps $C_{\mathrm{loc}}$
dimension-free.
\end{remark}

\begin{remark}[Comparison with ALMC]\label{rmk:comparison}
The iteration complexity for anchored Langevin Monte Carlo (ALMC) algorithm,
that is based on the Euler-Maruyama discretization of \eqref{overdamped:SDE:modified}
is obtained in Theorem 14 in \cite{GNZZ2025}, which states that
\begin{equation}\label{ALMC:result}
\mathcal{W}_{2}(\nu_{k},\pi)\leq\sqrt{2}e^{-(m-\varsigma) k\eta}\mathcal{W}_{2}(\nu_{0},\pi)+\sqrt{2}C\eta^{\frac{1}{2}},
\end{equation}
where $\nu_{k}$ is the distribution of the $k$-th iterate of ALMC
and $\pi$ is the target distribution, and the constants $\varsigma,C$
are explicitly given such that $C=\mathcal{O}(\sqrt{d})$.
Here $\varsigma$ denotes the constant called $\alpha$ in
\cite[Assumption 12]{GNZZ2025}, renamed to avoid a clash with the parameter
$\alpha$ of the exponential weight in \eqref{eq:alpharange}.
We can deduce from \eqref{ALMC:result} that
$\mathcal{W}_{2}(\nu_{K},\pi)\leq\varepsilon$
after $K$ iterates, where
\[
K=\mathcal{O}\!\left(\frac1{(m-\varsigma)\eta}\log\frac{\mathcal{W}_{2}(\nu_{0},\pi)}
  {\varepsilon}\right)
  \quad\text{with}\quad
  \eta\asymp\frac{\varepsilon^{2}}{C^{2}}\asymp\frac{\varepsilon^{2}}{d},
  \qquad\text{i.e.}\quad
  K=\widetilde {\mathcal{O}}\left(\frac{d}
  {\varepsilon^{2}}\right).
\]
Hence, we conclude that the iteration complexity
$K=\widetilde{\mathcal O}(\sqrt d/\varepsilon)$ of
Corollary~\ref{cor:iteration}
for UALMC algorithm has better dependence on both dimension $d$
and accuracy level $\varepsilon$
than the complexity $K=\widetilde{\mathcal O}(d/\varepsilon^{2})$ of ALMC;
the improvement is by a factor $\sqrt d$ in the dimension and by a factor
$1/\varepsilon$ in the accuracy.
\end{remark}

\section{Numerical Experiments}\label{sec:numerical}

We conduct numerical experiments to validate the
theory and to investigate the practical performance of UALMC. We apply the
algorithm to four tasks, in increasing order of complexity: (i) simulating
multivariate Laplace distributions (synthetic data, where the exact
$2$-Wasserstein distance to the target is measurable); (ii) Bayesian logistic
regression with the non-differentiable regularizers above, on two real data
sets from the UCI Machine Learning Repository; (iii) Bayesian inference for a
two-layer feedforward neural network with ReLU activation, on the same real
data sets; and (iv) sampling heavy-tailed multivariate Student's $t$
distributions (synthetic), the regime of Assumption~\ref{ass:HQ}. Every
reference potential is evaluated in closed form, using the formulas of
Appendix~\ref{app:closed:form} and their extensions derived there.

\subsection{Setup}\label{sec:setup}

In our numerical examples, $U$ always takes the form:
\begin{equation}
U(x)=f(x)+g(x),
\end{equation}
where $f(x)$ is differentiable but $g(x)$ is not.
We construct the anchored reference potential as
\begin{equation}
U_0(x)=f(x)+g_{0}(x),
\end{equation}
where $g_{0}(x)$ is a smoothed version of $g(x)$,
e.g. the Gaussian smoothing; see the discussions in Appendix~\ref{sec:nonsmooth}:
\begin{equation}
g_{0}(x)=\mathbb{E}_{\xi}[g(x+\mu\xi)],\qquad\xi\sim\mathcal{N}(0,\Id),
\end{equation}
and in \cite{GNZZ2025} its gradient $\nabla g_{0}(x)$ is computed using a Monte Carlo step.
As in \cite{GNZZ2025}, we consider the following three non-differentiable regularizers:
\begin{itemize}
    \item Lasso regularizer:
    $g(x) = \lambda\sum_{i=1}^{d}|x_i|$.
        \item Minimax concave penalty (MCP) regularizer (see e.g. \cite{Zhang10}) is $g(x) = \sum_{i=1}^{d}p_{\lambda}(x_i)$ where $p_{\lambda}(x)$ has the form:
    $p_{\lambda}(x) =
            \lambda |x|-\frac{x^2}{2a}$ if $|x| \leq a\lambda$, and
            $p_{\lambda}(x) =\frac{a\lambda^2}{2}$ otherwise,
where $\lambda>0$ is the regularization parameter and $a>1$ the scaling
parameter.
    \item Smoothly clipped absolute deviation (SCAD) regularizer (see e.g. \cite{Fan01}) is $g(x) = \sum_{i=1}^{d}p_{\lambda}(x_i)$ where $p_{\lambda}(x)$ is defined as (with $a>1$):
    \begin{equation}
        p_{\lambda}(x) = \begin{cases}
            \lambda |x| &\quad \text{if } |x| \leq \lambda, \\
            \frac{2a\lambda |x|-x^2-\lambda^2}{2(a-1)} &\quad \text{if } \lambda < |x| \leq a\lambda, \\
            \frac{\lambda^2(a+1)}{2} &\quad \text{otherwise}.
        \end{cases}
    \end{equation}
\end{itemize}

Indeed, all three examples of non-differentiable regularizers are separable, $g(x)=\sum_{i=1}^{d}p(x_i)$, with $p$
continuous and \emph{piecewise quadratic}; for such $p$ the Gaussian
convolution \eqref{defn:g:0} and its derivative are available in closed form
in terms of $\Phi$ and $\phi$, so that no Monte Carlo simulation is needed; see Appendix~\ref{app:closed:form} for details.
The UALMC iteration \eqref{eq:scheme} requires, at each step, the two
quantities
\[
  a_{k}=e^{U(x_{k})-U_{0}(x_{k})}=e^{g(x_{k})-g_{0}(x_{k})},
  \qquad
  \nabla U_{0}(x_{k})=\nabla f(x_{k})+\nabla g_{0}(x_{k}),
\]
both of which are now evaluated to machine precision in $\mathcal{O}(d)$
operations using \eqref{eq:cf:lasso}--\eqref{eq:cf:scad} in Appendix~\ref{app:closed:form}. This removes the
inner Monte Carlo estimator of $g_{0}$ and $\nabla g_{0}$: besides the obvious
speed-up, it also removes a source of bias, since $a_{k}$ depends on $g_{0}$
\emph{nonlinearly} (through the exponential), so that an unbiased Monte Carlo
estimator of $g_{0}(x_{k})$ does \emph{not} produce an unbiased estimator of
$a_{k}$.

\paragraph{Simulating Laplace distributions (synthetic).}
The sampling benchmark is the product Laplace distribution with density
\begin{equation}\label{eq:laplacetarget}
  \pi(x)=\prod_{i=1}^{d}\frac{1}{\sqrt2}\,e^{-\sqrt2\,|x_{i}|}
  \;\propto\; e^{-\sqrt2\,\|x\|_{1}},
  \qquad x\in\R^{d},
\end{equation}
i.e.\ $U(x)=\sqrt2\,\|x\|_{1}$ is the Lasso regularizer with $f\equiv0$ and
$\lambda=\sqrt2$, chosen so that each marginal is Laplace with scale
$b=1/\sqrt2$ and unit variance. The reference potential $U_{0}$ is its
closed-form Gaussian smoothing with parameter $\mu$. Because the marginal
quantile function is explicit, the $2$-Wasserstein distance to the true
target is \emph{measured} rather than proxied, which makes it ideal to validate
the theory.

\paragraph{Bayesian logistic regression with regularizers (real data).}
Given data $(X_{i},y_{i})_{i=1}^{n}$ with $X_{i}\in\R^{d}$,
$y_{i}\in\{0,1\}$, the target is the posterior density of the weights
$x\in\R^{d}$,
\begin{equation}\label{eq:logistictarget}
  \pi(x)\;\propto\;e^{-U(x)},\qquad
  U(x)=f(x)+m_{0}\|x\|^{2}+g(x),
\end{equation}
where
\begin{equation}\label{eq:logisticnll}
  f(x)=-\sum_{i=1}^{n}\Bigl[y_{i}\log\sigma\bigl(x^{\top}X_{i}\bigr)
  +(1-y_{i})\log\bigl(1-\sigma\bigl(x^{\top}X_{i}\bigr)\bigr)\Bigr]
\end{equation}
is the negative log-likelihood with the sigmoid $\sigma(z)=1/(1+e^{-z})$, and
$e^{-m_{0}\|x\|^{2}-g(x)}$ is the prior: a Gaussian ridge factor, which
provides the strong convexity required by Assumption~\ref{ass:main}(B),
times a non-differentiable penalty factor with $g$ one of the regularizers
above. The reference potential replaces $g$ by its Gaussian smoothing,
$U_{0}=f+m_{0}\|x\|^{2}+g_{0}$. We use two real data sets from the UCI
Machine Learning Repository~\cite{Dua19}: the Breast Cancer Wisconsin
(Diagnostic) data set ($n=569$ samples, $d=30$ features of cell nuclei in
digitized images of breast masses) and the Banknote Authentication data set
($n=1372$ samples, $d=4$ features extracted from images of genuine and forged
banknotes), with features normalized to $[-1,1]$ and no intercept.
Throughout, $m_{0}=0.5$, $\lambda=1$, and $a=10$ for SCAD and MCP, matching
\cite{GNZZ2025}; the elastic net and the capped $\ell_{1}$ penalty of
Section~\ref{sec:penalties} use $\lambda_{2}=0.5$ and $\theta=1.5$. Each run
evolves $20$ particles, initialized from $\mathcal N(0,\frac14 I_{d})$, for
$1{,}500$ iterations with $\eta=0.02$ and $\gamma=1$. Accuracy is the
training-set accuracy of the particles' classifiers, averaged over the
particles and over $3$ runs.

\paragraph{Bayesian neural network (real data).}
The third task is Bayesian inference for a two-layer feedforward network on
the same two data sets: $32$ hidden units with ReLU activation
$r(t)=\max(t,0)$ feeding one sigmoid output unit, and
\begin{equation}\label{eq:nntarget}
  \pi(w)\propto e^{-U(w)},\quad
  U(w)=\mathrm{BCE}(w)+m_{0}\|w\|^{2},\quad
  \mathrm{pred}(w_{1},w_{2})
  =\sigma\Bigl(\textstyle\sum_{j=1}^{32}r\bigl(X w_{1,j}\bigr)\,w_{2,j}\Bigr),
\end{equation}
where $\mathrm{BCE}$ is the binary cross-entropy of those predictions,
$w=(w_{1},w_{2})$ collects both layers' weights ($992$ parameters on Breast
Cancer), and a small ridge $m_{0}=0.01$ is included for stability. The
non-differentiability now sits in the \emph{activation} rather than in a
penalty: the reference potential smooths $r$ by
Corollary~\ref{cor:relu}, so this experiment exercises the closed forms
outside the penalty class. We also use the leaky ReLU $\max(t,0.1t)$. Each
run evolves $10$ particles, initialized from $\mathcal N(0,4I)$, for
$1{,}200$ iterations with $\gamma=2$ and $\eta\in\{0.003,0.01\}$, two runs
per step size.

\paragraph{Sampling heavy-tailed distributions (synthetic).}
The heavy-tailed benchmark is the multivariate Student's $t$ distribution
with $\nu$ degrees of freedom and density
\begin{equation}\label{eq:studenttarget}
  \pi(x)\;\propto\;\left(1+\frac{\|x\|^{2}}{\nu}\right)^{-(\nu+d)/2},
  \qquad x\in\R^{d}.
\end{equation}
In the notation of Assumption~\ref{ass:HQ} this is
$q(x)=1+\|x\|^{2}/\nu$, $U=(\beta+1)\log q$ and $U_{0}=\beta\log q$ with
$\beta=(\nu+d)/2-1$; with $\Sigma=\Id$ we have $m_{q}=L_{q}=2/\nu$ and the
condition of Assumption~\ref{ass:HQ} reduces to $\nu>8$. We sweep $\nu$
across both sides of that threshold, with $\gamma=3$, step sizes
$\eta\in\{0.05,0.01\}$, simulated time $T=40$, $200{,}000/d$ chains and two
runs per setting.

\begin{sloppypar}
\paragraph{Algorithms.} UALMC is \eqref{eq:scheme} with the correlated
increment of covariance $C$ in \eqref{eq:Cnoise}, initialized as in
\eqref{eq:init} with $v_{0}\sim\mathcal N(0,I_{d})$ independent of
$x_{0}$. ALMC is the overdamped anchored algorithm of \cite{GNZZ2025}.
We also run kinetic Langevin Monte Carlo (KLMC) and unadjusted Langevin algorithm (ULA) on
$U_{0}$; these target $\pi_{0}\propto e^{-U_{0}}$ and isolate the effect of the
anchoring.
\end{sloppypar}

\paragraph{Measurement.} An empirical $\mathcal{W}_{2}$ carries a finite-sample floor of
its own, so alongside every measurement we compute the $\mathcal{W}_{2}$ between the
same number of exact i.i.d.\ draws and the true quantile function, averaged over
repetitions, and report the discretization bias as
$\sqrt{\max\left(\mathcal{W}_{2}^{2}-\mathrm{floor}^{2},0\right)}$. Where a quantity falls at or
below that floor we say so rather than reporting it. Iteration counts $K$ are
medians over repetitions, minimized over the free parameters as the definition
of iteration complexity requires, with checkpoints every five iterations.

\subsection{Invariance and the Anchoring Mechanism}\label{sec:mech}

Theorem~\ref{thm:invariant} states that UALD preserves
$\pi(x,v)\propto e^{-\|v\|^{2}/2-U(x)}$, so the algorithm targets $\pi$ and not
the smoothed law $\pi_{0}$. This is what distinguishes the construction from
smoothing-based methods such as \cite{Chatterji20}, which converge to the
smoothed target and carry an asymptotic bias growing with the smoothing
parameter $\mu$. The effect is directly observable when simulating the
univariate Laplace distribution \eqref{eq:laplacetarget} at
$d=1$, $\eta=0.02$, $\gamma=3$ (medians over four repetitions after $T=300$
of simulated time):

\begin{center}
\begin{tabular}{lrr}
\toprule
$\mathcal{W}_{2}$ to the \emph{true} Laplace target & $\mu=0.5$ & $\mu=1.0$ \\
\midrule
UALMC \eqref{eq:scheme}      & $0.010$ & $0.008$ \\
kinetic Langevin on $U_{0}$         & $0.094$ & $0.243$ \\
\bottomrule
\end{tabular}
\end{center}

The un-anchored sampler converges to $\pi_{0}$ and its error to $\pi$ grows with
$\mu$, as it must; the anchored sampler's does not --- both of its entries are
within run-to-run noise of the estimator's sampling floor ($\approx0.007$
here), so no bias is resolvable in either column. At $\mu=1$ the improvement
is a factor of at least $30$ in $\mathcal{W}_{2}$, and, unlike for the smoothing-based benchmark, the error of
the anchored sampler is insensitive to the choice of $\mu$ on this target.

\subsection{Discretization Bias in the Step Size}\label{sec:eta}

\begin{figure}[t]
\centering
\includegraphics[width=0.92\textwidth]{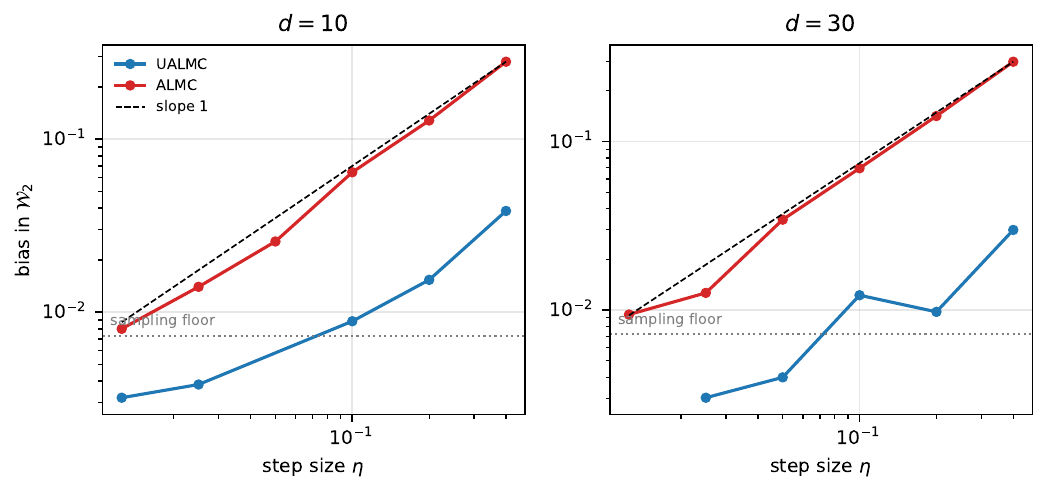}
\caption{Discretization bias against step size on the product Laplace target,
$\mu=1/d$, after $T=300$ of simulated time, for UALMC \eqref{eq:scheme} with
$\gamma=1.5\sqrt L$ and ALMC \cite{GNZZ2025}; one run per step size. Dotted:
the estimator's sampling floor; points whose measured $\mathcal{W}_{2}$ does not
exceed the floor are omitted.}
\label{fig:bias}
\end{figure}

We next measure the asymptotic bias in the step size, again on the synthetic
Laplace target \eqref{eq:laplacetarget}, where it is exactly measurable.
Theorem~\ref{thm:disc} bounds it by
$(2C_{\mathrm{loc}}/\theta)\sqrt d\,\eta$, and Figure~\ref{fig:bias} is
consistent with this: at $d=10$,
$\eta=0.4$ the measured UALMC bias is $0.038$ against $0.280$ for ALMC
--- an improvement by a factor of about $7$ at the largest step size ---
and for $\eta\le0.05$ the UALMC bias falls below what the estimator
resolves, while ALMC needs $\eta\approx0.0125$ to reach the same level.
Equivalently, at matched bias the underdamped scheme tolerates a step size
four to seven times larger at $d=10$. The ALMC exponents, fitted over the
step sizes whose bias exceeds twice the floor, are $+1.14$
and $+1.04$ at $d=10,30$ (and $+1.04$ at $d=1$ in a shorter $T=30$ run).

\begin{remark}[On the comparison of rates]\label{rem:rates}
Theorem 14 of \cite{GNZZ2025} bounds the ALMC bias by
$\sqrt2\,C\eta^{1/2}$. That estimate comes from a mean-square argument, made
necessary by the state-dependent diffusion coefficient
$\sigma(x)=e^{(U-U_{0})/2}$ of the overdamped anchored SDE, and strong order
$1/2$ is generic in that setting. The distributional error measured by $\mathcal{W}_{2}$
is typically one order better than the pathwise error, so the observed
$\eta^{1}$ behavior is compatible with that bound rather than in conflict with
it. Both schemes are governed by upper bounds conservative to differing
degrees, and the comparison in Remark~\ref{rmk:comparison} is best read as a
comparison of guarantees rather than of realized performance.
\end{remark}

\subsection{Iteration Complexity in the Dimension}\label{sec:dim}

\begin{figure}[t]
\centering
\includegraphics[width=0.56\textwidth]{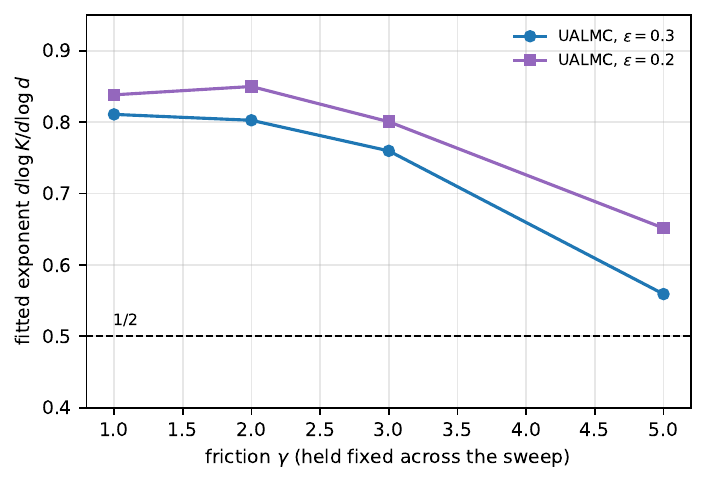}
\caption{Fitted exponent of $K$ in $d$, with $\gamma$ held fixed across the
sweep as in Corollary~\ref{cor:iteration}, against the value at which it is
held. Dashed: the exponent $1/2$ of Corollary~\ref{cor:iteration}.}
\label{fig:expgamma}
\end{figure}

We now test the dimension dependence of the iteration complexity, sweeping
$d$ from $1$ to $100$ on the synthetic Laplace family
\eqref{eq:laplacetarget} with $\mu=1/d$ and counting the iterations $K$ until
the measured $\mathcal{W}_{2}$ falls below $\varepsilon$.
Corollary~\ref{cor:iteration} gives $K=\widetilde{\mathcal O}(\sqrt
d/\varepsilon)$ with $\gamma$ \emph{fixed}, its contraction constant provided by
Proposition~\ref{prop:linear}, which requires $\gamma>\sqrt L$. The faithful
test is therefore to hold $\gamma$ fixed, sweep the dimension, and minimize $K$
over the step size alone.

\begin{table}[t]
\centering
\begin{tabular}{rrrrrrrrr}
\toprule
$\gamma$ & $d=1$ & $2$ & $5$ & $10$ & $20$ & $50$ & $100$ & exponent \\
\midrule
$1$ & 10 & 20 & 75 & 100 & 132 & 352 & 412 & $+0.811$ \\
$2$ & 15 & 20 & 50 & 105 & 140 & 415 & 452 & $+0.803$ \\
$3$ & 15 & 25 & 35 &  70 & 160 & 218 & 525 & $+0.760$ \\
$5$ & 25 & 35 & 45 &  62 & 108 & 255 & 280 & $\mathbf{+0.559}$ \\
\bottomrule
\end{tabular}
\caption{Iterations to $\mathcal{W}_{2}<0.3$ on the product Laplace target, $\mu=1/d$,
minimized over the step size with $\gamma$ held fixed, medians over four
repetitions. The exponent approaches the $1/2$ of
Corollary~\ref{cor:iteration} as $\gamma$ grows; at $\varepsilon=0.2$ the
value at $\gamma=5$ is $+0.651$.}
\label{tab:fixedgamma}
\end{table}

At $\gamma=5$ the fitted exponent is $+0.559$ against the predicted $1/2$, and
the approach is monotone in $\gamma$ at $\varepsilon=0.3$ and nearly so at
$\varepsilon=0.2$ (Figure~\ref{fig:expgamma}). This is what the
analysis predicts: Proposition~\ref{prop:linear} produces a contraction in
the constant twisted metric only for $\gamma>\sqrt L$, and
Corollary~\ref{cor:iteration} inherits that requirement, so agreement should
improve as $\gamma$ increases and degrade below the threshold, which is what is
observed.

\begin{remark}[The role of $\gamma$]\label{rem:gammarole}
Minimizing $K$ over $\gamma$ as well as $\eta$ raises the fitted exponent to
$+0.734\pm0.027$ at $\varepsilon=0.3$ (standard error of the least-squares
slope over the seven dimensions). Writing $K\asymp T_{\mathrm{mix}}/\eta\asymp
(\gamma/m)(\sqrt d/\varepsilon)$, an optimal friction growing like
$\gamma^{*}\propto d^{\,0.3}$,
consistent with the observed $\gamma^{*}$
rising from $1$ to $5$ across the sweep, that yields $K\propto d^{\,0.8}$. The two
measurements are accounted for by the same expression and differ only in whether
$\gamma$ is a constant or a free parameter. A smaller $\gamma$ buys a smaller
bias at fixed $d$ (Figure~\ref{fig:mu}, right) at the cost of a worse dimensional
scaling.
\end{remark}

\subsection{Choice of the Smoothing Scale}\label{sec:mu}

\begin{figure}[t]
\centering
\includegraphics[width=0.50\textwidth]{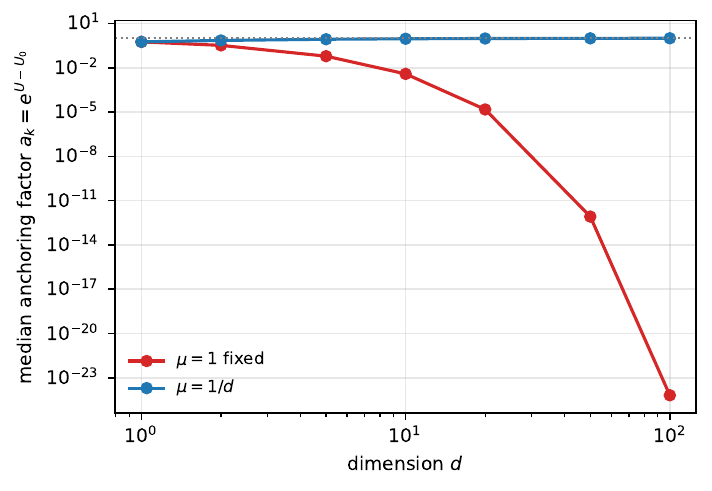}
\hfill
\includegraphics[width=0.46\textwidth]{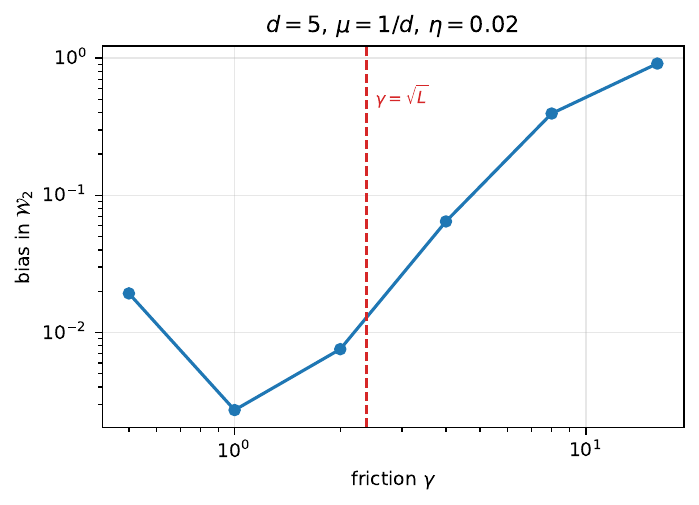}
\caption{Left: median anchoring factor $a_{k}=e^{U-U_{0}}$ at draws from the
target. Right: bias against friction at $d=5$, $\eta=0.02$; dashed:
$\gamma=\sqrt L$, the threshold of Proposition~\ref{prop:linear}.}
\label{fig:mu}
\end{figure}

Assumption~\ref{ass:main}(A) requires $\omega_{0}=\|U-U_{0}\|_{\infty}$
finite, and the constants of Theorem~\ref{thm:disc} depend on it. For a
separable penalty the mismatch is additive over coordinates, and for the
Gaussian-smoothed $\ell_{1}$ penalty it is available exactly:
\begin{equation}\label{eq:omega0}
  \omega_{0}(d)=d\,\lambda\,\mu\sqrt{2/\pi},\qquad \omega_{1}(d)=\lambda\sqrt d .
\end{equation}
Assumption~\ref{ass:main}(A) therefore holds for every $d$ and every $\mu$,
but keeping $\omega_{0}$, and hence the constants, bounded as the dimension
grows requires $\mu=O(1/d)$. The practical consequence is sharp on the
synthetic Laplace family: at $\mu=1$
the median $a_{k}$ falls from $0.56$ at $d=1$ to $6.6\times10^{-25}$ at
$d=100$, whereas at $\mu=1/d$ it rises to $0.99$
(Figure~\ref{fig:mu}, left). Since the position increment
carries the factor $a_{k}$, the chain is otherwise frozen. The same
mechanism operates on real data: in the Bayesian logistic regression
\eqref{eq:logistictarget} with the Lasso penalty
(Figure~\ref{fig:logistic}), the anchored methods collapse to chance
accuracy at $\mu=1$ on Breast Cancer ($d=30$) but are untouched on Banknote
($d=4$), exactly as the exponential-in-$d$ decay of $a_{k}$ predicts.

\begin{figure}[t]
\centering
\includegraphics[width=0.95\textwidth]{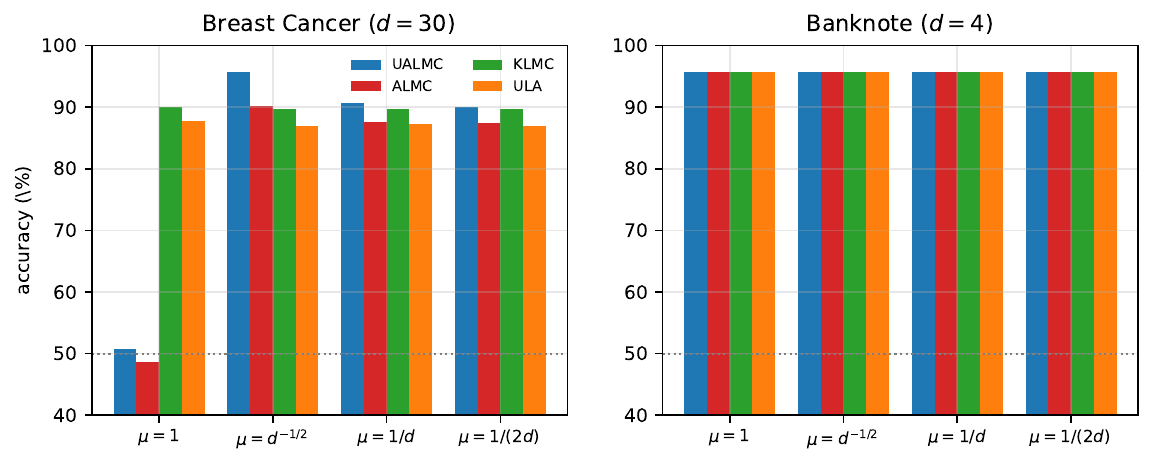}
\caption{Bayesian logistic regression with the $\ell_{1}$ penalty and exact
Gaussian smoothing, final accuracy against the choice of $\mu$, for both data
sets and all four samplers. Dotted: chance level. The anchored methods fail at
$\mu=1$ on Breast Cancer ($d=30$) and are unaffected on Banknote ($d=4$), as
\eqref{eq:omega0} predicts.}
\label{fig:logistic}
\end{figure}

\subsection{Non-Differentiable Regularizers}\label{sec:penalties}

\begin{figure}[t]
\centering
\includegraphics[width=0.86\textwidth]{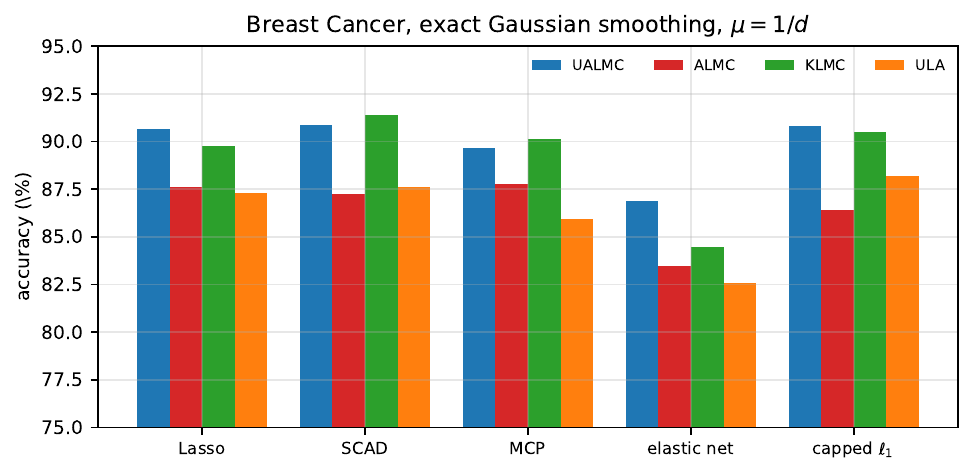}
\caption{Bayesian logistic regression on Breast Cancer ($n=569$, $d=30$) with
mixed penalty $U=f+m_{0}\|x\|^{2}+g$, exact Gaussian smoothing, $\mu=1/d$.}
\label{fig:penalties}
\end{figure}

We now conduct Bayesian logistic regression with the posterior
\eqref{eq:logistictarget} on the real Breast Cancer Wisconsin data set,
comparing five separable penalties for $g$: the Lasso, SCAD and MCP
regularizers of the main text, and, by Corollary~\ref{cor:more}, the elastic
net and the capped $\ell_{1}$.
UALMC leads ALMC on all five, by $1.9$ to $4.4$ accuracy points
(Figure~\ref{fig:penalties}), so the improvement of the underdamped over the
overdamped anchored algorithm is not an artifact of a particular
regularizer. The elastic net is hardest for every sampler, including KLMC and
ULA, which never form $a_{k}$, so the drop reflects the target, which carries
the extra ridge term $\lambda_{2}x^{2}$, rather than the anchoring: its
smoothing $\lambda_{2}(x^{2}+\mu^{2})$ (Corollary~\ref{cor:more}) adds only a
constant $\lambda_{2}\mu^{2}$ per coordinate to the mismatch $U-U_{0}$, which
rescales $a_{k}$ uniformly.

The choice of $\mu$ dominates everything else. On the same problem, with the
Lasso penalty,
\begin{center}
\begin{tabular}{lrrrr}
\toprule
Breast Cancer, $d=30$ & $\mu=1$ & $\mu=d^{-1/2}$ & $\mu=1/d$ & $\mu=1/(2d)$ \\
\midrule
UALMC & $50.8\%$ & $\mathbf{95.8\%}$ & $90.6\%$ & $90.1\%$ \\
ALMC  & $48.7\%$ & $90.2\%$ & $87.6\%$ & $87.4\%$ \\
KLMC  & $90.1\%$ & $89.8\%$ & $89.8\%$ & $89.8\%$ \\
median $a_{k}$ & $4\times10^{-7}$ & $0.6$ & $1.0$ & $1.0$ \\
\bottomrule
\end{tabular}
\end{center}
At $\mu=1$ the anchored methods collapse to chance while KLMC, which never
forms $a_{k}$, is unaffected. Among the scaled choices the best here is
$\mu=d^{-1/2}$ rather than $\mu=1/d$: keeping $\omega_{0}$ bounded is what
$\mu=\Theta(1/d)$ buys, but at finite $d$ there is a trade-off, since smaller
$\mu$ sharpens $U_{0}$ and inflates the contribution $2\lambda/(\mu\sqrt{2\pi})$
of $g_{0}$ to $L$. The
low-dimensional Banknote problem ($d=4$) is insensitive to the choice
(Figure~\ref{fig:logistic}).

\begin{remark}[Accuracy as a diagnostic]\label{rem:acc}
Classification accuracy is a functional of a decision boundary, not of the law,
and it need not separate an asymptotically unbiased sampler from a biased one:
KLMC, which targets $\pi_{0}$, matches UALMC to within a point at $\mu=1/d$
for the Lasso, SCAD, MCP and capped $\ell_{1}$ penalties, and at every $\mu$
on Banknote (Figure~\ref{fig:logistic}). We report it for comparability with \cite{GNZZ2025}; the
$\mathcal{W}_{2}$ experiments of Sections~\ref{sec:mech}--\ref{sec:dim} are what discriminate.
\end{remark}

\subsection{Group Penalties}\label{sec:group}

\begin{figure}[t]
\centering
\includegraphics[width=0.95\textwidth]{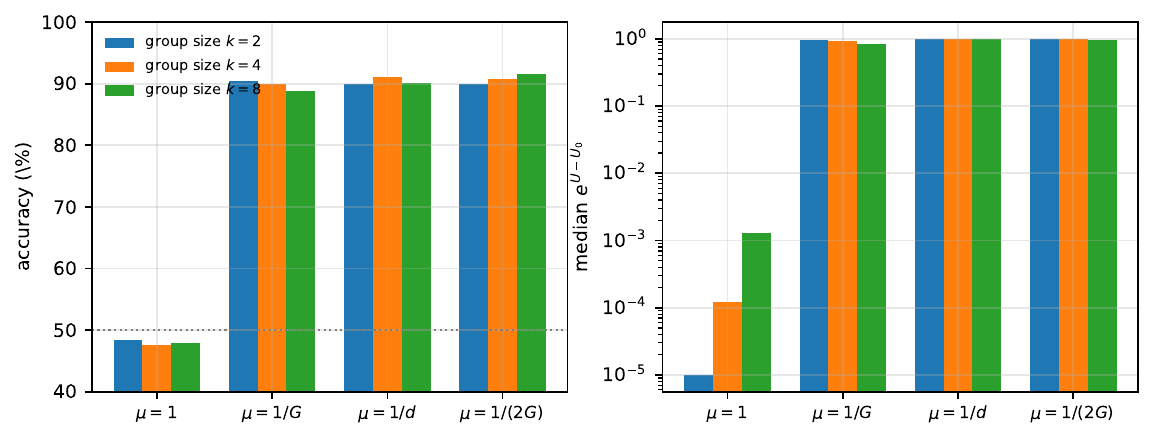}
\caption{Group-lasso logistic regression on Breast Cancer, UALMC. Left:
accuracy. Right: median anchoring factor. Dotted: chance level.}
\label{fig:group}
\end{figure}

Staying with Bayesian logistic regression on the Breast Cancer data set, we
replace the separable penalty in \eqref{eq:logistictarget} by the group-lasso
regularizer $g(x)=\lambda\sum_{j=1}^{G}\|x_{\mathcal G_{j}}\|_{2}$, where
$\mathcal G_{1},\ldots,\mathcal G_{G}$ partition the $d=30$ coordinates into
$G$ contiguous groups of size $k$, the last group taking the remainder when $k$
does not divide $d$ --- a non-differentiable penalty \emph{outside}
the separable piecewise-quadratic class, smoothed exactly by
Lemma~\ref{lem:group}. With this smoothing the
mismatch becomes
\[
  \omega_{0}=\lambda\,\mu\sum_{j=1}^{G}c_{|\mathcal G_{j}|}\le G\,\lambda\,\mu\,c_{k},
  \qquad c_{k}=\frac{\sqrt2\,\Gamma\left(\frac{k+1}{2}\right)}{\Gamma\left(\frac{k}{2}\right)},
\]
with equality when all groups have size $k$, so it is the number of
\emph{groups}, not the dimension, that sets the scale, and at fixed $k$ the
requirement becomes $\mu=O(1/G)$. At fixed $\mu=1$ coarser grouping
relieves the pressure but does not remove it:

\begin{center}
\begin{tabular}{rrrrr}
\toprule
$k$ & $G$ & $\omega_{0}$ at $\mu=1$ & median $e^{U-U_{0}}$ & accuracy \\
\midrule
2 & 15 & $18.8$ & $1\times10^{-5}$ & $48.4\%$ \\
4 &  8 & $14.4$ & $1\times10^{-4}$ & $47.6\%$ \\
8 &  4 & $10.6$ & $1\times10^{-3}$ & $47.9\%$ \\
\bottomrule
\end{tabular}
\end{center}

\noindent
All three remain at chance; with $\mu=1/G$ accuracy is $88.9$--$90.4\%$ and
the median anchoring factor lies in $[0.8,1.0]$ (Figure~\ref{fig:group}). Since $c_{k}$ grows like $\sqrt k$ while $G$ falls like
$1/k$, the product $Gc_{k}$ decreases only like $1/\sqrt k$.

\subsection{Heavy-Tailed Targets}\label{sec:heavy}

To sample heavy-tailed targets, we need to discretize \eqref{eq:UALD:heavy}
properly. We define the
\emph{tamed UALMC}:
\begin{equation}\label{eq:tamedscheme}
\begin{aligned}
  &v_{k+1}=\psi_0(\eta)v_{k}-\psi_1(\eta)\,\beta\nabla q(x_{k})
  +\sqrt{2\gamma}\,\xi_{k+1},
  \\
  &x_{k+1}=x_{k}+\bar a_{k}\left(\psi_1(\eta)v_{k}
  -\psi_2(\eta)\,\beta\nabla q(x_{k})\right)
  +\sqrt{2\gamma}\,\bar a_{k}\,\xi'_{k+1},
\end{aligned}
\end{equation}
by replacing $a_{k}$ in \eqref{eq:scheme} with
\begin{equation}\label{eq:tamed}
  \bar a_{k}:=\frac{q(x_{k})}{1+\sqrt\eta\,q(x_{k})}\ \le\
  \min\left\{q(x_{k}),\,\eta^{-1/2}\right\},
\end{equation}
while keeping the force untamed,
$g_{k}:=a_{k}\nabla U_{0}(x_{k})=\beta\nabla q(x_{k})$.

\begin{figure}[t]
\centering
\includegraphics[width=0.95\textwidth]{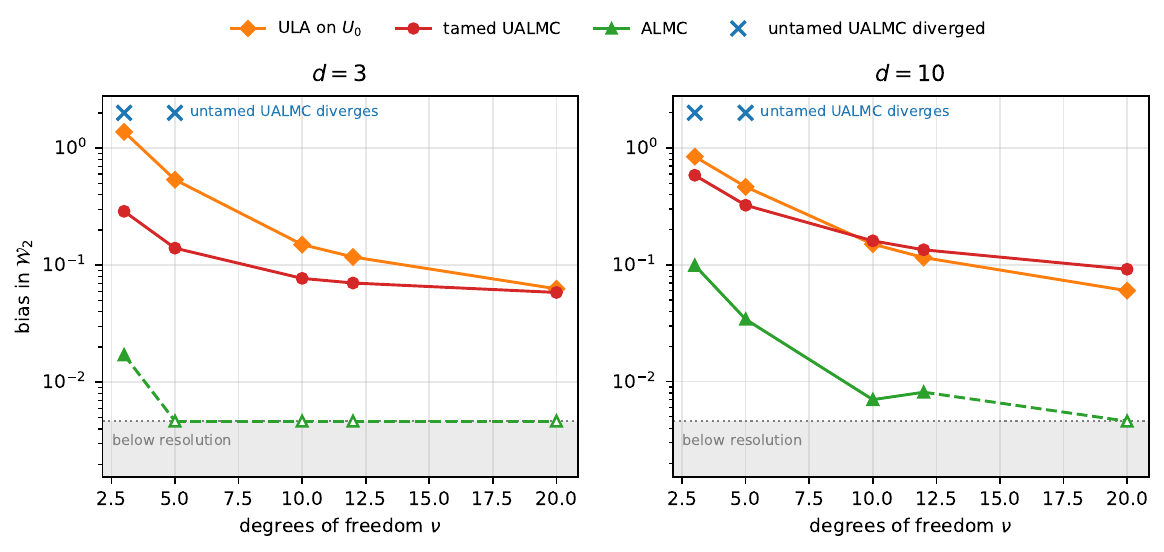}
\caption{Student's $t$ target \eqref{eq:studenttarget}, bias against degrees
of freedom $\nu$ at $\eta=0.01$ for $d=3$ (left) and $d=10$ (right), mean of
two runs. The shaded band lies below the estimator's sampling floor; an open
marker on its edge means the bias is not resolvable in both runs. $\times$: the
untamed iteration diverged in both runs. Where it survives, its bias cannot be
told apart from that of ALMC (the two runs overlap, or the bias is not
resolvable), so it is not drawn separately. Assumption~\ref{ass:HQ} requires
$\nu>8$.}
\label{fig:taming}
\end{figure}

We sample the synthetic multivariate Student's $t$ family
\eqref{eq:studenttarget} for $d\in\{1,3,5,10\}$ and
$\nu\in\{3,5,10,12,20\}$, on both sides of the $\nu>8$ threshold of
Assumption~\ref{ass:HQ} (Figure~\ref{fig:taming}). The anchored force $-\beta\nabla q$ does not
degenerate in the tails, where the
true force $-(\beta+1)\nabla q/q$ vanishes.
At $\eta=0.01$, ULA on $U_{0}$ is biased by up to
$1.90$ across the range, against at most $8.1\times10^{-3}$, at or near the
estimator's sampling floor of about $4\times10^{-3}$, for ALMC inside the range
of Assumption~\ref{ass:HQ} and for the untamed kinetic iteration where it
survives; the tamed scheme \eqref{eq:tamedscheme} carries a larger bias,
$5\times10^{-2}$ to $1.6\times10^{-1}$ over the same range, the price of the
stability established next.

\paragraph{The taming is necessary.} The tamed scheme
\eqref{eq:tamedscheme} did not diverge in any of the $80$ runs. The untamed
iteration diverged in every run at $\nu=3$, in $12$ of $16$ at $\nu=5$, and,
at $d=10$ with $\eta=0.05$, in both runs at each of $\nu=10$ and $\nu=12$
--- the last two \emph{inside} the range of Assumption~\ref{ass:HQ}:

\begin{center}
\begin{tabular}{lrrrrr}
\toprule
untamed divergences, pooled over $d$ & $\nu=3$ & $\nu=5$ & $\nu=10$ & $\nu=12$ & $\nu=20$ \\
\midrule
$\eta=0.05$ & $8/8$ & $6/8$ & $2/8$ & $2/8$ & $0/8$ \\
$\eta=0.01$ & $8/8$ & $6/8$ & $0/8$ & $0/8$ & $0/8$ \\
\bottomrule
\end{tabular}
\end{center}

This is the behavior the construction anticipates: $a_{k}=q(x_{k})$ grows
superlinearly, explicit schemes with superlinearly growing coefficients lose
moment boundedness, and the taming $\bar a_{k}=q/(1+\sqrt\eta\,q)$ restores it.
Where the untamed iteration does survive it is more accurate --- by a factor of
$17$ to $22$ in the two cells at $\eta=0.01$ where its bias is resolvable above
the sampling floor, and $7$ to $16$ in the two such cells at $\eta=0.05$ --- so
taming trades accuracy for stability, and in this regime stability is the
binding constraint. In summary, on heavy-tailed targets the anchoring reduces
the bias of plain Langevin on the smoothed potential by one to two orders of
magnitude in $\mathcal{W}_{2}$ wherever both biases are resolvable (a factor of $9$
to $80$ for ALMC and $15$ to $20$ for the untamed iteration at $\eta=0.01$). The
tamed variant \eqref{eq:tamedscheme} is the only iteration in the comparison
that never diverges; at $\eta=0.01$ and $d=10$ its bias is comparable to that
of ULA, the price of that stability.

\subsection{Bayesian Neural Network}\label{sec:nn}

Finally, we perform Bayesian inference for the two-layer feedforward network
of \eqref{eq:nntarget} on the two real data sets, smoothing the ReLU
activation by Corollary~\ref{cor:relu} so that no Monte Carlo enters. Here
the mismatch accumulates over data points and hidden units rather
than over coordinates, and at $\mu=1$ the anchored methods diverge in all runs
for both ReLU and leaky ReLU on both data sets, within their first $25$
iterations. The failure is by \emph{overflow}: evaluated along the KLMC
trajectory, which never uses it, the median of $a_{k}=e^{U-U_{0}}$ over the
particles reaches $4\times10^{67}$ on Banknote. KLMC itself never diverges, and
ULA's few divergences (on Breast Cancer, at $\eta=0.01$) occur at every choice
of $\mu$, so they do not depend on it. Under any scaled choice of $\mu$ every
UALMC and KLMC run is stable, while ALMC and ULA still diverge in some runs at
$\eta=0.01$; UALMC reaches $96.2$--$98.7\%$ accuracy across the data sets and
activations, comparable to the best of the benchmarks on each.

\section{Conclusion}

Although various approaches have been proposed
in the Langevin literature for sampling non-differentiable and heavy-tailed targets
in modern-day large-scale machine learning problems,
all of them were developed in the overdamped setting only.
In this paper, we proposed a novel momentum-based underdamped anchored Langevin dynamics (UALD),
and underdamped anchored Langevin Monte Carlo (UALMC) algorithm based
on a discretization scheme of UALD.
Compared to the existing literature,
our algorithm is the first kinetic
(momentum-based) Langevin algorithm for non-differentiable and heavy-tailed targets that is
free of asymptotic bias.
The method replaces the original potential with a smooth reference potential and modifies the underdamped Langevin dynamics via multiplicative scaling in the drift term.
We established non-asymptotic guarantees in the 2-Wasserstein distance to the target distribution and showed
that the resulting iteration complexity has better dependence on dimension and accuracy
level compared to overdamped-based anchored Langevin algorithms (ALMC).
We provided numerical experiments to illustrate the theory
and practical performance of our proposed approach.

\section*{Acknowledgements}
Qi Feng is partially supported by the grants NSF DMS-2306769 and DMS-2420029.
Lingjiong Zhu is partially supported by the grants NSF DMS-2053454 and DMS-2208303.

\bibliographystyle{plain}
\bibliography{anchoredlangevin}


\appendix

\section{Non-Smooth Sampling}\label{sec:nonsmooth}

When the target function $U(x)$
is not differentiable, we do not have
access to $\nabla U(x)$ and thus the classical underdamped Langevin algorithm is not directly applicable.
By taking $U_{0}(x)$ as a smoothing function of $U(x)$, our underdamped anchored Langevin algorithm can be used for non-smooth sampling.
When the target is not differentiable, smoothing methods have been studied in the literature,
including the Gaussian smoothing for Langevin algorithms \cite{Chatterji20,GNZZ2025}.
We consider the setting where the target function $U(x)$ can be written as the sum
of a smooth function $f(x)$ and a non-smooth function $g(x)$:
\begin{equation}
U(x)=f(x)+g(x).
\end{equation}
We define $U_{0}(x)$ as
\begin{equation}\label{defn:U:0:f:g:0}
U_{0}(x)=f(x)+g_{0}(x),
\end{equation}
where $g_{0}(x)$ is the Gaussian smoothing function of $g(x)$:
\begin{equation}\label{defn:g:0}
g_{0}(x)=\mathbb{E}_{\xi}[g(x+\mu\xi)],\qquad\xi\sim\mathcal{N}(0,\Id),
\end{equation}
where $\mu>0$ is the smoothing parameter that can be fine-tuned.
We next assume that $g(x)$ is Lipschitz continuous such that it is almost everywhere differentiable. For such functions, the Clarke subdifferential $\partial g(x)$ at every point $x\in\mathbb{R}^d$ exists and is a compact set \cite{rockafellar2009variational,qi1993nonsmooth}.
\begin{assumption}\label{assump:Lip:g}
Assume that $g(x)$ is $K$-Lipschitz, i.e.
\begin{equation}
|g(x)-g(y)|\leq K\Vert x-y\Vert,\qquad\text{for any $x,y\in\mathbb{R}^{d}$},
\end{equation}
Furthermore, assume that $g$ is $\varrho$-weakly convex for some $\varrho\geq 0$, i.e. the function $g_\varrho(x):= g(x) + \frac{\varrho}{2}\|x\|^2$ is convex.
\end{assumption}
$\varrho$-weakly convex functions arise frequently in machine learning applications \cite{zhu2023distributionally,zhangsapdplus,davis2018subgradient}.
For example, when $g(x)=\|x\|_1$, Assumption~\ref{assump:Lip:g} is satisfied; similarly, smoothly clipped absolute deviation (SCAD) regularizer \cite{Fan01} and
minimax concave penalty (MCP) regularizer \cite{Zhang10}
in machine learning applications
are piecewise twice continuously differentiable admitting directional derivatives and satisfy Assumption~\ref{assump:Lip:g}.
In addition, we will assume that $f(x)$ is smooth and strongly convex.
\begin{assumption}\label{assump:smooth:f}
Assume that $f(x)$ is differentiable and $L_{f}$-smooth, i.e.
\begin{equation}
\Vert\nabla f(x)-\nabla f(y)\Vert\leq L_{f}\Vert x-y\Vert,\qquad\text{for any $x,y\in\mathbb{R}^{d}$},
\end{equation}
and further assume that $f(x)$ is $m_{f}$-strongly convex, i.e.
\begin{align*}
\left\langle\nabla f(x)-\nabla f(y),x-y\right\rangle\geq m_{f}\Vert x-y\Vert^{2},\qquad\text{for any $x,y\in\mathbb{R}^{d}$}.
\end{align*}
\end{assumption}
Next, we quote a series of technical results from \cite{GNZZ2025} that show that under Assumption~\ref{assump:Lip:g},
the difference between $U$ and $U_{0}$
and the difference between the sub-differential of $U$ and $U_{0}$ are uniformly bounded and can be made arbitrarily small when the smoothing coefficient $\mu$ is chosen to be small.
\begin{lemma}[\cite{GNZZ2025}]\label{lem:CP:bound}
    Under Assumption~\ref{assump:Lip:g}, we have
    $$
    |U(x)-U_0(x)|\leq K\mu\sqrt{d}.
    $$
\end{lemma}
\begin{lemma}[\cite{GNZZ2025}]\label{ineq-small-grad-error} Let Assumption~\ref{assump:Lip:g}  hold. Then we have
\begin{equation}
\sup_{x\in\mathbb{R}^d} \mbox{dist}\left(\nabla g_0(x), \partial g(x)\right)
\leq 3^{3/4}C_{g} \mu d^{3/2},
\end{equation}
where $\mbox{dist}(\cdot, \partial g(x))$ denotes the Hausdorff distance to the set $\partial g(x)$.
\end{lemma}
\begin{proposition}[\cite{GNZZ2025}]\label{prop-when-assump-hold}
Suppose Assumptions~\ref{assump:Lip:g} and \ref{assump:smooth:f} hold. Let $U(x) = f(x) + m_0\|x\|^2 + g(x)$ and assume $U_0= f(x) + m_0\|x\|^2 + g_0(x)$ is $c_0$-strongly convex and $L_0$-smooth with $c_0 = \Theta(1)$ as $\mu \to 0$.
Assume also
$$ \sup_{x\in\mathbb{R}^{d}} \left\|\nabla U_0(x)\left(e^{g(x)-g_0(x)}-1\right)\right\| = o(\mu),$$
$$ \sup_{x\in\mathbb{R}^{d}} \left\|U_0(x) \cdot e^{g(x)-g_0(x)}\cdot (\partial g(x)-\nabla g_0(x))\right\| = o(\mu),$$
as $\mu\rightarrow 0$.
\end{proposition}
\begin{remark}[Consistency with Assumption~\ref{ass:main}]
Lemma~\ref{lem:CP:bound} gives $\omega_0\le K\mu\sqrt d$, and
Lemma~\ref{ineq-small-grad-error} gives
$\omega_1\le3^{3/4}C_g\mu d^{3/2}$ in the sense of the Clarke
subdifferential, so that Assumption~\ref{ass:main}(A) holds with a mismatch
that can be made arbitrarily small by decreasing $\mu$. In particular, the
dimension-dependent smallness requirement $\omega_1=\mathcal O(d^{-1/2})$ of
\eqref{eq:standing} is met as soon as $\mu=\mathcal O(d^{-2})$.
\end{remark}

\section{Closed-Form Gaussian Smoothing}\label{app:closed:form}

\begin{lemma}[Closed-form Gaussian smoothing of piecewise-quadratic penalties]\label{lem:closedform}
Let $p:\mathbb{R}\rightarrow\mathbb{R}$ be continuous and suppose that there
are breakpoints
$-\infty=t_{0}<t_{1}<\cdots<t_{n}<t_{n+1}=+\infty$ and coefficients
$c_{0,j},c_{1,j},c_{2,j}$ such that
\[
  p(t)=c_{0,j}+c_{1,j}t+c_{2,j}t^{2},
  \qquad t\in(t_{j},t_{j+1}),\quad j=0,1,\ldots,n .
\]
Fix $\mu>0$ and $x\in\mathbb{R}$, and set $u_{j}:=(t_{j}-x)/\mu$ \textup{(}so
$u_{0}=-\infty$ and $u_{n+1}=+\infty$\textup{)} and
\[
  P_{j}:=\Phi(u_{j+1})-\Phi(u_{j}),
  \quad
  M_{j}:=\phi(u_{j})-\phi(u_{j+1}),
  \quad
  Q_{j}:=P_{j}+u_{j}\phi(u_{j})-u_{j+1}\phi(u_{j+1}),
\]
where $\phi(z):=(2\pi)^{-1/2}e^{-z^{2}/2}$ and $\Phi(z):=\int_{-\infty}^{z}\phi(u)\,du$
denote the standard normal density and distribution function,
with the conventions $\Phi(-\infty)=0$, $\Phi(+\infty)=1$,
$\phi(\pm\infty)=0$ and $u\phi(u)\rightarrow0$ as $u\rightarrow\pm\infty$.
Then, with $Z\sim\mathcal N(0,1)$,
\begin{align}
  p_{0}(x)&:=\E\left[p(x+\mu Z)\right]
  \nonumber
\\
&=\sum_{j=0}^{n}\left\{\left(c_{0,j}+c_{1,j}x+c_{2,j}x^{2}\right)P_{j}
  +\mu\left(c_{1,j}+2c_{2,j}x\right)M_{j}+c_{2,j}\mu^{2}Q_{j}\right\},
  \label{eq:cf:value}
\end{align}
and
\begin{align}
  p_{0}'(x)=\E\left[p'(x+\mu Z)\right]
  =\sum_{j=0}^{n}\left\{\left(c_{1,j}+2c_{2,j}x\right)P_{j}
  +2c_{2,j}\mu M_{j}\right\}.
  \label{eq:cf:grad}
\end{align}
Consequently, if $g(x)=\sum_{i=1}^{d}p(x_{i})$, then
$g_{0}(x)=\sum_{i=1}^{d}p_{0}(x_{i})$ and
\[
\nabla g_{0}(x)=\left(p_{0}'(x_{1}),\ldots,p_{0}'(x_{d})\right).
\]
\end{lemma}
\begin{proof}
Partition according to $\{x+\mu Z\in(t_{j},t_{j+1})\}=\{Z\in(u_{j},u_{j+1})\}$.
Using $\phi'(z)=-z\phi(z)$ and integration by parts, for $-\infty\le u<v\le+\infty$,
\[
  \E\left[\ones_{\{u<Z<v\}}\right]=\Phi(v)-\Phi(u),
  \qquad
  \E\left[Z\ones_{\{u<Z<v\}}\right]=\phi(u)-\phi(v),
\]
\[
  \E\left[Z^{2}\ones_{\{u<Z<v\}}\right]=\Phi(v)-\Phi(u)+u\phi(u)-v\phi(v).
\]
Expanding $c_{0,j}+c_{1,j}(x+\mu Z)+c_{2,j}(x+\mu Z)^{2}$ and summing over $j$
gives \eqref{eq:cf:value}. Since $p$ is Lipschitz, dominated convergence
justifies differentiating under the expectation, and $p'$ is piecewise linear
with $p'(t)=c_{1,j}+2c_{2,j}t$ on $(t_{j},t_{j+1})$ (the finitely many
breakpoints are $Z$-null), which gives \eqref{eq:cf:grad}.
\end{proof}

\begin{remark}
Specializing Lemma~\ref{lem:closedform} to the three regularizers gives the
following formulas, in which $\lambda>0$ and $a>1$.
\begin{itemize}
\item[(a)] \emph{Lasso regularizer}. Here, $p(t)=\lambda|t|$. With $n=1$, $t_{1}=0$, and
\eqref{eq:cf:value}--\eqref{eq:cf:grad} collapse to
\begin{equation}\label{eq:cf:lasso}
  p_{0}(x)=\lambda\left\{x\left(2\Phi\!\left(\frac{x}{\mu}\right)-1\right)
  +2\mu\,\phi\!\left(\frac{x}{\mu}\right)\right\},
  \qquad
  p_{0}'(x)=\lambda\left(2\Phi\!\left(\frac{x}{\mu}\right)-1\right),
\end{equation}
i.e.\ $p_{0}$ is a Huber-type function:
\[
2\Phi(z)-1=\mathrm{erf}\left(z/\sqrt2\right)
\quad\text{and}
\quad
2\mu\phi(x/\mu)=\mu\sqrt{2/\pi}\,e^{-x^{2}/(2\mu^{2})}.
\]
In particular
$p_{0}''(x)=\frac{2\lambda}{\mu}\phi(x/\mu)$, so the smoothed Lasso penalty is
$\frac{2\lambda}{\mu\sqrt{2\pi}}$-smooth, which quantifies the constant $L$ in
Assumption~\ref{ass:main}(B).
\item[(b)] \emph{Minimax concave penalty (MCP) regularizer}. The breakpoints are
$t_{1}=-a\lambda$, $t_{2}=0$, $t_{3}=a\lambda$ and the coefficients
$(c_{0,j},c_{1,j},c_{2,j})$ are
\[
  \left(\frac{a\lambda^{2}}{2},0,0\right),\quad
  \left(0,-\lambda,-\frac{1}{2a}\right),\quad
  \left(0,\lambda,-\frac{1}{2a}\right),\quad
  \left(\frac{a\lambda^{2}}{2},0,0\right),
\]
so that $p_{0}$ is given by \eqref{eq:cf:value} and, by \eqref{eq:cf:grad},
\begin{equation}\label{eq:cf:mcp}
  p_{0}'(x)=\lambda\left(\Phi(u_{1})+\Phi(u_{3})-2\Phi(u_{2})\right)
  -\frac{x}{a}\left(\Phi(u_{3})-\Phi(u_{1})\right)
  -\frac{\mu}{a}\left(\phi(u_{1})-\phi(u_{3})\right),
\end{equation}
with $u_{1}=\frac{-a\lambda-x}{\mu}$, $u_{2}=\frac{-x}{\mu}$,
$u_{3}=\frac{a\lambda-x}{\mu}$.
\item[(c)] \emph{Smoothly clipped absolute deviation (SCAD) regularizer}. The breakpoints are
$t_{1}=-a\lambda$, $t_{2}=-\lambda$, $t_{3}=0$, $t_{4}=\lambda$,
$t_{5}=a\lambda$, and, writing $c:=\frac{1}{2(a-1)}$, the coefficients
$(c_{0,j},c_{1,j},c_{2,j})$ are
\begin{align*}
  &\left(\frac{\lambda^{2}(a+1)}{2},0,0\right),\quad
  \left(-\lambda^{2}c,-2a\lambda c,-c\right),\quad
  \left(0,-\lambda,0\right),
  \\
  &\left(0,\lambda,0\right),\quad
  \left(-\lambda^{2}c,2a\lambda c,-c\right),\quad
  \left(\frac{\lambda^{2}(a+1)}{2},0,0\right),
\end{align*}
so that $p_{0}$ is given by \eqref{eq:cf:value} and, by \eqref{eq:cf:grad},
\begin{equation}\label{eq:cf:scad}
\begin{aligned}
  p_{0}'(x)
  &=-\frac{a\lambda+x}{a-1}\left(\Phi(u_{2})-\Phi(u_{1})\right)
   -\frac{\mu}{a-1}\left(\phi(u_{1})-\phi(u_{2})\right)
  \\
  &\qquad
   +\lambda\left(\Phi(u_{4})+\Phi(u_{2})-2\Phi(u_{3})\right)
   +\frac{a\lambda-x}{a-1}\left(\Phi(u_{5})-\Phi(u_{4})\right)
   \\
   &\qquad\qquad
   -\frac{\mu}{a-1}\left(\phi(u_{4})-\phi(u_{5})\right),
\end{aligned}
\end{equation}
with $u_{j}=(t_{j}-x)/\mu$, $j=1,\ldots,5$.
\end{itemize}
\end{remark}


Lemma~\ref{lem:closedform} requires only that $p$ be continuous and piecewise
quadratic. It does not need $p$ to be a penalty, even, or nonnegative, and the
case $c_{2,j}\equiv0$ covers piecewise \emph{linear} functions, in particular
activations. Its proof justifies \eqref{eq:cf:grad} through the Lipschitz
continuity of $p$; when a quadratic piece is unbounded, as for the elastic net
below, $p$ is not Lipschitz, but $|p'(t)|\le A+B|t|$ almost everywhere for some
constants $A,B$, so for $|h|\le1$ the difference quotients
$|p(x+h+\mu Z)-p(x+\mu Z)|/|h|$ are bounded by $A+B(|x|+1+\mu|Z|)$, which is
integrable, and dominated convergence gives \eqref{eq:cf:grad} as before. Below
we use the lemma's notation --- the breakpoints $t_{j}$, the
coefficients $(c_{0,j},c_{1,j},c_{2,j})$, $u_{j}=(t_{j}-x)/\mu$, and the
quantities $P_{j}$, $M_{j}$, $Q_{j}$ --- and
\eqref{eq:cf:value} (the value $p_{0}$) and \eqref{eq:cf:grad} (the
derivative $p_{0}'$).

\begin{corollary}[activations]\label{cor:relu}
For $p(t)=\max(t,0)$, taking $t_{1}=0$ with coefficients $(0,0,0)$ and
$(0,1,0)$ in \eqref{eq:cf:value}--\eqref{eq:cf:grad} gives
\begin{equation}\label{eq:relu}
  p_{0}(x)=x\,\Phi\!\left(\frac{x}{\mu}\right)
          +\mu\,\phi\!\left(\frac{x}{\mu}\right),
  \qquad
  p_{0}'(x)=\Phi\!\left(\frac{x}{\mu}\right).
\end{equation}
For the leaky variant, $p(t)=t$ on $t>0$ and $\alpha t$ on $t<0$,
\begin{equation}\label{eq:leaky}
  p_{0}(x)=\alpha x\,\Phi\!\left(-\frac{x}{\mu}\right)
          +x\,\Phi\!\left(\frac{x}{\mu}\right)
          +\mu(1-\alpha)\,\phi\!\left(\frac{x}{\mu}\right),
  \quad
  p_{0}'(x)=\alpha\Phi\!\left(-\frac{x}{\mu}\right)
           +\Phi\!\left(\frac{x}{\mu}\right),
\end{equation}
which returns \eqref{eq:relu} at $\alpha=0$ and the identity at $\alpha=1$.
\end{corollary}

\begin{proof}
$\max(t,0)$ is continuous and piecewise linear: with the single breakpoint
$t_{1}=0$ it equals $0$ on $(-\infty,0)$ and $t$ on $(0,\infty)$, i.e.\
coefficients $(c_{0,j},c_{1,j},c_{2,j})=(0,0,0)$ and $(0,1,0)$, so
Lemma~\ref{lem:closedform} applies with $c_{2,j}\equiv0$. In its notation
$u_{1}=-x/\mu$, and using $\Phi(-z)=1-\Phi(z)$ and $\phi(-z)=\phi(z)$,
\[
  P_{0}=\Phi(-x/\mu),\quad
  P_{1}=1-\Phi(-x/\mu)=\Phi(x/\mu),\quad
  M_{0}=-\phi(x/\mu),\quad
  M_{1}=\phi(x/\mu).
\]
Then \eqref{eq:cf:value} reduces to
$p_{0}(x)=x\,P_{1}+\mu\,M_{1}=x\,\Phi(x/\mu)+\mu\,\phi(x/\mu)$ and
\eqref{eq:cf:grad} to $p_{0}'(x)=P_{1}=\Phi(x/\mu)$, which is
\eqref{eq:relu}.
For the leaky variant the pieces are $\alpha t$ and $t$, i.e.\ coefficients
$(0,\alpha,0)$ and $(0,1,0)$, we have
\[
  p_{0}(x)=\alpha x P_{0}+\mu\alpha M_{0}+xP_{1}+\mu M_{1}
  =\alpha x\,\Phi(-x/\mu)+x\,\Phi(x/\mu)+\mu(1-\alpha)\,\phi(x/\mu),
\]
$p_{0}'(x)=\alpha P_{0}+P_{1}$, which is \eqref{eq:leaky}. At $\alpha=1$,
$\Phi(-z)+\Phi(z)=1$ collapses these to $p_{0}(x)=x$, $p_{0}'\equiv1$.
This completes the proof.
\end{proof}

Since $p_{0}''(x)=\phi(x/\mu)/\mu$, the smoothed ReLU is
$1/(\mu\sqrt{2\pi})$-smooth, and the smoothed leaky ReLU, with
$p_{0}''(x)=(1-\alpha)\phi(x/\mu)/\mu$, is $|1-\alpha|/(\mu\sqrt{2\pi})$-smooth.
These are the analogues of the constant $2\lambda/(\mu\sqrt{2\pi})$ of the
smoothed Lasso, which enters $L$ in Assumption~\ref{ass:main}(B) directly; in
the network of Section~\ref{sec:nn} they enter $L$ through the chain rule,
multiplied by the weights and the data.

\begin{corollary}[further penalties]\label{cor:more}
The elastic net $p(t)=\lambda_{1}|t|+\lambda_{2}t^{2}$ has coefficients
$(0,-\lambda_{1},\lambda_{2})$, $(0,\lambda_{1},\lambda_{2})$ about $t_{1}=0$,
giving the smoothed Lasso plus $\lambda_{2}(x^{2}+\mu^{2})$. The capped
$\ell_{1}$ penalty $p(t)=\min(\lambda|t|,\lambda\theta)$ has breakpoints
$-\theta,0,\theta$ with coefficients $(\lambda\theta,0,0)$, $(0,-\lambda,0)$,
$(0,\lambda,0)$, $(\lambda\theta,0,0)$. Being bounded, its per-coordinate
mismatch is at most $\lambda\min\{\theta,\mu\sqrt{2/\pi}\}$: unlike the
Lasso's $\lambda\mu\sqrt{2/\pi}$, it saturates in the bandwidth $\mu$, so
$\theta$ varies the mismatch without varying the scale of the penalty.
\end{corollary}

\begin{proof}
Gaussian smoothing is linear in $p$, so the elastic net splits into its two
terms. The $\lambda_{1}|t|$ part is the Lasso case \eqref{eq:cf:lasso} of
the Remark following Lemma~\ref{lem:closedform}; equivalently, the stated
coefficients are those of $|t|$ scaled by $\lambda_{1}$ with $c_{2,j}=
\lambda_{2}$ added to both pieces. For the quadratic term,
$\E\left[(x+\mu Z)^{2}\right]=x^{2}+\mu^{2}$ and
$\E\left[2(x+\mu Z)\right]=2x$ directly, so that
\[
p_{0}(x)=\lambda_{1}\left\{x(2\Phi(x/\mu)-1)+2\mu\phi(x/\mu)\right\}
+\lambda_{2}\left(x^{2}+\mu^{2}\right),
\]
and
\[
p_{0}'(x)=\lambda_{1}\left(2\Phi(x/\mu)-1\right)+2\lambda_{2}x.
\]

For the capped $\ell_{1}$, the four pieces
$\lambda\theta,\;-\lambda t,\;\lambda t,\;\lambda\theta$ on
$(-\infty,-\theta)$, $(-\theta,0)$, $(0,\theta)$, $(\theta,\infty)$ agree with
$\min(\lambda|t|,\lambda\theta)$ on each interval and match at the
breakpoints ($\lambda\theta$ at $t=\pm\theta$, $0$ at $t=0$), so $p$ is
continuous and piecewise linear with the stated coefficient triples, and
Lemma~\ref{lem:closedform} gives $p_{0}$ and $p_{0}'$ by substituting them
into \eqref{eq:cf:value}--\eqref{eq:cf:grad} with $u_{j}=(t_{j}-x)/\mu$,
$t_{j}\in\{-\theta,0,\theta\}$. For the mismatch claim, write
$p=\lambda q$ with $q(t)=\min(|t|,\theta)$; $q$ is $1$-Lipschitz and bounded
by $\theta$, so $|q(t+s)-q(t)|\le\min(|s|,\theta)$ for all $t,s$, and
\begin{align*}
  |q_{0}(t)-q(t)|
  &\le\E\left[\,|q(t+\mu Z)-q(t)|\,\right]
  \\
  &\le\E\left[\min(\mu|Z|,\theta)\right]
  \\
  &\le\min\left\{\mu\,\E|Z|,\;\theta\right\}
  =\min\left\{\mu\sqrt{2/\pi},\;\theta\right\},
\end{align*}
and hence $\sup_{t}|p(t)-p_{0}(t)|\le\lambda\min\{\theta,\mu\sqrt{2/\pi}\}$.
For the Lasso this supremum is attained at $t=0$ where
$p_{0}(0)=\lambda\mu\,\E|Z|=\lambda\mu\sqrt{2/\pi}$, which grows linearly in
$\mu$; the capped penalty's mismatch stays below $\lambda\theta$, which it
approaches as $\mu$ grows.
This completes the proof.
\end{proof}

\begin{lemma}[group penalties]\label{lem:group}
For $g(x)=\lambda\|x\|_{2}$ on a group of size $k$, which for $k\ge2$ is not
piecewise quadratic, the smoothing is the mean of a noncentral $\chi_{k}$:
\begin{equation}\label{eq:group}
  \lambda\,\E\|x+\mu Z\|_{2}
  =\lambda\mu\sqrt2\,\frac{\Gamma(\frac{k+1}{2})}{\Gamma(\frac{k}{2})}\,
   {}_{1}F_{1}\!\left(-\frac12;\frac{k}{2};
   -\frac{\|x\|^{2}}{2\mu^{2}}\right),
   \qquad Z\sim\mathcal N(0,I_{k}),
\end{equation}
with radial gradient obtained from
$\frac{d}{dz}{}_{1}F_{1}(a;b;z)=\frac{a}{b}{}_{1}F_{1}(a{+}1;b{+}1;z)$.
\end{lemma}

\begin{proof}
Since $\|x+\mu Z\|_{2}=\mu\|m+Z\|_{2}$ with $m:=x/\mu$, it suffices to show
that for $Z\sim\mathcal N(0,I_{k})$ and $\delta:=\|m\|_{2}$,
\begin{equation}\label{eq:groupclaim}
  \E\left[\left\|m+Z\right\|_{2}\right]
  =\sqrt2\,\frac{\Gamma\!\left(\frac{k+1}{2}\right)}
                {\Gamma\!\left(\frac{k}{2}\right)}\,
   {}_{1}F_{1}\!\left(-\frac12;\frac{k}{2};-\frac{\delta^{2}}{2}\right);
\end{equation}
the lemma follows by multiplying by $\lambda\mu$ and setting
$\delta=\|x\|/\mu$.

\textbf{Step 1 (reduction and change of variables).} The law of $\|m+Z\|$
depends on $m$ only through $\delta$, by rotational invariance of $Z$, so
take $m=\delta e_{1}$. Substituting $w=z+\delta e_{1}$ in the Gaussian
integral and using
$e^{-\|w-\delta e_{1}\|^{2}/2}=e^{-\delta^{2}/2}e^{-\|w\|^{2}/2}
e^{\delta w_{1}}$,
\begin{equation}\label{eq:tilt}
  \E\left[\|Z+\delta e_{1}\|\right]=e^{-\delta^{2}/2}\,
  \E\left[\|Z\|\,e^{\delta Z_{1}}\right].
\end{equation}

\textbf{Step 2 (conditioning on the radius).} Write $Z=RU$ with $R=\|Z\|$ a
$\chi_{k}$ variable independent of $U$ uniform on the unit sphere. The odd
moments of $U_{1}$ vanish by symmetry, and the even ones are
\[
  \E\left[U_{1}^{2j}\right]
  =\frac{(2j-1)!!}{k(k+2)\cdots(k+2j-2)}
  =\frac{(\frac12)_{j}}{(\frac{k}{2})_{j}},
\]
using $(2j-1)!!=2^{j}(\frac12)_{j}$ and
$k(k+2)\cdots(k+2j-2)=2^{j}(\frac{k}{2})_{j}$ for the Pochhammer symbols
$(a)_{j}=a(a+1)\cdots(a+j-1)$. Hence, expanding the exponential and
exchanging expectation and sum, which is justified because $|U_{1}|\le1$ makes
the series dominated by $e^{\delta r}$, and using $(2j)!=4^{j}j!\,(\frac12)_{j}$,
\[
  \E\left[e^{\delta r U_{1}}\right]
  =\sum_{j\ge0}\frac{(\delta r)^{2j}}{(2j)!}\,
   \frac{(\frac12)_{j}}{(\frac{k}{2})_{j}}
  =\sum_{j\ge0}\frac{(\delta^{2}r^{2}/4)^{j}}{j!\,(\frac{k}{2})_{j}}.
\]

\textbf{Step 3 (radial moments).} All terms are nonnegative, so Tonelli's
theorem permits the interchange
\[
  \E\left[\|Z\|\,e^{\delta Z_{1}}\right]
  =\sum_{j\ge0}\frac{(\delta^{2}/4)^{j}}{j!\,(\frac{k}{2})_{j}}\,
   \E\left[R^{2j+1}\right],
  \qquad
  \E\left[R^{2j+1}\right]
  =2^{\,j+\frac12}\,
   \frac{\Gamma\!\left(\frac{k+1}{2}+j\right)}
        {\Gamma\!\left(\frac{k}{2}\right)},
\]
the latter being the standard $\chi_{k}$ moment
$\E R^{s}=2^{s/2}\Gamma(\frac{k+s}{2})/\Gamma(\frac{k}{2})$ at $s=2j+1$.
With $\Gamma(\frac{k+1}{2}+j)
=\Gamma(\frac{k+1}{2})\,(\frac{k+1}{2})_{j}$ this gives
\[
  \E\left[\|Z\|\,e^{\delta Z_{1}}\right]
  =\sqrt2\,\frac{\Gamma\!\left(\frac{k+1}{2}\right)}
                {\Gamma\!\left(\frac{k}{2}\right)}
   \sum_{j\ge0}\frac{(\frac{k+1}{2})_{j}}{(\frac{k}{2})_{j}}\,
   \frac{(\delta^{2}/2)^{j}}{j!}
  =\sqrt2\,\frac{\Gamma\!\left(\frac{k+1}{2}\right)}
                {\Gamma\!\left(\frac{k}{2}\right)}\,
   {}_{1}F_{1}\!\left(\frac{k+1}{2};\frac{k}{2};
   \frac{\delta^{2}}{2}\right).
\]

\textbf{Step 4 (Kummer's transformation).} Kummer's identity
${}_{1}F_{1}(a;b;z)=e^{z}\,{}_{1}F_{1}(b-a;b;-z)$ with $a=-\frac12$,
$b=\frac{k}{2}$, $z=-\frac{\delta^{2}}{2}$ reads
${}_{1}F_{1}(-\frac12;\frac{k}{2};-\frac{\delta^{2}}{2})
=e^{-\delta^{2}/2}\,
{}_{1}F_{1}(\frac{k+1}{2};\frac{k}{2};\frac{\delta^{2}}{2})$.
Combining with \eqref{eq:tilt} and Step~3 gives \eqref{eq:groupclaim}.

\textbf{Gradient.} The right-hand side of \eqref{eq:group} is a function of
$z=-\|x\|^{2}/(2\mu^{2})$ alone, and ${}_{1}F_{1}(a;b;\cdot)$ is entire, so
term-by-term differentiation of its series gives
$\frac{d}{dz}{}_{1}F_{1}(a;b;z)=\frac{a}{b}\,{}_{1}F_{1}(a+1;b+1;z)$, and
the chain rule yields
\[
  \nabla_{x}\,\lambda\,\E\left[\|x+\mu Z\|_{2}\right]
  =\frac{\lambda\sqrt2}{k\,\mu}\,
   \frac{\Gamma\!\left(\frac{k+1}{2}\right)}
        {\Gamma\!\left(\frac{k}{2}\right)}\,
   {}_{1}F_{1}\!\left(\frac12;\frac{k}{2}+1;
   -\frac{\|x\|^{2}}{2\mu^{2}}\right)x .
\]
Finally, the case $k=1$ is included in the proof above, and
\eqref{eq:groupclaim} is consistent with the Lasso formula
\eqref{eq:cf:lasso}: at $k=1$ it reads
$\E|t+Z|=\sqrt{2/\pi}\;{}_{1}F_{1}(-\frac12;\frac12;-t^{2}/2)$, while the
direct computation behind \eqref{eq:cf:lasso} gives
$\E|t+Z|=t\left(2\Phi(t)-1\right)+2\phi(t)$; equating the two is the
classical closed form
${}_{1}F_{1}(-\frac12;\frac12;-t^{2}/2)
=\sqrt{\pi/2}\,\left[t\left(2\Phi(t)-1\right)+2\phi(t)\right]$
of the half-integer Kummer function.
This completes the proof.
\end{proof}

\end{document}